%% file: main.tex
\documentclass{article}
\usepackage{iclr2027_conference,times}

\usepackage[utf8]{inputenc}
\usepackage[T1]{fontenc}
\usepackage{hyperref}
\usepackage{url}
\usepackage{booktabs}
\usepackage{amsfonts}
\usepackage{nicefrac}
\usepackage{microtype}
\usepackage{xcolor}
\usepackage{graphicx}
\usepackage{subcaption}
\usepackage{float}
\usepackage{amsmath}
\usepackage{amssymb}
\usepackage{mathtools}
\usepackage{amsthm}

\usepackage{multirow}
\usepackage{algorithm}
\usepackage{algorithmic}
\usepackage{enumitem}
\usepackage{placeins}
\usepackage[capitalize,noabbrev]{cleveref}
\usepackage{graphicx}
\usepackage{subcaption}
\usepackage{booktabs}
\usepackage{xcolor}
\usepackage{float}
\newtheorem{lemma}{Lemma}
\newtheorem{corollary}{Corollary}
\theoremstyle{remark}
\newtheorem{remark}{Remark}
\theoremstyle{plain}
\newcommand{\KL}{\mathrm{KL}}

\title{\normalfont\bfseries\centering
Information-Theoretic Classifier-Free Guidance\\
with Adaptive Schedule Optimization}

\author{%
\begin{minipage}{0.96\textwidth}
\centering\normalfont
\textbf{Haobo Chen}$^{1}$\quad
\textbf{Xiangxiang Xu}$^{2}$\quad
\textbf{Yuheng Bu}$^{1}$\\[0.5em]
{\small $^{1}$Department of Computer Science, University of California, Santa Barbara\\
$^{2}$Department of Computer Science, University of Rochester\\[0.4em]
\texttt{haobo@ucsb.edu}\quad
\texttt{xiangxiangxu@rochester.edu}\quad
\texttt{buyuheng@ucsb.edu}}
\end{minipage}
}

\iclrfinalcopy
\date{}
\hypersetup{hidelinks,
  pdftitle={Information-Theoretic Classifier-Free Guidance with Adaptive Schedule Optimization},
  pdfauthor={Haobo Chen, Xiangxiang Xu, Yuheng Bu}}

\begin{document}
\maketitle
\fancyhf{}
\pagestyle{plain}
\thispagestyle{plain}
\renewcommand{\headrulewidth}{0pt}

\begin{abstract}
Diffusion models have achieved strong performance in image and video generation, where conditional generation is often controlled by classifier-free guidance (CFG). CFG improves condition consistency by increasing a guidance weight, but stronger guidance typically reduces distributional coverage. It remains unclear how this consistency-coverage trade-off should be controlled across the reverse diffusion trajectory, since the distribution induced by CFG is not simply the fixed-time tilted distribution given by the guided score field. To address this issue, we propose an information-theoretic framework for CFG schedule optimization. Our approach uses a clean endpoint reference to specify the desired consistency-coverage trade-off, while optimizing the actual distribution induced by the guided sampler toward this reference. We derive trajectory-level formulas to estimate the objective from samples and score evaluations, avoiding explicit density estimation. {\color{black}Across four diffusion backbones on ImageNet-512 and COCO, the learned schedules achieve improved consistency--FID trade-offs, with higher condition consistency at comparable FID, than constant and non-uniform guidance baselines.}

\end{abstract}

\input{01_intro}

\input{02_preliminaries}
\input{03_information_objective}

\input{04_algorithm}
\input{05_experiments}
\input{06_conclusion}
\section*{AI Use Statement}
Generative AI tools were used to assist with language editing, formatting checks, code verification, literature search, and reference verification. All AI-assisted content was reviewed and verified by the authors, who take full responsibility for the final content of the paper.

\newpage
\bibliography{ref}
\bibliographystyle{iclr2027_conference}

\clearpage
\appendix
\input{appendix}

\end{document}

%% file: 01_intro.tex
\section{Introduction}
\label{sec:intro}

Diffusion models have emerged as a powerful paradigm for generative modeling. Building on denoising diffusion and score-based formulations~\citep{sohl2015deep,ho2020denoising,song2019generative,song2021scorebased}, they have achieved strong performance across many domains. In image generation, latent diffusion models enable high-resolution synthesis by moving the diffusion process to a learned latent space~\citep{rombach2022high}, while recent large-scale systems further improve text-to-image and ultra-high-resolution image generation~\citep{podell2023sdxl,zhang2025diffusion4k}. Beyond images, diffusion models have been extended to multimodal generation~\citep{saharia2022photorealistic,ramesh2022hierarchical, ho2022video,ho2022imagenvideo,austin2021structured}.

Much of the practical usefulness of diffusion models comes from conditional generation, where samples must follow class labels, text prompts, or other structured conditions. This control is central to modern image, text-to-image, and video generation systems, where success depends not only on quality but also on faithful alignment with the specified condition.
Classifier-free guidance (CFG) has become one of the most widely used mechanisms for improving such conditional generation~\citep{ho2022classifierfree}, building on earlier classifier-guided sampling~\citep{dhariwal2021diffusion}. 
Specifically, CFG combines conditional and unconditional score estimates through a guidance weight $w$. Larger $w$ often improves condition consistency and perceptual quality, but this typically comes at the cost of reduced distributional coverage. 

\begin{figure*}[t]
    \centering

    \begin{minipage}[t]{0.47\textwidth}
        \vspace{0pt}
        \centering
        \includegraphics[width=\linewidth]{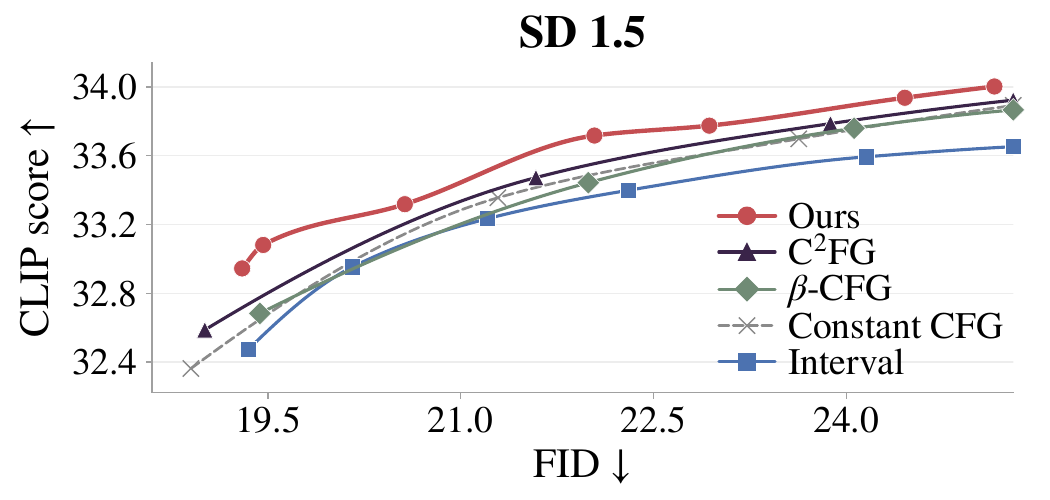}

        \vspace{0.45em}

        \includegraphics[width=\linewidth]{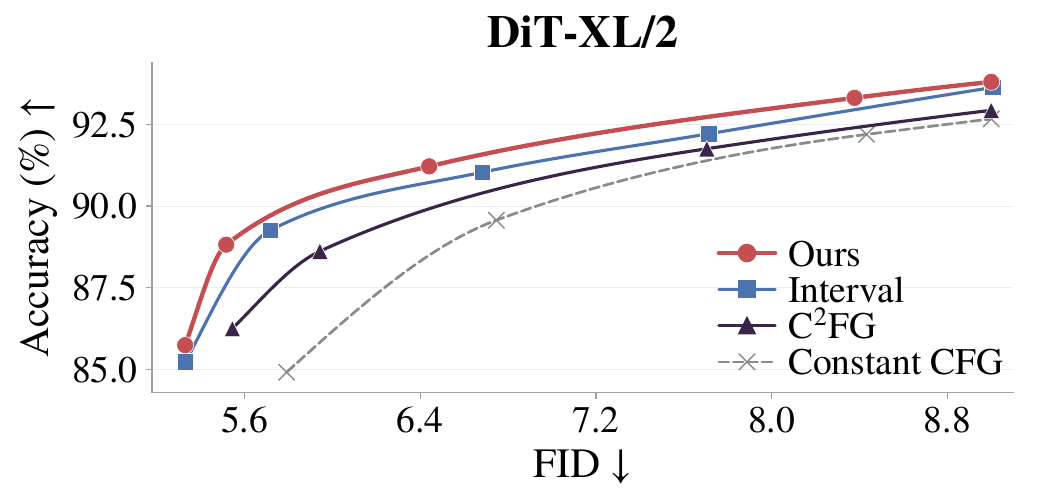}
    \end{minipage}
    \hfill
    \begin{minipage}[t]{0.50\textwidth}
        \vspace{0pt}
        \centering
        \includegraphics[width=\linewidth]{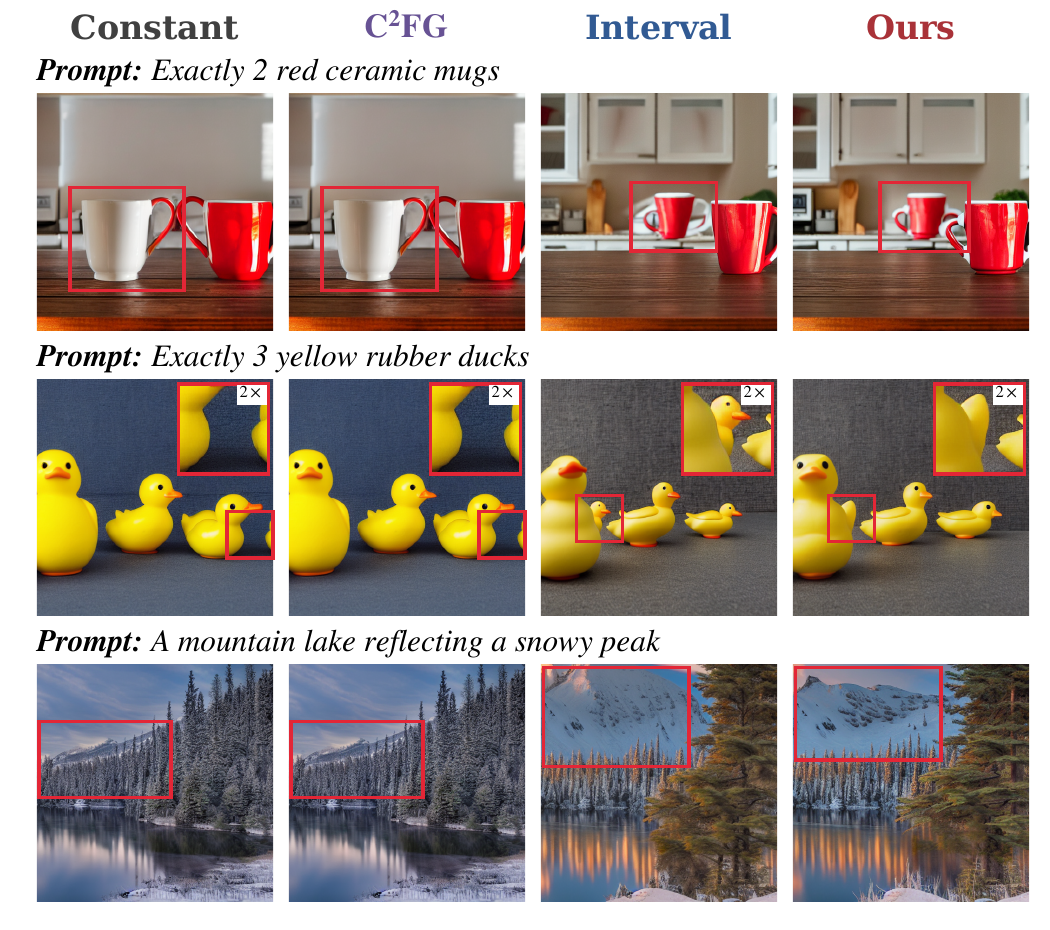}
    \end{minipage}

    \caption{
    Overview of the empirical comparison.
    Left: consistency--FID trade-off curves for SD~1.5 on COCO (top) and
    DiT-XL/2 on ImageNet-512 (bottom); upper-left is preferred.
    Right: SD~1.5 qualitative comparisons among Constant CFG, C$^2$FG\citep{gao2026c2fg},
    Interval \citep{kynkaanniemi2024applying}, and Ours at comparable dataset-level FID (approximately 21).
    Each row uses the same prompt and initial noise across methods;
    red boxes highlight prompt-relevant regions where the generated content differs.
    CLIP score and classification accuracy measure condition consistency, while FID measures distributional agreement with the data distribution; together, they characterize the evaluation metric trade-off.
    }
    \label{fig:guidance_motivation}
\vspace{-2em}
\end{figure*}

Such a nuanced trade-off has motivated methods beyond a single constant guidance weight. Some works show that applying guidance only over selected noise intervals can improve sample quality and distribution coverage over constant guidance throughout the reverse process~\citep{kynkaanniemi2024applying}. Other adaptive schedule designs vary guidance across denoising steps, including $\beta$-CFG and C$^2$FG~\citep{malarz2025classifier,gao2026c2fg}, while learned approaches can optimize an additional guidance function~\citep{galashov2025learn}.
Prior stage-wise analyses suggest that CFG can have qualitatively different effects across noise levels, including early direction changes, intermediate mode separation, and late concentration~\citep{jin2025stage}. A weaker generative model can also be used as the guidance reference to improve generation~\citep{karras2024guiding}.

However, a principled way for choosing the guidance schedule remains unclear. The main difficulty is characterizing the actual distribution induced by CFG. At a fixed time, CFG resembles a distributional tilt toward the condition, but its samples are produced by the entire reverse trajectory and need not match this fixed-time tilted interpretation~\citep{bradley2024classifier,chidambaram2024guidance,moufad2025cfgig}. A further practical challenge is that the generated distribution is implicit: during generation, we typically have only samples and score evaluations, rather than an explicit density, making consistency-based objectives difficult to estimate directly.

This leads to the central question of this paper: 
\vspace{-0.65em}
\begin{center}
\textbf{\emph{
Can we find a principled approach that balances the trade-off between consistency and coverage in CFG schedules?
}}
\end{center}
\vspace{-0.65em}
We address this question by using a clean endpoint reference to specify the desired consistency-coverage trade-off, and by optimizing the actual distribution induced by the guided sampler toward this reference.
The resulting objective combines two information-theoretic quantities evaluated under the actual CFG distribution: a consistency term, which measures alignment between generated samples and the condition, and a coverage term, which penalizes deviation from the original data distribution. Optimizing this trajectory-level objective yields an adaptive guidance schedule. 
\textcolor{black}{
Figure~\ref{fig:guidance_motivation} provides a preview of our empirical results, demonstrating improved condition and prompt consistency at comparable FID scores. The full evaluation is given in \cref{sec:experiments}.}
Our main contributions include:
\begin{itemize}[itemsep=-2.5pt,topsep=0pt, leftmargin=10pt]

\item We use a clean endpoint reference to specify the desired consistency-coverage trade-off, and formulate CFG schedule optimization as moving the actual distribution induced by the guided sampler toward this reference.

\item We derive trajectory-level formulas for estimating the consistency and coverage terms in the objective from samples and score evaluations, without explicit density estimation.

\item We directly optimize guidance weights across sampler steps using the actual CFG trajectories and frozen model scores, without prescribing a schedule shape or using evaluation metrics. Across four diffusion backbones on ImageNet-512 and COCO, the learned schedules improved evaluation metric trade-offs against constant and non-uniform guidance baselines.
\end{itemize}

%% file: 02_preliminaries.tex
\section{Preliminaries}
\label{sec:preliminaries}

\subsection{Diffusion Models and Classifier-Free Guidance (CFG)}
\label{subsec:diffusion_cfg}

We consider the variance exploding (VE) formulation of score-based diffusion models \citep{song2019generative,song2021scorebased}. Let
$(X_0,Y)\sim p_0(x_0,y)$ denote the data-label distribution over
$\mathbb R^D\times\mathcal Y$. The forward-noising process is
\begin{equation}
X_t = X_0+\sigma_t Z,
\qquad
Z\sim\mathcal N(0,I_D),
\qquad
Z\perp (X_0,Y),
\qquad
t\in[0,T],
\label{eq:ve_forward}
\end{equation}
where $\sigma_0=0$ and $\sigma_t$ is nondecreasing. Equivalently,
\begin{equation}
p_{t\mid 0}(x_t\mid x_0)
=
\mathcal N(x_t;x_0,\sigma_t^2 I_D).
\label{eq:ve_kernel}
\end{equation}
We denote by $p_t(x_t)$ and $p_t(x_t\mid y)$ the unconditional and conditional
marginals induced by this forward process at time step $t$.
The corresponding scores are given by
\begin{equation}
s_t^{\mathrm{un}}(x_t)
:=
\nabla_{x_t}\log p_t(x_t),
\qquad
s_t^{\mathrm{con}}(x_t,y)
:=
\nabla_{x_t}\log p_t(x_t\mid y).
\label{eq:population_scores}
\end{equation} 
Under VE parameterization, unconditional generation starts from
\(X_T\sim p_T\) and integrates the probability-flow ODE backward to \(t=0\)~\citep{song2021scorebased}:
\begin{equation}
\frac{dX_t}{dt}
=
-\sigma_t\dot{\sigma}_t\,s_t^{\mathrm{un}}(X_t),
\label{eq:ve_pf_ode}
\end{equation}
where $\dot{\sigma}_t=d\sigma_t/dt$. Similarly, the probability-flow ODE for conditional generation can be obtained by replacing
$s_t^{\mathrm{un}}$ with $s_t^{\mathrm{con}}$.

CFG \citep{ho2022classifierfree} strengthens
conditional generation by using a guidance weight $w\ge 0$. Motivated by the
classifier-guidance \citep{dhariwal2021diffusion}, at each fixed
noise level $t$, it can be interpreted as the tilted density
\begin{equation}
q_{t,\mathrm{tilt}}^w(x_t\mid y)
\propto
p_t(x_t)\,p_t(y\mid x_t)^w.
\label{eq:tilted_density}
\end{equation}
For \(w=1\), this reduces to the standard conditional distribution
\(p_t(x_t\mid y)\). Its score field recovers the standard CFG field
\begin{equation}
s_t^w(x_t,y)
:=
(1-w)s_t^{\mathrm{un}}(x_t)
+
w s_t^{\mathrm{con}}(x_t,y).
\label{eq:cfg_field}
\end{equation}
Substituting this guided field into the VE probability-flow ODE gives the CFG
sampling dynamics
\begin{equation}
\frac{dX_t}{dt}
=
-\sigma_t\dot{\sigma}_t\,s_t^w(X_t,y).
\label{eq:cfg_pf_ode}
\end{equation}

Let \(\Phi_{T\to t}^{w,y}\) denote the flow map of
\eqref{eq:cfg_pf_ode} from time \(T\) to \(t\). With the standard
initialization \(X_T\sim p_T\), we define the CFG-induced marginal by
\begin{equation}
q_t^w(\cdot\mid y)
:=
\bigl(\Phi_{T\to t}^{w,y}\bigr)_{\#}p_T,
\label{eq:cfg_induced_distribution}
\end{equation}
where \((\cdot)_{\#}\) denotes the pushforward of a distribution.
For a time-dependent schedule \(\mathbf w=\{w_t\}_{t\in[0,T]}\), we replace the
constant guidance weight by \(w_t\), giving
\begin{equation}
s_t^{\mathbf w}(x_t,y)
=
(1-w_t)s_t^{\mathrm{un}}(x_t)
+
w_t s_t^{\mathrm{con}}(x_t,y).
\label{eq:cfg_schedule_field}
\end{equation}
Let \(\Phi_{T\to t}^{\mathbf w,y}\) be the corresponding flow map. The
sampler-induced marginal is
\begin{equation}
q_t^{\mathbf w}(\cdot\mid y)
:=
\bigl(\Phi_{T\to t}^{\mathbf w,y}\bigr)_{\#}p_T.
\label{eq:cfg_schedule_induced_distribution}
\end{equation}
Thus, \(q_t^{\mathbf w}\) denotes the distribution of the random variable
\(X_t\) produced by the full guided trajectory, whereas \(x_t\) denotes a
realization. For a constant schedule \(w_t\equiv w\), we have
\(q_t^{\mathbf w}=q_t^w\).

\subsection{How CFG Changes the Generated Distribution}
\label{subsec:cfg_not_clean_endpoint_truncation}

\begin{figure*}[t]
    \centering
    \includegraphics[width=\textwidth]{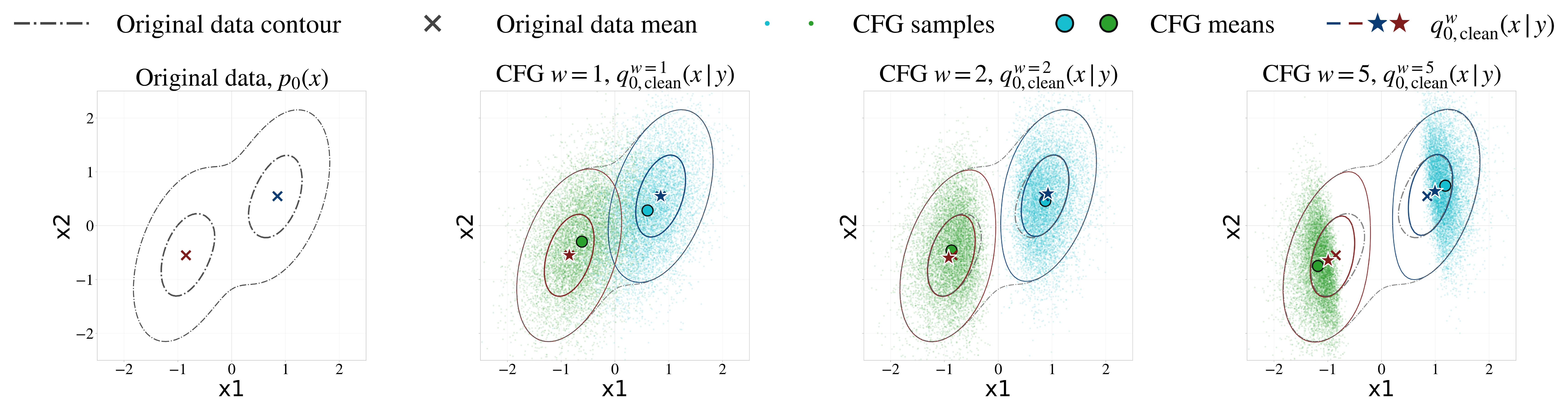}
    \caption{
    Mixture Gaussian visualization of CFG and clean endpoint tilting. The left panel shows the original data distribution \(p_0(x)\). In the remaining panels, colored dots are CFG-generated samples, circles are empirical CFG means, and stars are the means of the clean references \(q_{0,\mathrm{clean}}^w(x_0\mid y)\propto p_0(x_0)p_0(y\mid x_0)^w\). As \(w\) increases, the actual CFG samples separate and concentrate more aggressively than the clean references, illustrating that CFG does not simply implement direct endpoint reweighting.
    }
    \vspace{-1em}
    \label{fig:matched_comparison}
\end{figure*}

CFG induces a consistency-coverage trade-off in conditional generation: stronger guidance can improve alignment with the condition, but may reduce coverage by moving the generated distribution farther from the original data distribution \citep{ho2022classifierfree}. 
In particular, changing \(w\) does not simply make samples more condition-specific while keeping them near the original data distribution; it reshapes the generated distribution through the entire reverse trajectory.

To see this, we consider a two-dimensional Gaussian mixture example. Let
\(Y\in\{\pm 1\}\) with equal prior probability, and denote the original conditional
distributions as
\begin{equation}
p_0(x\mid y)
=
\mathcal N(x;\mu_y,\Sigma),
\qquad
x\in\mathbb R^2,
\qquad
\mu_1=(0.85,0.55),
\qquad
\mu_{-1}=-(0.85,0.55).
\label{eq:gaussian_toy_conditional}
\end{equation}
We apply the same forward noising process as in \eqref{eq:ve_forward}, then the noisy conditional distributions remain Gaussian $p_t(x_t\mid y)=\mathcal N(x_t;\mu_y,\Sigma+\sigma_t^2 I_2)$,
and both \(s_t^{\mathrm{un}}\) and \(s_t^{\mathrm{con}}\) can be computed
analytically.

The tilted distribution in \eqref{eq:tilted_density} gives a useful reweighting interpretation of the CFG score field at each noisy level. Since samples are evaluated at the clean endpoint at $t=0$, a natural reference for the desired consistency-coverage trade-off is the following clean endpoint tilted distribution,
\begin{equation}
q_{0,\mathrm{clean}}^w(x\mid y)
\propto
p_0(x)p_0(y\mid x)^w .
\label{eq:gaussian_toy_clean_tilt}
\end{equation}
We then examine whether the actual generated distribution $q_0^w(x_0\mid y)$ matches either the original distribution $p_0(x_0\mid y)$ or the clean reference $q_{0,\mathrm{clean}}^w(x_0\mid y)$. 
Figure~\ref{fig:matched_comparison} shows that the CFG  distribution does not match $p_0(x_0\mid y)$. As $w$ increases, the CFG samples separate and concentrate more aggressively. With moderate guidance, this movement can bring the empirical CFG mean closer to the corresponding data component. With large guidance, such as $w=5$, the sample clouds become too narrow, and some samples move into low-probability regions under the original data distribution.

The stars in Figure~\ref{fig:matched_comparison} show the means of distribution $q_{0,\mathrm{clean}}^w(x_0\mid y)$. This clean reference mainly suppresses mass near the overlap between the two components, so its mean moves relatively mildly and remains closer to the true mean. In contrast, CFG modifies the guided score field along the noisy reverse trajectory. Thus, although both clean endpoint tilting and CFG make the conditional samples more separated as $w$ increases, $q_0^w(x_0\mid y)$ does not match $q_{0,\mathrm{clean}}^w(x_0\mid y)$.
This example motivates the distinction formalized next: fixed-time reweighting provides an interpretation of CFG, but it does not characterize the actual distribution induced by the guided reverse dynamics.
\subsection{Why Tuning CFG is Difficult}
\label{subsec:mismatch}

The Gaussian example highlights the mismatch between the clean reference $q_{0,\mathrm{clean}}^w(x_0\mid y)$ and the actual CFG distribution $q_0^w(x_0\mid y)$.
As shown in Figure~\ref{fig:matched_comparison}, in general,
\begin{equation}
q_0^w(x_0\mid y)
\neq
q_{0,\mathrm{clean}}^w(x_0\mid y).
\label{eq:terminal_cfg_clean_mismatch}
\end{equation}
The source of this mismatch
can be understood through two related facts: tilting before noising is 
different from tilting after noising, and the fixed-time tilted distribution
should not be identified with the  distribution induced by the guided
reverse dynamics. This is consistent with recent analyses showing that guided diffusion does \emph{not} generally sample from the tilted distribution suggested by the fixed-time guided score field
\citep{bradley2024classifier,chidambaram2024guidance,moufad2025cfgig}.
We provide the detailed derivation in Appendix~\ref{app:mismatch-proof}.

This mismatch makes tuning CFG difficult because a desired clean-time trade-off does not directly specify how the guidance schedule should be distributed over noise levels. The clean reference specifies what kind of output distribution we would like to obtain: samples should be compatible with the condition while remaining close to the original clean data distribution. However, CFG does not directly follow this reference distribution. It controls noisy-level score fields along the reverse trajectory, and the final distribution is the accumulated result of all guidance decisions
along that trajectory. Therefore, changing \(w_t\) at one noise level can affect the later trajectory and the clean-time trade-off in a nonlocal way.

A second difficulty is that the actual CFG distribution is implicit. Even for a fixed schedule \(\mathbf w\), the induced distributions $q_t^{\mathbf w}(x_t\mid y)$ and $q_0^{\mathbf w}(x_0\mid y)$ are not available as explicit densities. Thus, the desired consistency-coverage trade-off cannot be evaluated by directly computing endpoint density ratios such as
\(\log q_0^{\mathbf w}(x_0\mid y)-\log p_0(x_0)\). In practice, we only have samples generated along the reverse trajectory and score evaluations from the conditional and unconditional diffusion models.

%% file: 03_information_objective.tex
\section{An Information-Theoretic Analysis of Classifier-Free Guidance}
\label{sec:it_cfg}

The previous discussion suggests that a principled objective for guidance schedule optimization should satisfy two requirements. First, it should specify the desired trade-off between condition consistency and closeness to the original clean data distribution. Second, it should be evaluable under the actual CFG distribution induced by the sampler, even though this distribution is available only implicitly through generated trajectories and score functions.

In this section, we construct an objective for guidance schedule optimization. The objective uses a clean endpoint reference to specify the desired consistency-coverage trade-off and optimizes the actual sampler-induced distribution toward this reference. To make the objective estimable, we rewrite it as trajectory-level expressions along the reverse sampler. This yields two terms: a consistency term measuring compatibility with the condition, and a coverage term measuring deviation from the original clean data distribution.

\subsection{Objective}
\label{subsec:terminal_clean_objective}

We first address how to specify the desired consistency-coverage trade-off. We use the clean reference as in~\eqref{eq:gaussian_toy_clean_tilt}, now indexed by a reference parameter \(\lambda\ge 0\):
\begin{equation}
q_{0,\mathrm{clean}}^\lambda(x_0\mid y)
\propto
p_0(x_0)p_0(y\mid x_0)^\lambda .
\label{eq:terminal_clean_tilt}
\end{equation}
Larger \(\lambda\) places more weight on samples that are consistent with the condition, while the factor \(p_0(x_0)\) anchors the reference to the original data distribution. The parameter $\lambda$ defines the clean endpoint reference; it does not affect the reverse sampler. In contrast, the guidance schedule $\mathbf w$ controls the sampling dynamics and induces the actual CFG distribution $q_0^{\mathbf w}(x_0\mid y)$.

The next step is to evaluate whether the guided sampler realizes this desired trade-off. Since we can sample trajectories from the actual CFG distribution, we compare the sampler-induced distribution to the clean reference:
\begin{equation}
\mathbb E_{p(y)}
\mathrm{KL}\!\left(
q_0^{\mathbf w}(x_0\mid y)
\,\|\,
q_{0,\mathrm{clean}}^\lambda(x_0\mid y)
\right).
\label{eq:clean_target_kl_objective}
\end{equation}
Note that we use the forward KL rather than the reverse KL because we do not have a sampler for the reference distribution $q_{0,\mathrm{clean}}^\lambda(x_0\mid y)$, which specifies the desired trade-off. In contrast, $q_0^{\mathbf w}(x_0\mid y)$ is the distribution actually produced by the guided sampler, so expectations under it can be estimated from generated trajectories.

The KL in \eqref{eq:clean_target_kl_objective} also decomposes into the two quantities we want to control.

\begin{lemma}
\label{lem:terminal_clean_objective}
For fixed \(\lambda\), minimizing $
\mathbb E_{p(y)}
\mathrm{KL}\!\left(
q_0^{\mathbf w}(x_0\mid y)
\,\|\,
q_{0,\mathrm{clean}}^\lambda(x_0\mid y)
\right)$
over guidance schedules \(w\) is equivalent to minimizing
\begin{align}
\mathcal L_{0}({\mathbf w};\lambda)
&:=
-\lambda
\underbrace{
\mathbb E_{p(y)q_0^{\mathbf w}(x_0\mid y)}
\left[
\log p_0(y\mid x_0)
\right]
}_{\text{consistency}} +
\underbrace{
\mathbb E_{p(y)}
\mathrm{KL}\!\left(
q_0^{\mathbf w}(x_0\mid y)
\,\|\,
p_0(x_0)
\right)
}_{\textcolor{black}{\text{coverage}}} .
\label{eq:terminal_cfg_objective}
\end{align}

\end{lemma}

The proof of \cref{lem:terminal_clean_objective} is provided in
Appendix~\ref{app:objective-proof}.\textcolor{black}{
Lemma~\ref{lem:terminal_clean_objective} shows that matching the clean
reference is equivalent to optimizing a consistency--KL trade-off. Moreover,
$q_{0,\mathrm{clean}}^\lambda$ is the exact minimizer of this objective when the
endpoint distribution $q(x_0\mid y)$ can be chosen freely.}
Figure~\ref{fig:gaussian_tradeoff_target} in
Appendix~\ref{app:gaussian-toy} illustrates these two information quantities
used in our objective for the Gaussian example discussed in
Section~\ref{subsec:cfg_not_clean_endpoint_truncation}. In particular,
increasing $w$ in CFG improves the consistency term but also increases the coverage KL,
illustrating the trade-off induced by constant guidance. In contrast, varying
$\lambda$ in the clean reference yields a better trade-off: at a similar
consistency level, the clean reference remains closer to the original data
distribution than the actual CFG distribution.

Moreover, as shown in the following two subsections, an additional advantage of this KL formulation is that both terms can be estimated along the guided probability-flow trajectory from generated samples and score evaluations.  Classical de Bruijn identities relate the evolution of information quantities to score functions
\citep{cover2006elements}. Unlike the classical setting, however, our objective is evaluated under the sampler-induced distribution \(q_t^{\mathbf w}\), whose density is only
implicitly defined. To overcome this challenge, Lemmas~\ref{lem:consistency_path_identity} and
\ref{lem:coverage_kl_identity} derive the corresponding trajectory
formulas using the guided samples and model scores
\citep{song2021maximum,song2021scorebased}.

\subsection{Consistency Term}
\label{subsec:consistency_term}

The consistency term rewards samples that are compatible with the condition $y$ under the original data distribution. Since it is difficult to evaluate directly at $t=0$, the following lemma expresses this quantity along the guided probability-flow trajectory.

\begin{lemma}
\label{lem:consistency_path_identity}
Under the VE probability-flow dynamics in \eqref{eq:cfg_pf_ode}, for any
\(t\in[0,T]\),
\begin{align}
&\mathbb E_{p(y)q_t^{\mathbf w}(x_t\mid y)}
\left[
\log p_t(y\mid x_t)
\right]
=
\mathbb E_{p(y)q_T^{\mathbf w}(x_T\mid y)}
\left[
\log p_T(y\mid x_T)
\right]
\\
&\ 
-
\int_t^T \!\!
\sigma_s\dot{\sigma}_s\,
\mathbb E_{p(y)q_s^{\mathbf w}(x_s\mid y)}
\left[
\nabla_x\!\cdot s_s^{\mathrm{diff}}(x_s,y)
+
\left\langle
s_s^{\mathrm{diff}}(x_s,y),
s_s^{\mathrm{con}}(x_s,y)
\right\rangle
-
w_s
\left\|
s_s^{\mathrm{diff}}(x_s,y)
\right\|^2
\right]ds, \nonumber
\label{eq:consistency_path_identity}
\end{align}
where $
s_t^{\mathrm{diff}}(x_t,y) := s_t^{\mathrm{con}}(x_t,y) - s_t^{\mathrm{un}}(x_t).$
\end{lemma}

Setting \(t=0\) in \cref{lem:consistency_path_identity} gives the consistency term in \eqref{eq:terminal_cfg_objective}. The identity has a
de Bruijn-type \citep{cover2006elements,song2021maximum} interpretation:  it tracks how the noisy consistency
\(\mathbb E[\log p_t(y\mid x_t)]\) changes along the probability-flow
trajectory, instead of evaluating the quantity
directly. The integral captures two effects: the divergence term measures local expansion or contraction of the condition-dependent score direction, while the inner-product term measures its interaction with the guided score field. A detailed derivation is
given in Appendix~\ref{app:consistency-proof}.

In practice, for a fixed schedule \(\mathbf w\), we generate trajectories from
the guided sampler and evaluate the integrand in Lemma~
\ref{lem:consistency_path_identity}. The divergence
\(\nabla_x\!\cdot s_t^{\mathrm{diff}}\) can be estimated with Hutchinson trace estimators \citep{hutchinson1990stochastic,grathwohl2019ffjord},
and the inner-product term only requires evaluations of
\(s_t^{\mathrm{un}}\), \(s_t^{\mathrm{con}}\), and \(s_t^{\mathrm{diff}}\), which are available in CFG diffusion models.

\subsection{Coverage Term}
\label{subsec:kl_term}

\textcolor{black}{
We use \emph{coverage} to describe how well the guided distribution remains
close to the original clean data distribution. The second term in
\eqref{eq:terminal_cfg_objective},
\(\mathrm{KL}(q_0^{\mathbf w}(x_0\mid y)\|p_0(x_0))\),
measures this departure from \(p_0(x_0)\), with larger values indicating
greater deviation from the original data distribution.
}
As with the consistency term, we avoid direct estimation by rewriting this
quantity as an integral along the guided probability-flow trajectory. The
following lemma allows us to estimate this KL term without directly estimating
the clean densities \(q_0^{\mathbf w}(x_0\mid y)\) or \(p_0(x_0)\).

\begin{lemma}
\label{lem:coverage_kl_identity}
Under the VE probability-flow dynamics in \eqref{eq:cfg_pf_ode}, the
\textcolor{black}{KL term in \eqref{eq:terminal_cfg_objective}} satisfies
\begin{equation}
\begin{aligned}
&\mathbb E_{p(y)}
\mathrm{KL}\!\left(
q_0^{\mathbf w}(x_0\mid y)
\,\|\,
p_0(x_0)
\right)
=
\mathbb E_{p(y)}
\mathrm{KL}\!\left(
q_T^{\mathbf w}(x_T\mid y)
\,\|\,
p_T(x_T)
\right)
\\
&\quad
-
\int_0^T
\sigma_t\dot{\sigma}_t\,
\mathbb E_{p(y)q_t^{\mathbf w}(x_t\mid y)}
w_t
\left[
\nabla_x\!\cdot s_t^{\mathrm{diff}}(x_t,y)
+
\left\langle
s_t^{\mathrm{un}}(x_t),
s_t^{\mathrm{diff}}(x_t,y)
\right\rangle
\right]dt .
\end{aligned}
\label{eq:coverage_kl_identity}
\end{equation}
\end{lemma}

With the standard initialization \(X_T\sim p_T\) in generation, the first KL term on the right-hand side is zero. We keep it in the identity to make the probability-flow relation explicit.  Detailed derivations are given in Appendix~\ref{app:coverage-proof}.

{\color{black}
\begin{remark}[Local effect of guidance]
Holding the current sampler marginal \(q_t^{\mathbf w}\) fixed,
Lemma~\ref{lem:consistency_path_identity} shows that the local consistency
response to increasing \(w_t\) is proportional to
\(\mathbb E_{p(y)q_t^{\mathbf w}}[\|s_t^{\mathrm{diff}}\|^2]\).
Lemma~\ref{lem:coverage_kl_identity} shows that the same increase in \(w_t\)
also changes the KL term through
\(-\mathbb E_{p(y)q_t^{\mathbf w}}[
\nabla_x\!\cdot s_t^{\mathrm{diff}}
+\langle s_t^{\mathrm{un}},s_t^{\mathrm{diff}}\rangle]\),
thereby affecting coverage. Therefore, the local consistency–coverage trade-off is governed by both effects, rather than solely by the magnitude of the score difference, as commonly considered in the literature~\citep{gao2026c2fg,malarz2025classifier}. This joint perspective motivates a new learning schedule that achieves improved performance. The exact characterization of the local response is provided in
Corollary~\ref{cor:local_cfg_tradeoff} in
Appendix~\ref{app:coverage-proof}.
\end{remark}
}

%% file: 04_algorithm.tex
\section{Guidance Schedule Optimization Algorithm}
\label{sec:guidance_schedule_algorithm}

We now use the trajectory objective to learn a time-dependent guidance schedule. The objective derived above is written in continuous noise time \(t\in[0,T]\), while practical samplers run on a finite reverse-time grid $ T=t_K>\cdots>t_1\approx 0$. We therefore optimize one guidance value per sampling step. Let $\mathbf w=(w_1,\ldots,w_K)$ denote the discrete schedule, where $w_k := w(t_k)$, $k$ indexes the discrete sampler step, and $t_k$ is the corresponding time step.

Given a batch of conditions \(\{y^{(b)}\}_{b=1}^B\), we generate trajectories
\(\{x_{t_k}^{(b)}\}_{k=1,b=1}^{K,B}\) using the current schedule. We approximate the trajectory objective by the Riemann sum
\begin{align}
\widehat{\mathcal L}_{0}(\mathbf w;\lambda)
&=
\sum_{k=1}^{K}
\frac{a_k}{B}
\sum_{b=1}^{B}
\Bigg[
(\lambda-w_k)
\Big(
\nabla_x\!\cdot
s_{t_k}^{\mathrm{diff}}(x_{t_k}^{(b)},y^{(b)})
+
\left\langle
s_{t_k}^{\mathrm{diff}}(x_{t_k}^{(b)},y^{(b)}),
s_{t_k}^{\mathrm{con}}(x_{t_k}^{(b)},y^{(b)})
\right\rangle
\Big)
\nonumber\\
&\hspace{3.2cm}
+
w_k(1-\lambda)
\left\|
s_{t_k}^{\mathrm{diff}}(x_{t_k}^{(b)},y^{(b)})
\right\|^2
\Bigg].
\label{eq:discrete_schedule_loss}
\end{align}

Here, $a_k$ is the integration weight at sampling step $k$, accounting for the factor $\sigma_t\dot{\sigma}_t$ and the step size. We drop additive terms that do not depend on $\mathbf w$, since they do not affect its optimization.

We optimize this objective by iteratively updating the guidance schedule. At each iteration, we generate trajectories using the current schedule, hold them fixed, and use $\partial \widehat{\mathcal L}_0/\partial w_k$ to propose updates to the guidance weights $w_k$. This calculation does not account for how the new weights change the sampler's later states. We therefore limit the update size and rerun the sampler with the candidate schedule to evaluate its effect. Following a trust-region approach~\citep{schulman2015trust}, we adjust this limit by comparing the objective decrease predicted from the fixed trajectories with the decrease measured after rerunning the sampler. Appendix~\ref{app:loss-implementation} gives the derivation and the update rule.
{\color{black}The complete procedure appears in Algorithm~\ref{alg:stepwise_guidance_update}, with descent and finite-trial analysis in Appendix~\ref{app:fixed-trajectory-analysis}.}

Rather than restricting the schedule to a predefined shape or a few
parameters, we directly optimize the guidance weight at each sampler step
where updates are allowed.
This lets us change both the overall guidance strength and when guidance
is stronger or weaker.
Unlike \emph{Learn to Guide Your Diffusion Model}~\citep{galashov2025learn},
we do not train an additional guidance function.
The diffusion model stays frozen, and we learn the schedule from its CFG
trajectories and conditional and unconditional scores, without real-image
supervision or evaluation metrics.
We examine the learned schedules and their qualitative effects in
\cref{subsec:learned_schedule_analysis,subsec:additional_generation_visualizations}.

%% file: 05_experiments.tex
\section{Experiments}
\label{sec:experiments}

Our formulation uses $\lambda$ to specify an endpoint trade-off and optimizes
time-dependent guidance under the actual CFG dynamics. We first test whether
this objective trade-off is reflected in standard evaluation metrics,
then examine how the learned schedules allocate guidance along the reverse
trajectory. 

\subsection{Experimental Setup}
\label{subsec:experimental_setup}

\noindent\textbf{Models and baselines.}
We evaluate text-to-image generation with SD~1.5 and SDXL on COCO val2017,
and class-conditional generation with DiT-XL/2 and EDM2-S on ImageNet-512.
We compare against constant CFG, Interval Guidance
\citep{kynkaanniemi2024applying}, C$^2$FG~\citep{gao2026c2fg}, and
$\beta$-CFG~\citep{malarz2025classifier}, where applicable. For our method, each operating point uses a schedule optimized for a
reference parameter $\lambda$.
The schedules are learned from generated trajectories and the conditional
and unconditional scores of the frozen diffusion model; no
evaluation metric is used during schedule optimization.

\noindent\textbf{Evaluation protocol.}
For text-to-image generation, we evaluate SD~1.5 and SDXL on the same
5,000 COCO val2017 image--caption pairs using FID~\citep{heusel2017gans}
and CLIP score~\citep{radford2021learning}.
Both models use 40-step deterministic sampling.
For class-conditional generation, each operating point is evaluated using
10,000 ImageNet-512 samples, with 100-step deterministic DDIM sampling
for DiT-XL/2 and 32-step deterministic Heun sampling for EDM2-S.
We report FID, ImageNet classification accuracy, Inception Score (IS), precision, and recall.
Within each backbone, competing methods use the same generation and evaluation
protocol, with conditioning inputs and initial noise reused across methods
where applicable.
Detailed baseline configurations, optimization settings, evaluation protocols,
metric definitions, and supplementary results are provided in
Appendix~\ref{app:implementation_details}.
{\color{black}Additional sampling-step robustness results for EDM2-S
are reported in Appendix~\ref{app:rebuttal_sampling_steps},
Table~\ref{tab:app_sampling_steps}.} 

\subsection{Evaluation Metric Trade-offs}
\label{subsec:external_tradeoff}

The objective in \cref{sec:it_cfg} balances condition consistency against deviation from the original data distribution. We evaluate this trade-off using CLIP--FID for text-to-image generation and classification accuracy--FID for class-conditional generation. The curves in Figures~\ref{fig:guidance_motivation} and~\ref{fig:external_tradeoff} are obtained by varying guidance strength for the baselines and $\lambda$ for our method. Because these parameters have different meanings, we compare the resulting metric trade-offs rather than matching their numerical values. Curves toward the upper-left are preferred: they offer higher condition consistency at comparable FID, or lower FID at comparable consistency.

Our learned schedules improve the measured trade-offs across all four backbones. On SD~1.5 and SDXL, they achieve higher CLIP scores at comparable FID, or lower FID at comparable CLIP. On DiT-XL/2 and EDM2-S, they achieve higher classification accuracy at comparable FID over a substantial portion of the evaluated range. These gains are measured by metrics that are not used during schedule optimization, providing external evidence that the trajectory objective yields improved evaluation metric trade-offs.

\begin{figure*}[t]
    \centering

    \begin{subfigure}[t]{0.48\textwidth}
        \centering
        \includegraphics[width=\linewidth]{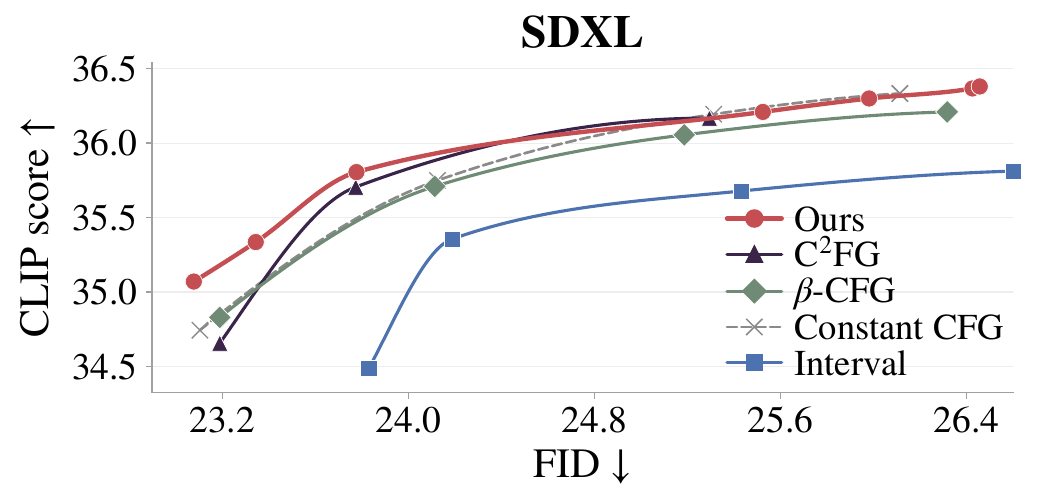}
        \caption{SDXL: CLIP--FID.}
        \label{fig:sdxl_fid_clip}
    \end{subfigure}
    \hfill
    \begin{subfigure}[t]{0.48\textwidth}
        \centering
        \includegraphics[width=\linewidth]{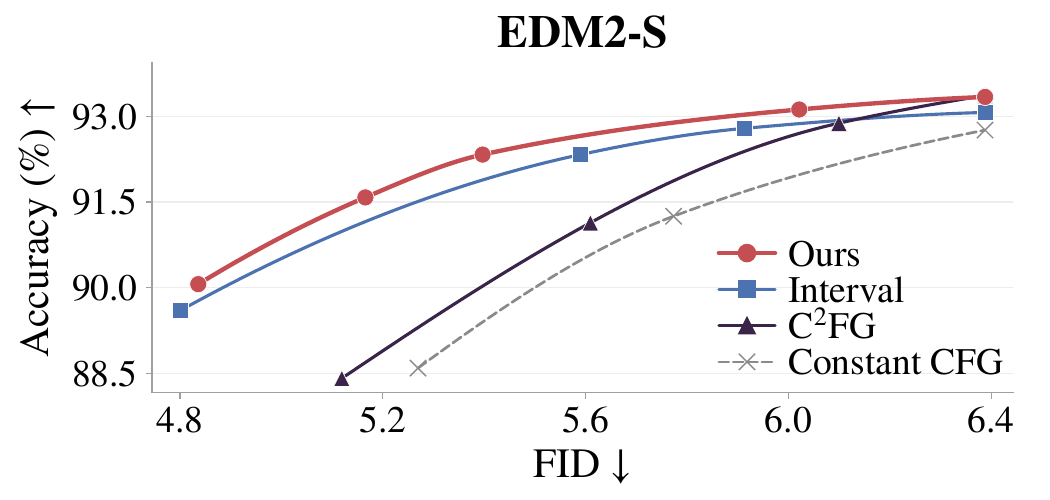}
        \caption{EDM2-S: classification accuracy--FID.}
        \label{fig:edm_fid_acc}
    \end{subfigure}

    \par\medskip

    \begin{subfigure}[t]{0.48\textwidth}
        \centering
        \IfFileExists{figure/sit256_fid_accuracy.png}{%
          \includegraphics[width=\linewidth]{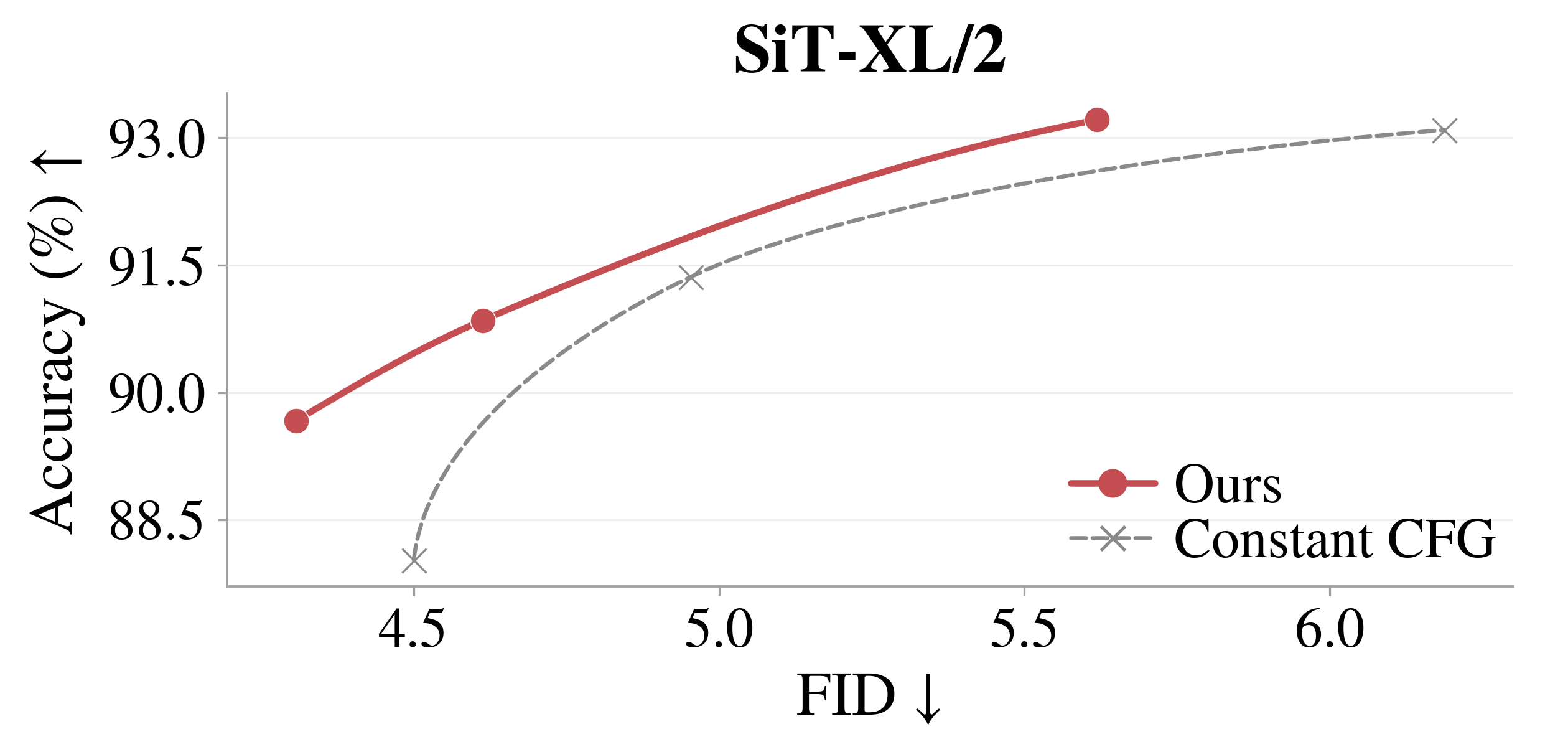}%
        }{\IfFileExists{figure/sit256_fid_accuracy.pdf}{%
          \includegraphics[width=\linewidth]{figure/sit256_fid_accuracy.pdf}%
        }{%
          \fbox{\parbox[c][3cm][c]{0.9\linewidth}{\centering
          SiT-XL/2 accuracy--FID\\
          \texttt{figure/sit256\_fid\_accuracy.png}}}%
        }}
        \caption{\textcolor{black}{SiT-XL/2: classification accuracy--FID.}}
        \label{fig:modern_sit_fid_accuracy}
    \end{subfigure}
    \hfill
    \begin{subfigure}[t]{0.48\textwidth}
        \centering
        \IfFileExists{figure/sd3_fid_clip.png}{%
          \includegraphics[width=\linewidth]{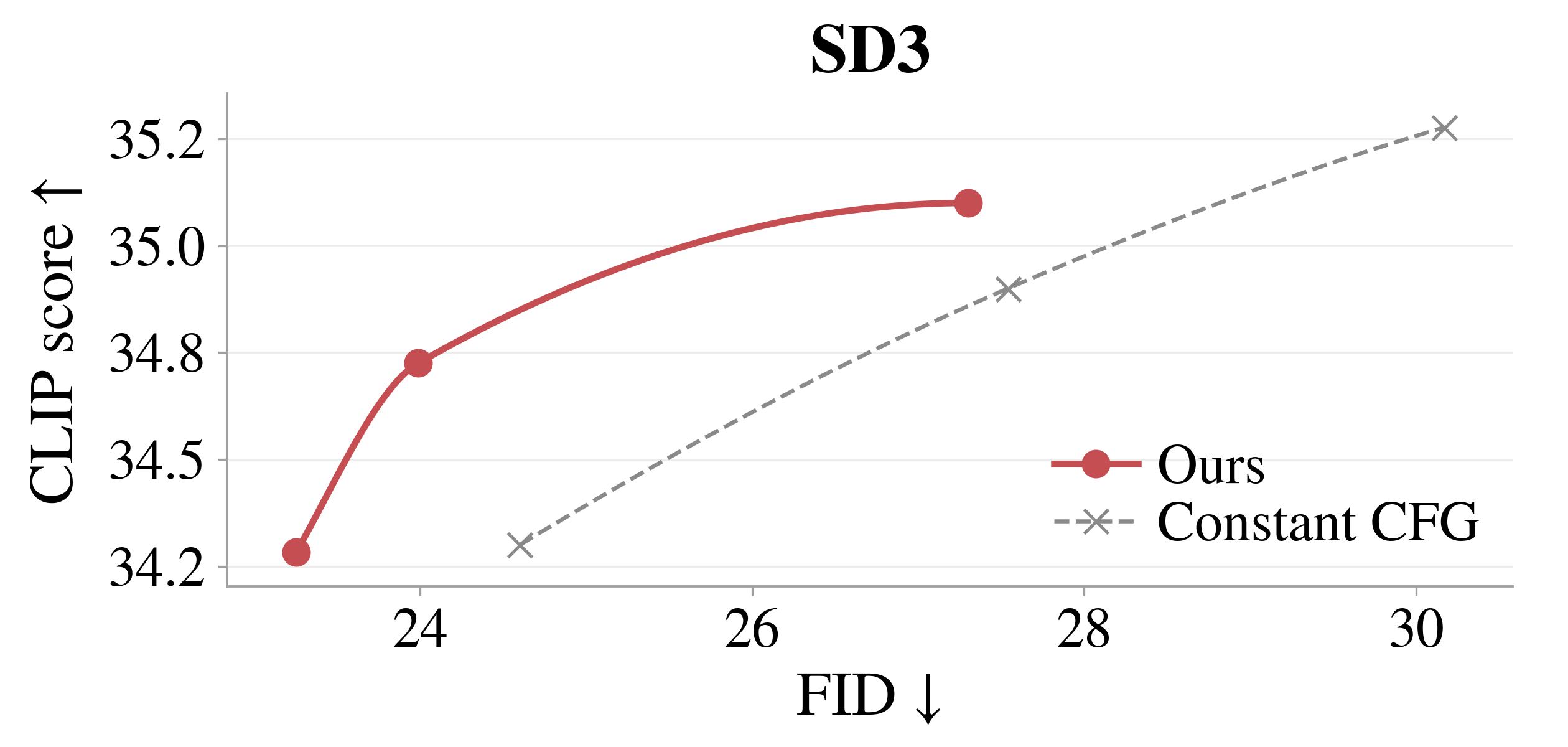}%
        }{\IfFileExists{figure/sd3_medium_fid_clip.pdf}{%
          \includegraphics[width=\linewidth]{figure/sd3_medium_fid_clip.pdf}%
        }{%
          \fbox{\parbox[c][3cm][c]{0.9\linewidth}{\centering
          SD3 Medium CLIP--FID\\
          \texttt{figure/sd3\_fid\_clip\_USER\_PROVIDED.png}}}%
        }}
        \caption{\textcolor{black}{SD3 Medium: CLIP--FID.}}
        \label{fig:modern_sd3_fid_clip}
    \end{subfigure}

    \caption{\textcolor{black}{
    Consistency--FID trade-off curves across diffusion and flow-matching architectures.
    Baseline curves sweep guidance strength, while our points use
    schedules optimized for different reference parameters $\lambda$.
    Upper-left is preferred.
    }}
    \label{fig:external_tradeoff}
\end{figure*}

The curves compare methods across a range of operating points.
To complement this comparison with concrete numerical results,
Table~\ref{tab:imagenet_best_fid} reports each method's lowest-FID
operating point within its evaluated ImageNet parameter grid. Our method achieves the lowest FID on DiT-XL/2 and nearly matches
Interval Guidance on EDM2-S.
At these minimum-FID operating points, it also achieves the highest IS
on both backbones and higher classification accuracy than Interval Guidance.
Its precision is highest on both backbones, while Interval Guidance attains
higher recall; full precision--recall curves are provided in
Appendix~\ref{app:additional_external_evaluation}.

\begin{table}[!htbp]
\centering
\caption{
Lowest-FID operating point for each method within its evaluated
ImageNet-512 parameter grid. All metrics use 10,000 generated samples.
}
\label{tab:imagenet_best_fid}
\setlength{\tabcolsep}{4.2pt}
\begin{tabular}{l c c c c c c}
\toprule
\textbf{Method}
& \textbf{Setting}
& \textbf{FID} $\downarrow$
& \textbf{Acc.} $\uparrow$
& \textbf{IS} $\uparrow$
& \textbf{Prec.} $\uparrow$
& \textbf{Rec.} $\uparrow$ \\
\midrule

\multicolumn{7}{l}{
    \textbf{DiT-XL/2} \quad \textit{(DDIM, 100 steps)}
} \\
Constant CFG
& $w=1.5$
& 5.79
& 84.91
& 224.18
& 0.799
& 0.704 \\

C$^2$FG
& $w_0=1.1$
& 5.55
& \textbf{86.25}
& 238.99
& 0.822
& 0.695 \\

Interval
& $w=2.5$
& 5.33
& 85.23
& 233.24
& 0.805
& \textbf{0.715} \\

Ours
& $\lambda=1.8$
& \textbf{5.29}
& 85.74
& \textbf{240.32}
& \textbf{0.828}
& 0.691 \\

\midrule

\multicolumn{7}{l}{
    \textbf{EDM2-S} \quad \textit{(Heun, 32 steps)}
} \\
Constant CFG
& $w=1.4$
& 5.27
& 88.59
& 209.66
& 0.782
& 0.717 \\

C$^2$FG
& $w_0=1.1$
& 5.12
& 88.41
& 209.39
& 0.780
& 0.722 \\

Interval
& $w=2.1$
& \textbf{4.81}
& 89.60
& 218.13
& 0.795
& \textbf{0.747} \\

Ours
& $\lambda=1.45$
& 4.83
& \textbf{90.06}
& \textbf{222.97}
& \textbf{0.797}
& 0.743 \\

\bottomrule
\end{tabular}
\end{table}

\input{05_modern_architectures_rebuttal}

\subsection{Learned Guidance Schedules}
\label{subsec:learned_schedule_analysis}

To interpret the gains in \cref{subsec:external_tradeoff}, we examine the learned schedules for DiT-XL/2 and SD~1.5 (Figure~\ref{fig:learned_schedules}).
As $\lambda$ varies, these schedules change in shape as well as strength, rather than simply rescaling a fixed profile or reproducing the predefined window of Interval Guidance.
\textcolor{black}{Stage-wise ablations in Figure~\ref{fig:interval_ablation} further examine how guidance at different noise levels affects generation.}
\begin{figure}[H]
    \centering

    \begin{subfigure}[t]{0.48\textwidth}
        \centering
        \includegraphics[width=\linewidth]{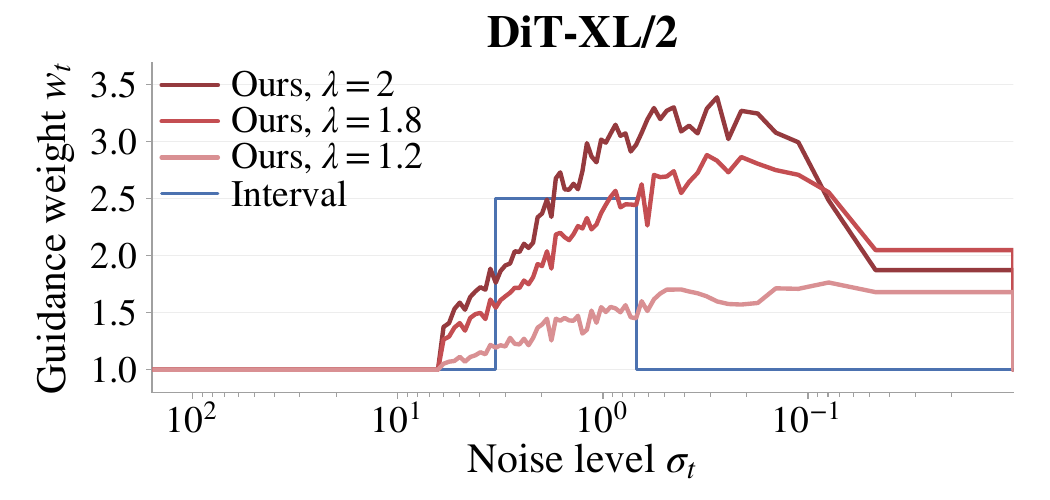}
        \caption{DiT-XL/2 on ImageNet-512.}
        \label{fig:dit_learned_schedule}
    \end{subfigure}
    \hfill
    \begin{subfigure}[t]{0.48\textwidth}
        \centering
        \includegraphics[width=\linewidth]{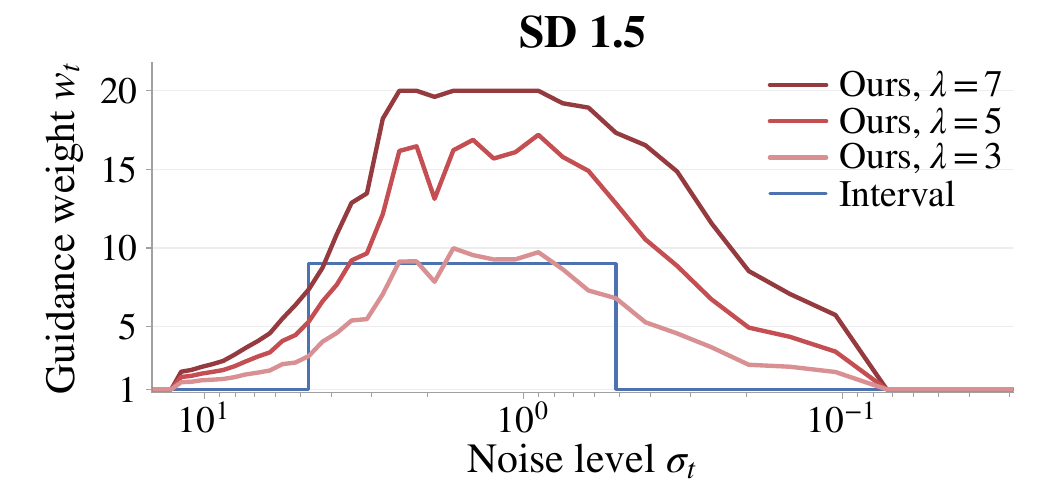}
        \caption{SD~1.5 on COCO.}
        \label{fig:sd15_learned_schedule}
    \end{subfigure}

    \caption{
    Learned guidance schedules across noise levels.
    Left: DiT-XL/2 on ImageNet-512.
    Right: SD~1.5 on COCO.
    Interval Guidance applies a fixed guidance strength over a prescribed
    noise range.
    The horizontal axis follows the reverse trajectory from high to low noise.
    }
    \label{fig:learned_schedules}
\end{figure}

\subsection{Additional Generation Visualizations}
\label{subsec:additional_generation_visualizations}
To better understand how guidance timing affects generation, we analyze both
the denoising trajectory and final samples. Appendix~\ref{subsec:trajectory_visualization}
shows that strong early guidance can introduce an incorrect structure that
persists through later denoising, while our learned schedule uses weaker early
guidance and avoids this error. Appendix~\ref{app:additional_generation} further
shows that strong constant guidance improves condition consistency but leads to less diverse generations, whereas our learned schedules retain clear
condition-specific content while preserving greater variation and composition.
Together, these results illustrate the benefit of adapting guidance across
noise levels.

%% file: 05_modern_architectures_rebuttal.tex
\subsection{\textcolor{black}{Generalization to Modern Architectures}}
\label{subsec:modern_architectures}
{\color{black}

Our initial backbone selection enabled comparisons against guidance schedulers
with established configurations. These baselines use manually chosen
hyperparameters for predefined guidance curves or intervals, whereas our
method optimizes the schedule through its information-theoretic objective.
To further examine modern architectures, we evaluate SiT-XL/2, a
flow-matching/interpolant backbone~\citep{ma2024sit}, and SD3 Medium,
another flow-matching model, against constant CFG.
The additional accuracy--FID and CLIP--FID comparisons, respectively,
show more favorable trade-offs for our learned schedules.
}

%% file: 06_conclusion.tex
\section{Conclusion}

This work formulates the desired consistency--coverage trade-off for classifier-free guidance using a clean endpoint reference and optimizes the sampler-induced distribution toward this target. We derive exact trajectory-level
identities for the consistency and KL terms under the sampler-induced
distribution and characterize how guidance locally affects both terms across
noise levels. These theoretical results motivate a time-dependent guidance schedule rather than a single global guidance strength. Experiments across four diffusion backbones on ImageNet-512 and COCO demonstrate that the learned schedules improve consistency--FID trade-offs over constant and non-uniform guidance baselines. Further analyses of the learned schedules and generation visualizations illustrate how guidance timing influences intermediate denoising dynamics and final sample quality.

%% file: appendix.tex
\section{Limitations and Broader Impacts}

\label{app:limitations_impacts}

\paragraph{Limitations.}
Our method is developed under the VE probability-flow formulation and relies on
score evaluations from pretrained diffusion models. The sampler-induced law
\(q_t^{\mathbf w}\) is available only through generated trajectories, so the
optimizer uses a fixed-trajectory descent direction and limits each proposed
schedule change before checking it on the resulting trajectory. This remains an
approximation to full trajectory-level optimization. The objective is also
estimated with finite samples and Hutchinson trace estimators and can therefore
depend on estimator variance, sampler discretization, and the chosen noise grid.
Moreover, \(\lambda\) specifies an endpoint preference in the objective; its mapping
to evaluation metrics such as FID, CLIP score, or classifier accuracy can be
model dependent and nonlinear. We evaluate four diffusion backbones on
ImageNet-512 and COCO, but broader evaluation across architectures, datasets,
and modalities remains future work.

\paragraph{Broader impacts.}
This work studies how to improve conditional diffusion generation by
learning where to apply classifier-free guidance along the reverse trajectory.
Better guidance schedules may improve the balance between condition
consistency, sample quality, and distributional coverage. At the same time,
improvements to image and text-to-image generation can make synthetic content
more realistic or more closely aligned with specified conditions, which may increase risks such as misleading
generated media or other misuse. We do not release a new pretrained model or
dataset, and our experiments are conducted on standard academic benchmarks using
existing models.

\section{Additional Related Work on Diffusion Time}
\label{app:additional_related_work}

A growing body of work suggests that diffusion time and noise level are
important axes for understanding and controlling generative diffusion models.
Variational diffusion models relate the continuous-time objective to the
signal-to-noise ratio and study noise-schedule optimization
\citep{kingma2021variational}. Prior work on timestep weighting also suggests that diffusion time is not uniform. Changing the weights assigned to different noise levels can substantially affect training and sample quality~\citep{choi2022perception,hang2023efficient}. Related design choices in EDM also emphasize the role of noise parameterization and sampling discretization \citep{karras2022elucidating}. These works focus on
training-time weighting, while our work studies a sampling-time question: how to
allocate classifier-free guidance across the reverse trajectory. 

Several recent works also study how semantic and structural information emerges
along the reverse trajectory. Critical-window analyses show that certain
features can emerge over narrow time intervals of the reverse process
\citep{li2024critical}. Statistical-physics analyses identify distinct
dynamical regimes in diffusion sampling \citep{raya2023spontaneous,
biroli2024dynamical}. Other works measure when semantic information is produced
during generation or analyze how interpretable concepts and semantic features
evolve across denoising timesteps \citep{handke2025measuring,
tinaz2025emergence,kim2025revelio}. These studies motivate treating guidance
location as a meaningful design choice rather than using a single constant
guidance strength at all steps.

Our objective is also related to information-theoretic views of diffusion.
Classical de Bruijn-type identities connect Gaussian smoothing, entropy, and
Fisher information \citep{cover2006elements}, and probability-flow likelihood
formulas use related trajectory-level identities for score-based diffusion
models \citep{song2021scorebased,song2021maximum}. Information-Theoretic
Diffusion further connects diffusion objectives to I-MMSE relations and
information quantities along the noising process \citep{kong2023information}.
In contrast to these works, which primarily analyze diffusion trajectories or
training objectives, we use trajectory-level information quantities to define
and optimize a sampling-time CFG schedule under the actual sampler-induced
distribution.
\section{Additional Discussion of the Gaussian Toy Example}
\label{app:gaussian-toy}

This appendix gives a qualitative discussion of the Gaussian toy example used
in the main text. The purpose is not to derive a closed-form solution for the
CFG-induced endpoint distribution, but to illustrate the terminal mismatch and
the resulting consistency--coverage trade-off.

The Gaussian toy example also illustrates why \(w=1\) need not be the best
empirical choice for the actual sampler. The choice \(w=1\) gives the standard
conditional score field,
\begin{equation}
s_t^w(x_t,y)
=
s_t^{\mathrm{con}}(x_t,y).
\end{equation}
Therefore, if the reverse process were initialized from the exact terminal
conditional distribution \(p_T(x_T\mid y)\), the conditional probability-flow
ODE would transport this distribution back to the clean conditional
distribution \(p_0(x_0\mid y)\).

However, practical sampling initializes the reverse process from the high-noise
Gaussian approximation
\begin{equation}
X_T\sim\mathcal N(0,\sigma_T^2 I),
\label{eq:gaussian_toy_terminal_prior_appendix}
\end{equation}
rather than from the exact class-conditional terminal distribution. In the
Gaussian example, if
\begin{equation}
X_0\mid Y=y \sim \mathcal N(\mu_y,\Sigma),
\end{equation}
then the exact noised conditional distribution is
\begin{equation}
p_T(x_T\mid y)
=
\mathcal N(\mu_y,\Sigma+\sigma_T^2 I).
\label{eq:gaussian_toy_terminal_conditional_appendix}
\end{equation}
When \(\sigma_T\) is large, this distribution is close to the
condition-independent Gaussian initialization in
\eqref{eq:gaussian_toy_terminal_prior_appendix}, but it is not identical to it.
The mismatch includes both a mean mismatch, through \(\mu_y\), and a covariance
mismatch, through \(\Sigma+\sigma_T^2 I\) versus \(\sigma_T^2 I\).

Thus, in this toy example, the reason \(w=1\) may fail to be empirically optimal
is not that the conditional score field is incorrect. Rather, the full sampler
is not exact because the reverse process starts from an approximate,
condition-independent terminal distribution instead of the exact
\(p_T(x_T\mid y)\). Moderate guidance can empirically compensate for this
terminal mismatch by pulling samples toward the corresponding data component.
However, excessively large guidance over-guides the trajectory, increases the
coverage-side KL, and may move samples into low-density regions of the clean
data distribution.

\begin{figure}[t]
    \centering
    \includegraphics[width=0.72\linewidth]{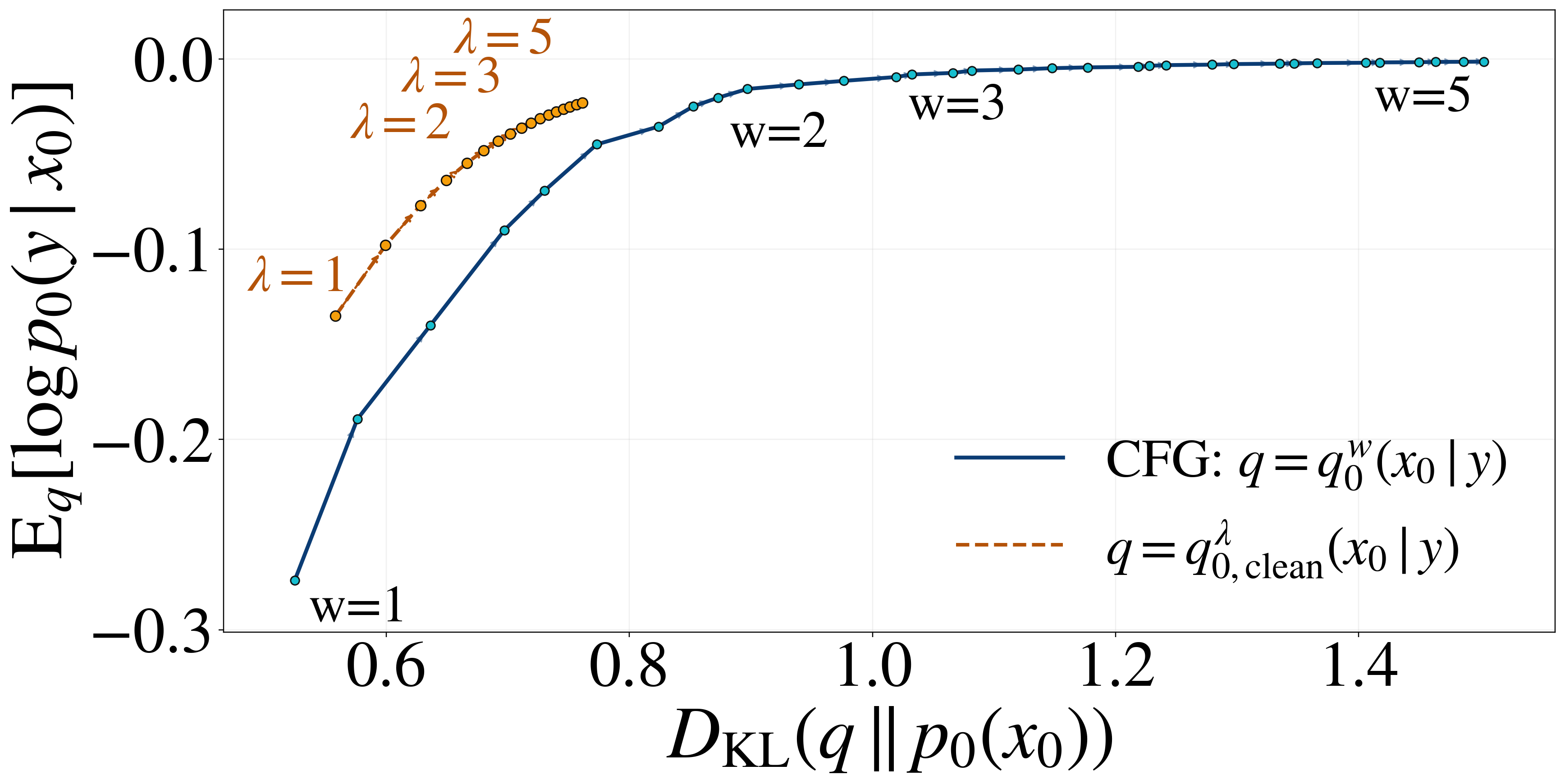}
    \caption{
    Gaussian trade-off between condition consistency and deviation from the
    original clean data distribution. The tilted family provides a natural
    endpoint reference: increasing \(\lambda\) improves condition consistency
    while controlling the deviation from the original clean distribution
    \(p_0(x_0)\).
    }
    \label{fig:gaussian_tradeoff_target}
\end{figure}

Figure~\ref{fig:gaussian_tradeoff_target} connects the Gaussian example to the
information quantities used in our objective. The horizontal axis is the
coverage-side KL
\begin{equation}
\mathrm{KL}\!\left(q\,\|\,p_0(x_0)\right),
\end{equation}
and the vertical axis is the consistency
\begin{equation}
\mathbb E_q[\log p_0(y\mid x_0)].
\end{equation}
The clean endpoint tilted family
\(q_{0,\mathrm{clean}}^\lambda(x_0\mid y)\) traces the desired endpoint
trade-off: increasing \(\lambda\) improves consistency while increasing the
associated KL divergence.

In contrast, the actual CFG-induced distributions \(q_0^w(x_0\mid y)\) move
along a different curve as \(w\) changes. For the same consistency level, CFG
can incur a larger coverage-side KL, and for large \(w\) it can move much
farther from the original clean data distribution. This gap shows that
increasing the sampler guidance weight \(w\) is not equivalent to increasing
the endpoint trade-off parameter \(\lambda\). Therefore, we use
\(q_{0,\mathrm{clean}}^\lambda(x_0\mid y)\) to define the desired
consistency--coverage trade-off, and then tune the actual CFG schedule
\(\mathbf w\) to approach this target using the sampler-induced distribution.

\input{appendix_rebuttal_gaussian}

\section{Details for the Mismatches}
\label{app:mismatch-proof}

\subsection{Clean Endpoint Tilt Versus Noisy Level Tilt}
\label{app:clean_noisy_tilt_mismatch}

We first explain why the clean endpoint tilt and the noisy level tilt are
generally different. The clean endpoint tilt is
\begin{equation}
q_{0,\mathrm{clean}}^w(x_0\mid y)
=
\frac{1}{Z_0(y;w)}
p_0(x_0)p_0(y\mid x_0)^w .
\label{eq:app_clean_tilt}
\end{equation}
After applying the forward noising kernel \(p_{t\mid 0}(x_t\mid x_0)\), its
time \(t\) marginal is
\begin{align}
q_{t,\mathrm{clean}}^w(x_t\mid y)
&=
\int p_{t\mid 0}(x_t\mid x_0)
q_{0,\mathrm{clean}}^w(x_0\mid y)\,dx_0
\nonumber\\
&=
\frac{1}{Z_0(y;w)}
\int p_{t\mid 0}(x_t\mid x_0)
p_0(x_0)p_0(y\mid x_0)^w\,dx_0 .
\label{eq:app_noised_clean_tilt}
\end{align}
Using Bayes' rule under the original forward process,
\begin{equation}
p_0(x_0\mid x_t)
=
\frac{
p_{t\mid 0}(x_t\mid x_0)p_0(x_0)
}{
p_t(x_t)
},
\label{eq:app_posterior_clean_given_noisy}
\end{equation}
we can rewrite \eqref{eq:app_noised_clean_tilt} as
\begin{equation}
q_{t,\mathrm{clean}}^w(x_t\mid y)
=
\frac{p_t(x_t)}{Z_0(y;w)}
\mathbb E_{p_0(x_0\mid x_t)}
\left[
p_0(y\mid x_0)^w
\right].
\label{eq:app_clean_then_noise}
\end{equation}

On the other hand, the noisy level tilt is
\begin{equation}
q_{t,\mathrm{tilt}}^w(x_t\mid y)
=
\frac{1}{Z_t(y;w)}
p_t(x_t)p_t(y\mid x_t)^w .
\label{eq:app_noisy_tilt}
\end{equation}
Under the original joint distribution, \(Y\to X_0\to X_t\) forms a Markov chain, therefore
\begin{equation}
p_t(y\mid x_t)
=
\mathbb E_{p_0(x_0\mid x_t)}
\left[
p_0(y\mid x_0)
\right],
\label{eq:app_noisy_posterior_average}
\end{equation}
we have
\begin{equation}
q_{t,\mathrm{tilt}}^w(x_t\mid y)
=
\frac{p_t(x_t)}{Z_t(y;w)}
\left(
\mathbb E_{p_0(x_0\mid x_t)}
\left[
p_0(y\mid x_0)
\right]
\right)^w .
\label{eq:app_noise_then_tilt}
\end{equation}

Comparing \eqref{eq:app_clean_then_noise} and \eqref{eq:app_noise_then_tilt},
the two constructions differ because
\begin{equation}
\mathbb E_{p_0(x_0\mid x_t)}
\left[
p_0(y\mid x_0)^w
\right]
\neq
\left(
\mathbb E_{p_0(x_0\mid x_t)}
\left[
p_0(y\mid x_0)
\right]
\right)^w
\qquad
\text{in general}.
\label{eq:app_power_expectation_noncommute}
\end{equation}
Thus,
\begin{equation}
q_{t,\mathrm{clean}}^w(x_t\mid y)
\neq
q_{t,\mathrm{tilt}}^w(x_t\mid y)
\qquad
\text{in general}.
\label{eq:app_clean_noisy_mismatch_result}
\end{equation}
Equality can hold in special cases, for example when \(w=0\), when \(w=1\) or when
\(p_0(y\mid X_0)\) is almost surely constant under \(p_0(x_0\mid x_t)\), but it
does not hold in general.

\subsection{Noisy Level Tilt Versus Actual CFG Distribution}
\label{app:noisy_tilt_actual_cfg_law}

The noisy level tilt
\begin{equation}
q_{t,\mathrm{tilt}}^w(x_t\mid y)
\propto
p_t(x_t)p_t(y\mid x_t)^w
\label{eq:app_noisy_tilt_again}
\end{equation}
is useful because its score is exactly the CFG field:
\begin{equation}
\nabla_{x_t}\log q_{t,\mathrm{tilt}}^w(x_t\mid y)
=
s_t^{\mathrm{un}}(x_t)
+
w\left(
s_t^{\mathrm{con}}(x_t,y)-s_t^{\mathrm{un}}(x_t)
\right)
=
s_t^w(x_t,y).
\label{eq:app_tilt_score_cfg_field}
\end{equation}
However, this identity is a statement about the score field at a single noise
level. It does not imply that the family
\(\{q_{t,\mathrm{tilt}}^w(\cdot\mid y)\}_{t\in[0,T]}\) is the distribution path
generated by the CFG probability-flow ODE.

The actual CFG distribution \(q_t^w(x_t\mid y)\) is defined by evolving samples
through the guided ODE, integrated backward from \(t=T\) to \(t=0\):
\begin{equation}
\frac{dX_t}{dt}
=
-\sigma_t\dot{\sigma}_t s_t^w(X_t,y).
\label{eq:app_cfg_ode}
\end{equation}
Therefore \(q_t^w(x_t\mid y)\) is determined by the full transport dynamics and
by the sampling distribution at the \(T\) end. A collection of densities
defined separately at each noise level, such as \(q_{t,\mathrm{tilt}}^w\), need
not satisfy the continuity equation induced by \eqref{eq:app_cfg_ode}.
Consequently,
\begin{equation}
q_t^w(x_t\mid y)
\neq
q_{t,\mathrm{tilt}}^w(x_t\mid y)
\qquad
\text{in general}.
\label{eq:app_path_tilt_mismatch_result}
\end{equation}
Equivalently, even though \(s_t^w\) is the score of the noisy level tilt, it is
not generally the score of the actual CFG law:
\begin{equation}
s_t^w(x_t,y)
\neq
\nabla_{x_t}\log q_t^w(x_t\mid y).
\label{eq:app_field_score_mismatch_result}
\end{equation}
This is why our information-theoretic quantities are evaluated under the
actual CFG law \(q_t^w(x_t\mid y)\), rather than under the noisy level tilt.

\section{Proof of \cref{lem:terminal_clean_objective}}
\label{app:objective-proof}
Assume the relevant densities are positive on the support of
\(q_0^{\mathbf w}(x_0\mid y)\) and the expectations below are finite.
\begin{proof}
By the definition of the clean endpoint reference,
\begin{equation}
q_{0,\mathrm{clean}}^\lambda(x_0\mid y)
=
\frac{1}{Z_\lambda(y)}
p_0(x_0)p_0(y\mid x_0)^\lambda ,
\end{equation}
and therefore
\begin{equation}
\log q_{0,\mathrm{clean}}^\lambda(x_0\mid y)
=
\log p_0(x_0)
+
\lambda \log p_0(y\mid x_0)
-
\log Z_\lambda(y).
\end{equation}
Substituting this identity into the endpoint KL gives
\begin{align}
&\mathbb E_{p(y)}
\mathrm{KL}\!\left(
q_0^{\mathbf w}(x_0\mid y)
\,\|\,
q_{0,\mathrm{clean}}^\lambda(x_0\mid y)
\right)
\nonumber\\
&=
\mathbb E_{p(y)}
\mathbb E_{q_0^{\mathbf w}(x_0\mid y)}
\left[
\log q_0^{\mathbf w}(x_0\mid y)
-
\log q_{0,\mathrm{clean}}^\lambda(x_0\mid y)
\right]
\nonumber\\
&=
\mathbb E_{p(y)}
\mathrm{KL}\!\left(
q_0^{\mathbf w}(x_0\mid y)
\,\|\,
p_0(x_0)
\right)
-
\lambda
\mathbb E_{p(y)q_0^{\mathbf w}(x_0\mid y)}
\left[
\log p_0(y\mid x_0)
\right]
+
\mathbb E_{p(y)}
\left[
\log Z_\lambda(y)
\right].
\end{align}
The last term is a scalar constant for fixed \(\lambda\), because it depends on
\(\lambda\) and the marginal \(p(y)\), but not on the guidance schedule
\(\mathbf w\). Therefore, for fixed \(\lambda\), minimizing the endpoint KL over
\(\mathbf w\) is equivalent to minimizing
\begin{align}
\mathcal L_0(\mathbf w;\lambda)
=
-\lambda
\mathbb E_{p(y)q_0^{\mathbf w}(x_0\mid y)}
\left[
\log p_0(y\mid x_0)
\right]
+
\mathbb E_{p(y)}
\mathrm{KL}\!\left(
q_0^{\mathbf w}(x_0\mid y)
\,\|\,
p_0(x_0)
\right).
\end{align}
This proves the claim.

To derive the trade-off interpretation of \(\lambda\),
We next derive the slope interpretation in
\eqref{eq:lambda_tradeoff_slope}--\eqref{eq:lambda_frontier_slope}.
For each condition \(y\), define
\begin{equation}
Z_y(\lambda)
:=
\int
p_0(x_0)p_0(y\mid x_0)^\lambda
\,dx_0,
\qquad
A_y(\lambda):=\log Z_y(\lambda).
\label{eq:app-lambda-log-normalizer}
\end{equation}
Then
\begin{align}
A_y'(\lambda)
&=
\mathbb E_{q_{0,\mathrm{clean}}^\lambda(x_0\mid y)}
\left[
\log p_0(y\mid x_0)
\right],
\label{eq:app-lambda-first-derivative}
\\
A_y''(\lambda)
&=
\operatorname{Var}_{q_{0,\mathrm{clean}}^\lambda(x_0\mid y)}
\left[
\log p_0(y\mid x_0)
\right]
\ge 0.
\label{eq:app-lambda-second-derivative}
\end{align}
Averaging \eqref{eq:app-lambda-second-derivative} over \(p(y)\) gives
\begin{equation}
\frac{d}{d\lambda}
\mathbb E_{p(y)q_{0,\mathrm{clean}}^\lambda(x_0\mid y)}
\left[
\log p_0(y\mid x_0)
\right]
=
\mathbb E_{p(y)}
\operatorname{Var}_{q_{0,\mathrm{clean}}^\lambda(x_0\mid y)}
\left[
\log p_0(y\mid x_0)
\right].
\label{eq:app-consistency-lambda-derivative}
\end{equation}

Moreover,
\begin{equation}
\log
\frac{
q_{0,\mathrm{clean}}^\lambda(x_0\mid y)
}{
p_0(x_0)
}
=
\lambda\log p_0(y\mid x_0)-A_y(\lambda),
\end{equation}
so
\begin{equation}
\KL\!\left(
q_{0,\mathrm{clean}}^\lambda(x_0\mid y)
\,\|\,
p_0(x_0)
\right)
=
\lambda A_y'(\lambda)-A_y(\lambda).
\label{eq:app-clean-reference-kl-partition}
\end{equation}
Differentiating and using
\eqref{eq:app-lambda-second-derivative} yields
\begin{align}
\frac{d}{d\lambda}
\mathbb E_{p(y)}
\KL\!\left(
q_{0,\mathrm{clean}}^\lambda(x_0\mid y)
\,\|\,
p_0(x_0)
\right)
&=
\lambda\,
\mathbb E_{p(y)}
\operatorname{Var}_{q_{0,\mathrm{clean}}^\lambda(x_0\mid y)}
\left[
\log p_0(y\mid x_0)
\right].
\label{eq:app-deviation-lambda-derivative}
\end{align}
Combining the last two identities gives
\begin{equation}
\frac{d}{d\lambda}
\mathbb E_{p(y)}
\KL\!\left(
q_{0,\mathrm{clean}}^\lambda(x_0\mid y)
\,\|\,
p_0(x_0)
\right)
=
\lambda
\frac{d}{d\lambda}
\mathbb E_{p(y)q_{0,\mathrm{clean}}^\lambda(x_0\mid y)}
\left[
\log p_0(y\mid x_0)
\right].
\label{eq:lambda_tradeoff_slope}
\end{equation}
Whenever the consistency term changes with \(\lambda\), dividing the two
derivatives gives
\begin{equation}
\frac{
d\,
\mathbb E_{p(y)}
\KL\!\left(
q_{0,\mathrm{clean}}^\lambda(x_0\mid y)
\,\|\,
p_0(x_0)
\right)
}{
d\,
\mathbb E_{p(y)q_{0,\mathrm{clean}}^\lambda(x_0\mid y)}
\left[
\log p_0(y\mid x_0)
\right]
}
=
\lambda .
\label{eq:lambda_frontier_slope}
\end{equation}
Since the variance is nonnegative, both consistency and distributional
deviation are \textbf{nondecreasing} in \(\lambda\) for \(\lambda\ge0\).

\end{proof}

\section{Proofs for the Consistency Quantities}
\label{app:consistency-proof}

\paragraph{Standard information-theoretic diffusion along the original VE family.}
Information theory provides a natural lens for analyzing the generation process:
mutual information measures dependence on the condition, while entropy measures the physical spread of the sample distribution.
Under additive Gaussian smoothing, the relevant entropy and mutual-information identities follow from the heat-flow form of de Bruijn's identity and its mutual-information counterpart; see \cite{cover2006elements}.

Let
\begin{equation}
\tau(t):=\sigma_t^2,
\qquad
\dot\tau(t):=\frac{d}{dt}\sigma_t^2.
\end{equation}
Then the conditional entropy $H(X_t\mid Y)$ and unconditional entropy $H(X_t)$ evolve according to
\begin{align}
H(X_t\mid Y)
&=
H(X_T\mid Y)
+\frac12\int_t^T \dot\tau(s)\,
\mathbb E_{p_s(x_s,y)}
\!\left[
\nabla_x\!\cdot s_s^{\mathrm{con}}(x_s,y)
\right]ds,
\label{eq:Ht-cond}
\\
H(X_t)
&=
H(X_T)
+\frac12\int_t^T \dot\tau(s)\,
\mathbb E_{p_s(x_s)}
\!\left[
\nabla_x\!\cdot s_s^{\mathrm{un}}(x_s)
\right]ds.
\label{eq:Ht-uncond}
\end{align}
Consequently, the mutual information $I(X_t;Y)=H(X_t)-H(X_t\mid Y)$ admits the exact representation
\begin{align}
I(X_t;Y)
&=
I(X_T;Y)
-\frac12\int_t^T \dot\tau(s)\,
\mathbb E_{p_s(x_s,y)}
\!\left[
\nabla_x\!\cdot s_s^{\mathrm{con}}(x_s,y)
\right]ds
\nonumber\\
&\quad
+\frac12\int_t^T \dot\tau(s)\,
\mathbb E_{p_s(x_s,y)}
\!\left[
\nabla_x\!\cdot s_s^{\mathrm{un}}(x_s)
\right]ds.
\label{eq:standard-mi}
\end{align}
\textit{Remark.}
For the original additive Gaussian smoothing family induced by $p_0$, integration by parts yields
\begin{equation}
\mathbb E_{p_t}
\!\left[
\nabla_x\!\cdot s_t^{\mathrm{un}}(X_t)
\right]
=
-
\mathbb E_{p_t}
\!\left[
\|s_t^{\mathrm{un}}(X_t)\|^2
\right],
\end{equation}
and analogously, for each fixed $y$,
\begin{equation}
\mathbb E_{p_t(\cdot\mid y)}
\!\left[
\nabla_x\!\cdot s_t^{\mathrm{con}}(X_t,y)
\right]
=
-
\mathbb E_{p_t(\cdot\mid y)}
\!\left[
\|s_t^{\mathrm{con}}(X_t,y)\|^2
\right].
\end{equation}

In this appendix, we derive the consistency identity used in
\cref{lem:consistency_path_identity}. Throughout, we use the VE/additive
Gaussian smoothing parameterization
\begin{equation}
X_t=X_0+\sigma_t Z,
\qquad
Z\sim\mathcal N(0,I_D),
\qquad
Z\perp (X_0,Y),
\end{equation}
and write
\begin{equation}
\tau(t):=\sigma_t^2,
\qquad
\dot\tau(t)=\frac{d}{dt}\sigma_t^2=2\sigma_t\dot\sigma_t .
\end{equation}
The original unconditional and conditional scores are
\begin{equation}
s_t^{\mathrm{un}}(x_t)
:=
\nabla_{x_t}\log p_t(x_t),
\qquad
s_t^{\mathrm{con}}(x_t,y)
:=
\nabla_{x_t}\log p_t(x_t\mid y).
\end{equation}
We also define
\begin{equation}
s_t^{\mathrm{diff}}(x_t,y)
:=
s_t^{\mathrm{con}}(x_t,y)-s_t^{\mathrm{un}}(x_t).
\label{eq:app-score-diff}
\end{equation}
By Bayes' rule,
\begin{equation}
s_t^{\mathrm{diff}}(x_t,y)
=
\nabla_{x_t}\log p_t(y\mid x_t).
\label{eq:app-score-diff-bayes}
\end{equation}

The CFG score field is
\begin{equation}
s_t^{\mathbf w}(x_t,y)
=
(1-w_t)s_t^{\mathrm{un}}(x_t)
+
w_t s_t^{\mathrm{con}}(x_t,y)
=
s_t^{\mathrm{un}}(x_t)
+
w_t s_t^{\mathrm{diff}}(x_t,y),
\label{eq:app-guided-score}
\end{equation}
and the induced probability-flow dynamics are
\begin{equation}
\frac{d}{dt}x_t
=
-\frac12\dot\tau(t)s_t^{\mathbf w}(x_t,y)
=
-\sigma_t\dot\sigma_t s_t^{\mathbf w}(x_t,y).
\label{eq:app-guided-ve-ode}
\end{equation}

\subsection{VE heat-flow identities}

For each fixed \(y\), the conditional density is obtained by Gaussian
smoothing:
\begin{equation}
p_t(x_t\mid y)
=
\int p_0(x_0\mid y)\,
\mathcal N(x_t;x_0,\sigma_t^2 I_D)\,dx_0.
\end{equation}
Similarly,
\begin{equation}
p_t(x_t)
=
\int p_0(x_0)\,
\mathcal N(x_t;x_0,\sigma_t^2 I_D)\,dx_0.
\end{equation}
Therefore both densities satisfy the heat equation in the variance parameter
\(\tau=\sigma_t^2\):
\begin{align}
\partial_t p_t(x_t\mid y)
&=
\frac12\dot\tau(t)\Delta_x p_t(x_t\mid y),
\\
\partial_t p_t(x_t)
&=
\frac12\dot\tau(t)\Delta_x p_t(x_t).
\end{align}
Using
\[
\Delta p = \nabla_x\!\cdot(p\nabla_x\log p)
=
p\left(\nabla_x\!\cdot \nabla_x\log p+\|\nabla_x\log p\|^2\right),
\]
we obtain
\begin{align}
\partial_t \log p_t(x_t\mid y)
&=
\frac12\dot\tau(t)
\left[
\nabla_x\!\cdot s_t^{\mathrm{con}}(x_t,y)
+
\|s_t^{\mathrm{con}}(x_t,y)\|^2
\right],
\label{eq:app-logpt-con-score-form}
\\
\partial_t \log p_t(x_t)
&=
\frac12\dot\tau(t)
\left[
\nabla_x\!\cdot s_t^{\mathrm{un}}(x_t)
+
\|s_t^{\mathrm{un}}(x_t)\|^2
\right].
\label{eq:app-logpt-un-score-form}
\end{align}

\subsection{Consistency identity}

Define
\begin{equation}
C(t)
:=
\mathbb E_{p(y)q_t^{\mathbf w}(x_t\mid y)}
\left[
\log p_t(y\mid x_t)
\right].
\end{equation}
Let
\[
F_t(x_t,y):=\log p_t(y\mid x_t).
\]
By Bayes' rule,
\begin{equation}
F_t(x_t,y)
=
\log p_t(x_t\mid y)-\log p_t(x_t)+\log p(y),
\end{equation}
and therefore
\begin{equation}
\nabla_x F_t(x_t,y)
=
s_t^{\mathrm{con}}(x_t,y)-s_t^{\mathrm{un}}(x_t)
=
s_t^{\mathrm{diff}}(x_t,y).
\label{eq:app-grad-F}
\end{equation}
Using \eqref{eq:app-logpt-con-score-form} and
\eqref{eq:app-logpt-un-score-form}, we also have
\begin{align}
\partial_t F_t(x_t,y)
&=
\frac12\dot\tau(t)
\Big[
\nabla_x\!\cdot s_t^{\mathrm{diff}}(x_t,y)
+
\|s_t^{\mathrm{con}}(x_t,y)\|^2
-
\|s_t^{\mathrm{un}}(x_t)\|^2
\Big]
\nonumber\\
&=
\frac12\dot\tau(t)
\Big[
\nabla_x\!\cdot s_t^{\mathrm{diff}}(x_t,y)
+
\left\langle
s_t^{\mathrm{diff}}(x_t,y),
s_t^{\mathrm{con}}(x_t,y)+s_t^{\mathrm{un}}(x_t)
\right\rangle
\Big].
\label{eq:app-partial-F}
\end{align}

The distribution \(q_t^{\mathbf w}(x_t\mid y)\) is transported by the velocity
\[
b_t^{\mathbf w}(x_t,y)
=
-\frac12\dot\tau(t)s_t^{\mathbf w}(x_t,y).
\]
Equivalently, \(q_t^{\mathbf w}\) satisfies the continuity equation
\[
\partial_t q_t^{\mathbf w}(x_t\mid y)
+
\nabla_x\cdot
\left(
q_t^{\mathbf w}(x_t\mid y)b_t^{\mathbf w}(x_t,y)
\right)
=0.
\]
Therefore, applying the standard transport identity to
\(F_t(x_t,y)=\log p_t(y\mid x_t)\), we have
\[
\frac{d}{dt}
\mathbb E_{q_t^{\mathbf w}(x_t\mid y)}
\left[
F_t(x_t,y)
\right]
=
\mathbb E_{q_t^{\mathbf w}(x_t\mid y)}
\left[
\partial_t F_t(x_t,y)
+
\left\langle
\nabla_x F_t(x_t,y),
b_t^{\mathbf w}(x_t,y)
\right\rangle
\right].
\]
Hence the derivative of \(C(t)\) along the transported distribution is
\begin{align}
\frac{d}{dt}C(t)
&=
\mathbb E_{p(y)q_t^{\mathbf w}(x_t\mid y)}
\left[
\partial_t F_t(x_t,y)
+
\left\langle
\nabla_x F_t(x_t,y),
b_t^{\mathbf w}(x_t,y)
\right\rangle
\right]
\nonumber\\
&=
\frac12\dot\tau(t)
\mathbb E_{p(y)q_t^{\mathbf w}(x_t\mid y)}
\Big[
\nabla_x\!\cdot s_t^{\mathrm{diff}}(x_t,y)
+
\left\langle
s_t^{\mathrm{diff}}(x_t,y),
s_t^{\mathrm{con}}(x_t,y)+s_t^{\mathrm{un}}(x_t)-s_t^{\mathbf w}(x_t,y)
\right\rangle
\Big].
\end{align}
Using
\[
s_t^{\mathbf w}(x_t,y)
=
s_t^{\mathrm{un}}(x_t)+w_t s_t^{\mathrm{diff}}(x_t,y),
\]
this becomes
\begin{align}
\frac{d}{dt}C(t)
&=
\frac12\dot\tau(t)
\mathbb E_{p(y)q_t^{\mathbf w}(x_t\mid y)}
\Big[
\nabla_x\!\cdot s_t^{\mathrm{diff}}(x_t,y)
+
\left\langle
s_t^{\mathrm{diff}}(x_t,y),
s_t^{\mathrm{con}}(x_t,y)
\right\rangle
-
w_t\|s_t^{\mathrm{diff}}(x_t,y)\|^2
\Big].
\label{eq:app-consistency-derivative}
\end{align}
Since \(\frac12\dot\tau(t)=\sigma_t\dot\sigma_t\), integrating
\eqref{eq:app-consistency-derivative} from \(t\) to \(T\) gives
\begin{align}
C(t)
&=
C(T)
-
\int_t^T
\sigma_s\dot\sigma_s
\mathbb E_{p(y)q_s^{\mathbf w}(x_s\mid y)}
\Big[
\nabla_x\!\cdot s_s^{\mathrm{diff}}(x_s,y)
+
\left\langle
s_s^{\mathrm{diff}}(x_s,y),
s_s^{\mathrm{con}}(x_s,y)
\right\rangle
\nonumber\\
&\hspace{6.6cm}
-
w_s\|s_s^{\mathrm{diff}}(x_s,y)\|^2
\Big]ds,
\end{align}
which proves \cref{lem:consistency_path_identity}.

\section{Proofs for the Coverage Quantity}
\label{app:coverage-proof}

We now prove \cref{lem:coverage_kl_identity}. Let
\begin{equation}
D(t)
:=
\mathbb E_{p(y)}
\mathrm{KL}\!\left(
q_t^{\mathbf w}(x_t\mid y)
\,\|\,
p_t(x_t)
\right).
\end{equation}
The guided distribution \(q_t^{\mathbf w}(x_t\mid y)\) is transported by
\begin{equation}
b_t^{\mathbf w}(x_t,y)
=
-\frac12\dot\tau(t)s_t^{\mathbf w}(x_t,y),
\end{equation}
while the original unconditional marginal \(p_t(x_t)\) is transported by
\begin{equation}
b_t^{\mathrm{un}}(x_t)
=
-\frac12\dot\tau(t)s_t^{\mathrm{un}}(x_t).
\end{equation}
The corresponding continuity equations are
\begin{align}
\partial_t q_t^{\mathbf w}
+
\nabla_x\!\cdot(q_t^{\mathbf w}b_t^{\mathbf w})
&=0,
\\
\partial_t p_t
+
\nabla_x\!\cdot(p_t b_t^{\mathrm{un}})
&=0.
\end{align}

Using the transport identity for
\[
D(t)=
\mathbb E_{p(y)q_t^{\mathbf w}}
\left[
\log q_t^{\mathbf w}(x_t\mid y)-\log p_t(x_t)
\right],
\]
we get
\begin{align}
\frac{d}{dt}D(t)
&=
\mathbb E_{p(y)q_t^{\mathbf w}}
\left[
-\nabla_x\!\cdot b_t^{\mathbf w}
-
\partial_t\log p_t
-
\left\langle
b_t^{\mathbf w},
\nabla_x\log p_t
\right\rangle
\right].
\end{align}
From the continuity equation for \(p_t\),
\begin{equation}
\partial_t\log p_t(x_t)
=
-\nabla_x\!\cdot b_t^{\mathrm{un}}(x_t)
-
\left\langle
b_t^{\mathrm{un}}(x_t),
s_t^{\mathrm{un}}(x_t)
\right\rangle .
\end{equation}
Substituting this identity gives
\begin{align}
\frac{d}{dt}D(t)
&=
\mathbb E_{p(y)q_t^{\mathbf w}}
\left[
-\nabla_x\!\cdot b_t^{\mathbf w}
+
\nabla_x\!\cdot b_t^{\mathrm{un}}
+
\left\langle
b_t^{\mathrm{un}}-b_t^{\mathbf w},
s_t^{\mathrm{un}}
\right\rangle
\right]
\nonumber\\
&=
\mathbb E_{p(y)q_t^{\mathbf w}}
\left[
-\nabla_x\!\cdot
\left(
b_t^{\mathbf w}-b_t^{\mathrm{un}}
\right)
-
\left\langle
b_t^{\mathbf w}-b_t^{\mathrm{un}},
s_t^{\mathrm{un}}
\right\rangle
\right].
\end{align}
Here \(w_t\) is a time-dependent scalar schedule, so it is independent of \(x_t\).
Since
\begin{equation}
b_t^{\mathbf w}(x_t,y)-b_t^{\mathrm{un}}(x_t)
=
-\frac12\dot\tau(t)w_t s_t^{\mathrm{diff}}(x_t,y),
\end{equation}
we obtain
\begin{align}
\frac{d}{dt}D(t)
&=
\frac12\dot\tau(t)
\mathbb E_{p(y)q_t^{\mathbf w}(x_t\mid y)}
w_t
\left[
\nabla_x\!\cdot s_t^{\mathrm{diff}}(x_t,y)
+
\left\langle
s_t^{\mathrm{un}}(x_t),
s_t^{\mathrm{diff}}(x_t,y)
\right\rangle
\right].
\label{eq:app-coverage-derivative}
\end{align}
Integrating from \(0\) to \(T\) and using
\(\frac12\dot\tau(t)=\sigma_t\dot\sigma_t\) yields
\begin{align}
D(0)
&=
D(T)
-
\int_0^T
\sigma_t\dot\sigma_t\,
\mathbb E_{p(y)q_t^{\mathbf w}(x_t\mid y)}
w_t
\left[
\nabla_x\!\cdot s_t^{\mathrm{diff}}(x_t,y)
+
\left\langle
s_t^{\mathrm{un}}(x_t),
s_t^{\mathrm{diff}}(x_t,y)
\right\rangle
\right]dt.
\end{align}
This is exactly \cref{lem:coverage_kl_identity}.

\begin{corollary}[Local consistency--coverage trade-off of CFG]
\label{cor:local_cfg_tradeoff}
Under the assumptions of Lemmas~\ref{lem:consistency_path_identity} and
\ref{lem:coverage_kl_identity}, holding the current sampler marginal
\(q_t^{\mathbf w}\) fixed, and whenever the denominator is nonzero,
\begin{equation}
\frac{
\displaystyle
\frac{\partial}{\partial w_t}
\left[
-\frac{d}{dt}
\mathbb E_{p(y)}
\mathrm{KL}\!\left(
q_t^{\mathbf w}(x_t\mid y)
\,\|\,
p_t(x_t)
\right)
\right]
}{
\displaystyle
\frac{\partial}{\partial w_t}
\left[
-\frac{d}{dt}
\mathbb E_{p(y)q_t^{\mathbf w}(x_t\mid y)}
\log p_t(y\mid x_t)
\right]
}
=
-
\frac{
\mathbb E_{p(y)q_t^{\mathbf w}(x_t\mid y)}
\left[
\nabla_x\!\cdot s_t^{\mathrm{diff}}(x_t,y)
+
\left\langle
s_t^{\mathrm{un}}(x_t),
s_t^{\mathrm{diff}}(x_t,y)
\right\rangle
\right]
}{
\mathbb E_{p(y)q_t^{\mathbf w}(x_t\mid y)}
\left[
\left\|
s_t^{\mathrm{diff}}(x_t,y)
\right\|^2
\right]
}.
\label{eq:local_cfg_tradeoff}
\end{equation}
\end{corollary}

\paragraph{Proof of Corollary~\ref{cor:local_cfg_tradeoff}.}
From \eqref{eq:app-consistency-derivative} and
\(s_t^{\mathrm{con}}=s_t^{\mathrm{un}}+s_t^{\mathrm{diff}}\),
the reverse-time consistency rate satisfies
\begin{align}
-\frac{d}{dt}C(t)
&=
-\sigma_t\dot\sigma_t\,
\mathbb E_{p(y)q_t^{\mathbf w}}
\left[
\nabla_x\!\cdot s_t^{\mathrm{diff}}
+
\left\langle s_t^{\mathrm{un}},s_t^{\mathrm{diff}}\right\rangle
+
(1-w_t)\|s_t^{\mathrm{diff}}\|^2
\right].
\end{align}
Holding the current marginal \(q_t^{\mathbf w}\) fixed therefore gives
\begin{equation}
\frac{\partial}{\partial w_t}
\left[-\frac{d}{dt}C(t)\right]
=
\sigma_t\dot\sigma_t\,
\mathbb E_{p(y)q_t^{\mathbf w}}
\left[\|s_t^{\mathrm{diff}}\|^2\right].
\label{eq:app-local-consistency-response}
\end{equation}
Similarly, \eqref{eq:app-coverage-derivative} gives
\begin{equation}
\frac{\partial}{\partial w_t}
\left[-\frac{d}{dt}D(t)\right]
=
-\sigma_t\dot\sigma_t\,
\mathbb E_{p(y)q_t^{\mathbf w}}
\left[
\nabla_x\!\cdot s_t^{\mathrm{diff}}
+
\left\langle s_t^{\mathrm{un}},s_t^{\mathrm{diff}}\right\rangle
\right].
\label{eq:app-local-kl-response}
\end{equation}
Taking the ratio of
\eqref{eq:app-local-kl-response} and
\eqref{eq:app-local-consistency-response} proves
\eqref{eq:local_cfg_tradeoff}.

Under the same boundary conditions used for integration by parts,
\begin{align}
-&\mathbb E_{p(y)q_t^{\mathbf w}}
\left[
\nabla_x\!\cdot s_t^{\mathrm{diff}}
+
\left\langle s_t^{\mathrm{un}},s_t^{\mathrm{diff}}\right\rangle
\right]
\nonumber\\
&=
\mathbb E_{p(y)q_t^{\mathbf w}}
\left[
\left\langle
\nabla_x\log q_t^{\mathbf w}-s_t^{\mathrm{un}},
s_t^{\mathrm{diff}}
\right\rangle
\right],
\end{align}
which gives the alignment interpretation stated in the main text.

\paragraph{Entropy--cross-entropy decomposition.}
The KL term also admits the standard decomposition
\begin{align}
&\mathbb E_{p(y)}
\KL\!\left(
q_0^{\mathbf w}(x_0\mid y)
\,\|\,
p_0(x_0)
\right)
\nonumber\\
&\qquad=
-H_{q_0^{\mathbf w}}(X_0\mid Y)
-
\mathbb E_{p(y)q_0^{\mathbf w}(x_0\mid y)}
\left[
\log p_0(x_0)
\right].
\label{eq:coverage_entropy_cross_entropy}
\end{align}
The first term is the negative conditional entropy of the generated
distribution, while the second is its cross-entropy with the original clean
data distribution. This identity is provided only as an interpretation of the
KL term; the proof of \cref{lem:coverage_kl_identity} above does not rely
on it.

\section{Loss Implementation}
\label{app:loss-implementation}

\subsection{Combining the consistency and coverage identities}

Recall that the endpoint objective is
\begin{equation}
\mathcal L_0(\mathbf w;\lambda)
=
-\lambda C(0)+D(0),
\end{equation}
where
\begin{equation}
C(0)
:=
\mathbb E_{p(y)q_0^{\mathbf w}(x_0\mid y)}
\left[
\log p_0(y\mid x_0)
\right],
\end{equation}
and
\begin{equation}
D(0)
:=
\mathbb E_{p(y)}
\mathrm{KL}\!\left(
q_0^{\mathbf w}(x_0\mid y)
\,\|\,
p_0(x_0)
\right).
\end{equation}
From \cref{lem:consistency_path_identity}, we have
\begin{align}
C(0)
&=
C(T)
-
\int_0^T
\sigma_t\dot\sigma_t\,
\mathbb E_{p(y)q_t^{\mathbf w}(x_t\mid y)}
\left[
A_t(x_t,y)
-
w_t R_t(x_t,y)
\right]dt,
\label{eq:app-C-path-short}
\end{align}
where
\begin{align}
A_t(x_t,y)
&:=
\nabla_x\!\cdot s_t^{\mathrm{diff}}(x_t,y)
+
\left\langle
s_t^{\mathrm{diff}}(x_t,y),
s_t^{\mathrm{con}}(x_t,y)
\right\rangle,
\\
R_t(x_t,y)
&:=
\left\|
s_t^{\mathrm{diff}}(x_t,y)
\right\|^2.
\end{align}
From \cref{lem:coverage_kl_identity}, we also have
\begin{align}
D(0)
&=
D(T)
-
\int_0^T
\sigma_t\dot\sigma_t\,
\mathbb E_{p(y)q_t^{\mathbf w}(x_t\mid y)}
\left[
w_t B_t(x_t,y)
\right]dt,
\label{eq:app-D-path-short}
\end{align}
where
\begin{equation}
B_t(x_t,y)
:=
\nabla_x\!\cdot s_t^{\mathrm{diff}}(x_t,y)
+
\left\langle
s_t^{\mathrm{un}}(x_t),
s_t^{\mathrm{diff}}(x_t,y)
\right\rangle .
\end{equation}
Since
\begin{equation}
s_t^{\mathrm{con}}(x_t,y)
=
s_t^{\mathrm{un}}(x_t)
+
s_t^{\mathrm{diff}}(x_t,y),
\end{equation}
we have
\begin{equation}
A_t(x_t,y)
=
B_t(x_t,y)
+
R_t(x_t,y),
\qquad
\text{or equivalently}
\qquad
B_t(x_t,y)
=
A_t(x_t,y)-R_t(x_t,y).
\label{eq:app-B-A-R-relation}
\end{equation}

Substituting \eqref{eq:app-C-path-short} and
\eqref{eq:app-D-path-short} into
\(\mathcal L_0(\mathbf w;\lambda)=-\lambda C(0)+D(0)\), and dropping the
terminal terms \(-\lambda C(T)+D(T)\) that are shared by schedules using the
same terminal sampling prior, gives the schedule-dependent trajectory objective
\begin{align}
\mathcal L_{0,\mathrm{traj}}(\mathbf w;\lambda)
&=
\int_0^T
\sigma_t\dot\sigma_t\,
\mathbb E_{p(y)q_t^{\mathbf w}(x_t\mid y)}
\left[
\lambda
\left(
A_t(x_t,y)-w_tR_t(x_t,y)
\right)
-
w_tB_t(x_t,y)
\right]dt
\nonumber\\
&=
\int_0^T
\sigma_t\dot\sigma_t\,
\mathbb E_{p(y)q_t^{\mathbf w}(x_t\mid y)}
\left[
\lambda A_t
-
\lambda w_t R_t
-
w_t(A_t-R_t)
\right]dt
\nonumber\\
&=
\int_0^T
\sigma_t\dot\sigma_t\,
\mathbb E_{p(y)q_t^{\mathbf w}(x_t\mid y)}
\left[
(\lambda-w_t)A_t
+
w_t(1-\lambda)R_t
\right]dt .
\label{eq:app-continuous-combined-loss}
\end{align}
This is the continuous trajectory loss that is discretized in the implementation.

\subsection{Fixed-trajectory proposal with adaptive update size}

We now describe the adaptive step-limit update used to optimize our trajectory
objective. The objective itself is always \eqref{eq:app-riemann-loss}; the
adaptive rule only determines how a fixed-trajectory proposal is accepted and
how its update radius is adjusted after rerunning the sampler. We use this
trust-region-style update for all reported SD~1.5, SDXL, and DiT-XL/2
experiments. EDM2-S optimizes the same trajectory objective without the
adaptive-radius mechanism, as described below.

For a sampler with reverse-time grid
\[
T=t_K>\cdots>t_1\approx 0,
\]
we write the stepwise schedule as
\[
\mathbf w=(w_1,\ldots,w_K),
\qquad
w_k=w(t_k).
\]
At optimization iteration \(m\), trajectories
\[
\{x_{t_k}^{(b,m)}\}_{k=1,b=1}^{K,B}
\]
are generated under the current schedule \(\mathbf w^{(m)}\). Holding these
trajectories fixed, define
\begin{align}
A_k^{(m)}
&:=
\frac{1}{B}
\sum_{b=1}^B
\Bigg[
\nabla_x\!\cdot
s_{t_k}^{\mathrm{diff}}(x_{t_k}^{(b,m)},y^{(b)})
+
\left\langle
s_{t_k}^{\mathrm{diff}}(x_{t_k}^{(b,m)},y^{(b)}),
s_{t_k}^{\mathrm{con}}(x_{t_k}^{(b,m)},y^{(b)})
\right\rangle
\Bigg],
\label{eq:app-local-A}
\\
R_k^{(m)}
&:=
\frac{1}{B}
\sum_{b=1}^B
\left\|
s_{t_k}^{\mathrm{diff}}(x_{t_k}^{(b,m)},y^{(b)})
\right\|^2 .
\label{eq:app-local-R}
\end{align}
The implementation applies the same sample averaging and dimension
normalization used by the corresponding loss estimator to these quantities.
The finite-grid objective is
\begin{align}
\widehat{\mathcal L}_{0}(\mathbf w;\lambda)
&=
\sum_{k=1}^K
a_k
\left[
(\lambda-w_k)A_k
+
w_k(1-\lambda)R_k
\right],
\label{eq:app-riemann-loss}
\end{align}
where \(a_k\) is the integration weight defined in
\cref{sec:guidance_schedule_algorithm}.

The fixed-trajectory derivative with respect to \(w_k\) is therefore
\begin{equation}
G_k^{(m)}
=
a_k
\left[
-A_k^{(m)}
+
(1-\lambda)R_k^{(m)}
\right].
\label{eq:step_update_direction}
\end{equation}
Thus, the integration weight is already included in \(G_k^{(m)}\); it is not
multiplied again when forming the schedule update. This derivative captures
only the explicit dependence of the objective on \(w_k\) while the sampled
trajectory is held fixed. It does not differentiate through the dependence of
the sampler-induced trajectory on the schedule.

To form the next candidate, we choose an update
\(\boldsymbol\delta=(\delta_1,\ldots,\delta_K)\) that minimizes the
fixed-trajectory linear change
\begin{equation}
\sum_{k=1}^K G_k^{(m)}\delta_k
\label{eq:app-tr-linear-objective}
\end{equation}
subject to
\begin{equation}
\|\boldsymbol\delta\|_2
\le R_{\mathrm{TR}}^{(m)},
\qquad
w_{\min}
\le
w_k^{(m)}+\delta_k
\le
w_{\max},
\qquad
\delta_k=0
\ \text{for frozen steps}.
\label{eq:app-tr-constraints}
\end{equation}
Here \(R_{\mathrm{TR}}^{(m)}\) limits the total change of the schedule in one
iteration. In the implementation, the constrained update can be written as
\begin{equation}
\delta_k
=
\operatorname{clip}
\left(
-\frac{G_k^{(m)}}{\mu},
\,w_{\min}-w_k^{(m)},
\,w_{\max}-w_k^{(m)}
\right)
\label{eq:app-tr-delta}
\end{equation}
for updatable steps, with \(\delta_k=0\) on frozen steps. The scalar
\(\mu>0\) is found by bisection so that the update satisfies the
\(\ell_2\) constraint.

Let
\begin{equation}
\Delta_{\mathrm{pred}}
=
-\sum_{k=1}^K G_k^{(m)}\delta_k
\label{eq:app-tr-pred}
\end{equation}
be the decrease predicted from the current trajectories. We then rerun the
sampler under the candidate schedule
\(\widetilde{\mathbf w}=\mathbf w^{(m)}+\boldsymbol\delta\) and compute
\begin{equation}
\Delta_{\mathrm{actual}}
=
\widehat{\mathcal L}_{0}(\mathbf w^{(m)};\lambda)
-
\widehat{\mathcal L}_{0}(\widetilde{\mathbf w};\lambda).
\label{eq:app-tr-actual}
\end{equation}
The current and candidate evaluations use the same conditions, initial noise,
and Hutchinson random vectors, while all trajectory-dependent quantities are
recomputed under \(\widetilde{\mathbf w}\). When
\begin{equation}
\Delta_{\mathrm{pred}}
\le
10^{-12}
\max\!\left(
1,
\left|
\widehat{\mathcal L}_{0}(\mathbf w^{(m)};\lambda)
\right|
\right),
\label{eq:app-tr-numerical-guard}
\end{equation}
we do not form a ratio; the current schedule is retained and the allowed update
size is reduced. Otherwise, we compute
\begin{equation}
r
=
\frac{\Delta_{\mathrm{actual}}}
{\Delta_{\mathrm{pred}}}.
\label{eq:app-tr-ratio}
\end{equation}

The candidate schedule is used for the next iteration when
\(\Delta_{\mathrm{actual}}>0\) and \(r\ge 0.1\); otherwise the current schedule
is retained. Independently, the allowed update size for the next iteration is
\begin{equation}
R_{\mathrm{TR}}^{(m+1)}
=
\operatorname{clip}_{[R_{\min},R_{\max}]}
\begin{cases}
0.5R_{\mathrm{TR}}^{(m)}, & r<0.25,\\
1.5R_{\mathrm{TR}}^{(m)}, &
r\ge0.75
\ \text{and}\
\|\boldsymbol\delta\|_2
\ge0.9R_{\mathrm{TR}}^{(m)},\\
R_{\mathrm{TR}}^{(m)}, & \text{otherwise},
\end{cases}
\label{eq:app-tr-radius-update}
\end{equation}
with
\begin{equation}
R_{\min}=R_0/128,
\qquad
R_{\max}=8R_0.
\label{eq:app-tr-radius-bounds}
\end{equation}
The initial radius \(R_0\) is model dependent. In particular,
\(0.1\le r<0.25\) can still update the schedule while reducing the radius for
the next iteration. Each optimization iteration evaluates only one candidate,
so the number of iterations is not the number of schedule changes actually
taken.

\begin{algorithm}[H]
\caption{Guidance Schedule Update with Adaptive Step Limit}
\label{alg:stepwise_guidance_update}
\begin{algorithmic}[1]
\REQUIRE \(\lambda\), initial schedule \(\mathbf w^{(0)}\), bounds
\([w_{\min},w_{\max}]\), initial radius \(R_0\), and frozen sampler steps
\STATE Set \(R_{\mathrm{TR}}^{(0)}\gets R_0\),
\(R_{\min}\gets R_0/128\), and \(R_{\max}\gets 8R_0\)
\FOR{\(m=0,\ldots,M-1\)}
    \STATE Generate current trajectories and estimate
    \(\widehat{\mathcal L}_{0}(\mathbf w^{(m)};\lambda)\)
    \STATE Compute \(G_k^{(m)}\) from \eqref{eq:step_update_direction}
    \STATE Solve \eqref{eq:app-tr-linear-objective}--\eqref{eq:app-tr-constraints}
    for \(\boldsymbol\delta\)
    \STATE Set \(\widetilde{\mathbf w}\gets
    \mathbf w^{(m)}+\boldsymbol\delta\)
    \STATE Rerun the sampler with \(\widetilde{\mathbf w}\) using the same
    conditions, initial noise, and Hutchinson random vectors
    \STATE Compute \(\Delta_{\mathrm{pred}}\) and
    \(\Delta_{\mathrm{actual}}\)
    \IF{\(\Delta_{\mathrm{pred}}\) satisfies
    \eqref{eq:app-tr-numerical-guard}}
        \STATE \(\mathbf w^{(m+1)}\gets\mathbf w^{(m)}\)
        \STATE \(R_{\mathrm{TR}}^{(m+1)}
        \gets \max(R_{\min},0.5R_{\mathrm{TR}}^{(m)})\)
    \ELSE
        \STATE Compute
        \(r\gets\Delta_{\mathrm{actual}}/\Delta_{\mathrm{pred}}\)
        \IF{\(\Delta_{\mathrm{actual}}>0\) and \(r\ge0.1\)}
            \STATE \(\mathbf w^{(m+1)}\gets\widetilde{\mathbf w}\)
        \ELSE
            \STATE \(\mathbf w^{(m+1)}\gets\mathbf w^{(m)}\)
        \ENDIF
        \STATE Update \(R_{\mathrm{TR}}^{(m+1)}\) using
        \eqref{eq:app-tr-radius-update}
    \ENDIF
\ENDFOR
\STATE \textbf{return} \(\mathbf w^{(M)}\)
\end{algorithmic}
\end{algorithm}

\input{appendix_fixed_trajectory_short}
\input{appendix_rebuttal_optimization}

\section{Experiment Details}
\label{app:implementation_details}

This appendix provides the model-specific generation settings, baseline
configurations, optimization settings, evaluation protocols, and supplementary
results for the experiments in \cref{sec:experiments}. Unless stated otherwise,
all methods compared within the same backbone use the same conditioning inputs,
initial noise, sampler, and evaluation pipeline.

\subsection{Experimental Configurations}
\label{app:baseline_settings}
\noindent\textbf{Baseline settings and comparison protocol.}
Unless otherwise noted, we use the settings from the original papers,
adjusting them for the model or sampler when needed.
The Interval Guidance paper does not report settings for SD~1.5,
so we use a middle-stage guidance window based on prior stage-wise analyses
of CFG at different noise levels~\citep{jin2025stage}.
For each baseline, we keep the schedule shape and other hyperparameters
fixed and vary only the overall guidance strength to obtain a trade-off curve.
Within each backbone, all methods use the same sampler, conditioning inputs,
number of samples, and evaluation procedure.
We also reuse the same initial noise where applicable.

For the initial schedule $\mathbf w^{(0)}$ required by Algorithm~\ref{alg:stepwise_guidance_update},
we use the constant schedule $w=1$ as the default initialization.

For our learned schedules, we exclude selected steps near the two ends of the
noise grid from optimization and keep them at $w=1$. Under a finite sampler
discretization, guidance at the highest-noise end is weakly responsive, whereas
the lowest-noise end can be overly sensitive to discretization and score
estimation error. Optimizing these endpoint steps can therefore add computation
without a reliable benefit and can amplify numerical error. The exact excluded
indices depend on the model and sampler discretization.
\paragraph{Stable Diffusion 1.5.}
We evaluate SD~1.5 on 5,000 COCO val2017 image--caption pairs, using the first
caption associated with each image in a fixed ordering. Generation uses
40 deterministic Euler steps with a Karras noise schedule (\(\rho=3\)) at
\(512\times512\) resolution. Constant CFG uses
\(w\in\{1,3,5,7,9\}\). Interval Guidance sets \(w=1\) outside
\(0.2\le t/999\le0.8\) (generation steps 14--29) and sweeps the active
guidance strength over \(\{5,7,9,11,16,21\}\). C$^2$FG uses rate \(0.2\)
and sweeps \(w_0\in\{3,5,7,9\}\). For \(\beta\)-CFG, we fix
\(a=b=2\) and \(\gamma=0\), so the schedule shape is
\(f(u)=6u(1-u)\), and sweep its amplitude over \(\{3,5,7,9\}\).

For our method, optimization uses 64 fixed prompt contexts; each iteration
uses 8 contexts and 4 Rademacher Hutchinson probes. All reported SD~1.5
schedules use the adaptive step-limit update in
Appendix~\ref{app:loss-implementation}, with
\(\lambda\in\{3,5,7,9,\ldots\}\), where the ellipsis denotes additional
settings used for the trade-off curve, and 15 optimization iterations. On a single
NVIDIA B200 GPU, these 15 iterations take approximately 6.2 minutes. The
optimized weights are constrained to \([1,20]\).

\paragraph{SDXL.}
We evaluate SDXL on the same 5,000 COCO val2017 image--caption pairs.
Generation uses 40 deterministic Euler/EDM steps with a Karras noise schedule
(\(\rho=3\)) at \(1024\times1024\) resolution. Constant CFG uses
\(w\in\{1,3,5,7,9\}\). The Interval baseline applies guidance over
\(\sigma\in(0.28,5.42]\) (generation steps 13--32) and sweeps the active
guidance strength over \(\{5,9,13,17\}\); this interval follows the published
example setting rather than a COCO-specific re-optimization. C$^2$FG uses rate
\(0.2\) with target mean guidance in \(\{3,5,7\}\). For \(\beta\)-CFG, we
again fix \(a=b=2\) and \(\gamma=0\) and use target mean guidance in
\(\{3,5,7,9\}\).

For our method, optimization uses 64 fixed prompt contexts, 8 contexts per
iteration, and 4 Rademacher Hutchinson probes. All reported SDXL schedules use
the adaptive step-limit update in Appendix~\ref{app:loss-implementation}, with
\(\lambda\in\{2,3,5,6,7,8,9\}\) and 15 optimization iterations. On a
single NVIDIA B200 GPU, these 15 iterations take approximately 62.4--66.4
minutes. The optimized weights are constrained to \([1,20]\).

\paragraph{DiT-XL/2.}
We use the class-conditional DiT-XL/2 \(512\times512\) checkpoint and evaluate
10,000 generated images per operating point. Generation uses 100 deterministic
DDIM steps with \(\eta=0\) and \texttt{clip\_denoised=False}. Target labels are
sampled with a fixed random seed and are reused across methods together with
the initial noise; they are not explicitly class-balanced. Constant CFG uses
\(w\in\{1,1.25,1.5,1.75,2,3\}\). Interval Guidance sets \(w=1\) outside the
mapped model-time interval \(201<t\le502\), corresponding to zero-based DDIM
steps 50--79, and sweeps the active guidance strength over
\(\{1.5,2.5,3.5,4.5,5.5\}\). This interval is mapped from the 250-step
configuration in the original method rather than copied unchanged. C$^2$FG
uses rate \(\ln 2\) and sweeps
\(w_0\in\{0.8,1,1.1,1.2,1.4,2,3\}\).

Our DiT schedules use the TR update in Appendix~\ref{app:loss-implementation}
with \(\lambda\in\{1.2,1.5,1.6,1.8,1.9,2,3,4\}\). Each optimization iteration
uses 8 fresh trajectories and 8 Gaussian Hutchinson probes. We use 10
iterations for \(\lambda\in\{1.5,3\}\) and 7 iterations for the remaining
values. A 7-iteration run takes approximately 132.3--132.4 minutes on a single
NVIDIA B200 GPU. The schedule is represented on the original 250-step model-time grid
and interpolated by model time to the 100-step DDIM generation grid.
Optimized weights are constrained to \([1,5]\).

\paragraph{EDM2-S.}
We use the conditional EDM2-S \(512\times512\) checkpoint together with the
EDM2-XS unconditional checkpoint and apply ordinary CFG (not autoguidance).
Each operating point is evaluated with 10,000 generated images. Generation
uses the deterministic 32-step Heun sampler with \(S_{\mathrm{churn}}=0\),
\(\sigma_{\min}=0.002\), \(\sigma_{\max}=80\), and \(\rho=7\); the same
guidance value is used in the predictor and corrector evaluations. Target
labels and initial noise are fixed within this backbone and reused across
methods. Constant CFG uses
\(w\in\{1,1.2,1.4,1.6,1.8,2,3,4\}\). Interval Guidance sets \(w=1\) outside
steps 15--21, approximately \(\sigma\in(0.28,2.90]\), and sweeps the active
guidance strength over \(\{1.5,2.1,3.5,4,5,6\}\). C$^2$FG uses rate \(0.5\)
and sweeps \(w_0\in\{0.9,1,1.1,1.4,1.7,2,3,4\}\); this rate is adapted from
the published EDM2 autoguidance setting to ordinary CFG and is therefore not
an exact reproduction of that configuration.

Our EDM2-S runs use
\(\lambda\in\{1.2,1.45,1.8,2,2.5,3\}\).
Optimization uses a fixed pool of 512 candidate trajectories, 8 trajectories
per iteration, 8 Gaussian Hutchinson probes, and 60 iterations with update step
size \(\eta=0.012\). On a single NVIDIA B200 GPU, these 60 iterations take
approximately 13.5--15.6 minutes.
These runs optimize the same trajectory objective but use the
fixed-trajectory update without the adaptive-radius mechanism. EDM2-S has only
32 sampling steps, and we observed consistently high proposal acceptance in
this setting, so the additional radius adaptation was unnecessary in practice.

The wall-clock cost depends on both model size and the number of sampling
steps. For example, generating 10,000 samples on a single NVIDIA B200 GPU
takes approximately 12 hours for DiT-XL/2 but about 1 hour for EDM2-S.

\subsection{50K-Sample ImageNet Validation}
\label{app:imagenet_50k_validation}

The main paper uses 10,000 generated samples per operating point for the dense
ImageNet trade-off sweeps. This keeps the comparison computationally tractable
when many operating points are required for every method. The cost is
particularly high for DiT-XL/2, which uses 100-step DDIM sampling: generating
10,000 samples for a single setting takes approximately 12 hours on one
NVIDIA B200 GPU. We therefore use 10K samples for the main trade-off curves and
perform an additional 50K evaluation on selected EDM2-S settings to check
whether increasing the evaluation sample size materially changes the relative
comparison.

Table~\ref{tab:edm2s_50k_validation} reports the EDM2-S 50K evaluation at the
same operating points used in the main comparison. FID, top-1 classification
accuracy, and IS are computed from 50,000 generated samples. Precision, recall,
and F-score continue to use the fixed 10,000-generated versus 10,000-real
protocol. The completed 50K results preserve the main comparison: Interval
Guidance has slightly lower FID than Ours, while Ours has higher classification
accuracy and IS. Both substantially improve over constant CFG and C$^2$FG in
FID at these operating points. This supports using 10K samples for the dense
trade-off sweeps in the main paper
while reserving 50K evaluation for selected robustness checks.

\begin{table*}[t]
\caption{
EDM2-S validation with 50,000 generated samples for FID, classification
accuracy, and IS. Precision, recall, and F-score use the fixed 10K-generated
versus 10K-real protocol.
}
\label{tab:edm2s_50k_validation}
\centering
\setlength{\tabcolsep}{5.0pt}
\renewcommand{\arraystretch}{1.12}
\begin{tabular}{l c c c c c c c}
\toprule
\textbf{Method}
& \textbf{Setting}
& \textbf{FID $\downarrow$}
& \textbf{Acc. (\%) $\uparrow$}
& \textbf{IS $\uparrow$}
& \textbf{Prec. $\uparrow$}
& \textbf{Rec. $\uparrow$}
& \textbf{F-score $\uparrow$} \\
\midrule
Constant CFG
& $w=1.4$
& 2.29
& 88.470
& 268.600
& 0.782
& 0.717
& 0.748 \\

C$^2$FG
& $w_0=1.1$
& 2.10
& 88.140
& 267.720
& 0.780
& 0.722
& 0.750 \\

Interval
& $w=2.1$
& \textbf{1.68}
& 89.400
& 282.970
& 0.795
& \textbf{0.747}
& 0.770 \\

Ours
& $\lambda=1.45$
& 1.70
& \textbf{89.940}
& \textbf{289.880}
& \textbf{0.797}
& 0.743
& 0.769 \\
\bottomrule
\end{tabular}
\end{table*}

For the EDM2-S adaptive schedule in
Table~\ref{tab:edm2s_50k_validation}, optimization uses the same setting as
the corresponding main-paper operating point: a fixed pool of 512 trajectories,
60 optimization iterations, update step size $\eta=0.012$, no trust-region
update, and the same endpoint-step exclusion described above.

\subsection{Evaluation Protocols and Metrics}
\label{app:evaluation_metrics}

\noindent\textbf{Evaluation protocol.}
We evaluate all competing guidance schedules under a matched protocol within
each diffusion backbone. In particular, methods compared on the same backbone
use the same sampler, conditioning inputs, number of sampling steps, and
evaluation pipeline, with the same target conditions and initial noise reused
across methods where applicable. Baseline methods sweep their corresponding
guidance-strength parameter while keeping the schedule form fixed, whereas
our operating points are obtained from schedules optimized with different
values of $\lambda$. Since these parameters have different meanings across
methods, we compare the resulting evaluation metric trade-offs rather than
matching their numerical parameter values.

For text-to-image generation, we evaluate SD~1.5 and SDXL on COCO using FID
and CLIP similarity. For class-conditional generation, we evaluate DiT-XL/2
and EDM2-S on ImageNet-512 using FID, top-1 classification accuracy,
Inception Score (IS), precision, and recall. These evaluation metrics are used
only for evaluation and are not included in the schedule-optimization
objective.

The evaluation implementation is kept fixed across methods within each
backbone. However, DiT-XL/2 and EDM2-S follow their respective Inception
evaluation pipelines rather than a single shared implementation. We therefore
use the ImageNet metrics primarily to compare guidance methods within a fixed
backbone and do not use absolute cross-backbone differences in IS as evidence
for one method over another.

\paragraph{Fr\'echet Inception Distance (FID).}
FID measures the discrepancy between the generated and reference image
distributions in Inception feature space~\citep{heusel2017gans}. Lower FID
indicates closer agreement with the reference distribution and reflects both
sample fidelity and distributional coverage. For the ImageNet experiments,
DiT-XL/2 uses the ADM TensorFlow Inception evaluation pipeline, whereas
EDM2-S uses the NVIDIA PyTorch Inception detector
\texttt{inception-2015-12-05.pkl}. For the COCO experiments, both SD~1.5 and
SDXL use the NVIDIA/StyleGAN Inception-v3 detector
\texttt{inception-2015-12-05.pkl} and its 2048-dimensional features. Generated
and reference images are resized while preserving aspect ratio and
center-cropped to \(512\times512\) before feature extraction; FID is then
computed from the empirical feature means and covariances using the standard
matrix-square-root formulation. All methods within a given backbone are
evaluated using the same FID implementation.

\paragraph{Classification accuracy.}
For both DiT-XL/2 and EDM2-S, we measure condition consistency using the same
torchvision ResNet-50 classifier with
\texttt{ResNet50\_Weights.IMAGENET1K\_V2}
(\texttt{resnet50-11ad3fa6.pth}). Images are resized so that the shorter side
is 232 pixels, center-cropped to \(224\times224\), and normalized using the
ImageNet preprocessing associated with these weights. We compare the top-1
predicted class with the target class used for generation and report the mean
top-1 accuracy in percent. Higher accuracy therefore indicates stronger
agreement between the generated sample and its conditioning class.

\paragraph{CLIP similarity.}
For both SD~1.5 and SDXL, we compute text--image alignment using OpenCLIP with
the ViT-H/14 model (\texttt{ViT-H-14}) pretrained with the
\texttt{laion2b\_s32b\_b79k} weights. Image and text embeddings are
independently \(\ell_2\)-normalized, and we average their cosine similarity
over generated image--caption pairs. The reported CLIP value is \(100\) times
this mean cosine similarity; no learned logit scale is applied. Higher values
therefore indicate stronger alignment with the conditioning caption.

\paragraph{Inception Score (IS).}
For DiT-XL/2, IS is computed using the ADM evaluator
\texttt{Evaluator.compute\_inception\_score} with the TensorFlow
\texttt{classify\_image\_graph\_def.pb} Inception-v3 model. The computation
reuses the Inception outputs produced during the corresponding FID evaluation,
so no additional image generation is required.

For EDM2-S, we use NVIDIA's PyTorch
\texttt{inception-2015-12-05.pkl} detector, which is also used by our EDM2
FID evaluation. We convert the detector outputs to softmax probabilities
without adding the output-layer bias and compute the standard Inception Score,
\begin{equation}
\mathrm{IS}
=
\exp\!\left(
\mathbb{E}_{x}
\left[
D_{\mathrm{KL}}\!\left(
p(y\mid x)\,\|\,p(y)
\right)
\right]
\right).
\end{equation}
Thus, both evaluations use Inception-v3 classifiers without an output-layer
bias, but they rely on different model files and software implementations
(TensorFlow for DiT-XL/2 and PyTorch for EDM2-S). We have not assumed numerical
equivalence between the two implementations and therefore interpret IS
comparisons within, rather than across, backbones.

\paragraph{Precision and recall.}
For both DiT-XL/2 and EDM2-S, we compute improved precision and recall
\citep{kynkaanniemi2019improved} in the same DINOv2 feature space. We use
DINOv2 ViT-L/14 (\texttt{dinov2\_vitl14},
\texttt{dinov2\_vitl14\_pretrain.pth}) to extract 1024-dimensional features
from images resized to \(224\times224\) and normalized with ImageNet
statistics. Precision and recall are then computed with our project-local
torch-fidelity implementation of \texttt{prc\_features\_to\_metric}, using
Euclidean distances and \(k=3\) nearest-neighbor manifolds. Precision measures
the fraction of generated samples supported by the real-data manifold and
primarily reflects sample fidelity, whereas recall measures the fraction of
the real-data manifold covered by generated samples and primarily reflects
distributional coverage.

\paragraph{Metric roles.}
In our analysis, classification accuracy and CLIP similarity measure
\emph{condition consistency}; recall measures \emph{distributional coverage};
precision measures \emph{sample fidelity}; and FID summarizes overall
distributional agreement between generated and reference images. IS provides
an additional class-conditional generation metric.

\subsection{Additional Quantitative Results}
\label{app:additional_external_evaluation}

\paragraph{Precision--recall trade-offs.}
Figure~\ref{fig:additional_pr_curves} reports the full precision--recall
trade-offs for DiT-XL/2 and EDM2-S. These curves provide a complementary view
of the minimum-FID operating points reported in
Table~\ref{tab:imagenet_best_fid}; the methods are less clearly separated in
precision--recall space than in the classification accuracy--FID comparisons
in the main text.

\begin{figure*}[t]
    \centering
    \begin{subfigure}[t]{0.48\textwidth}
        \centering
        \includegraphics[width=\linewidth]{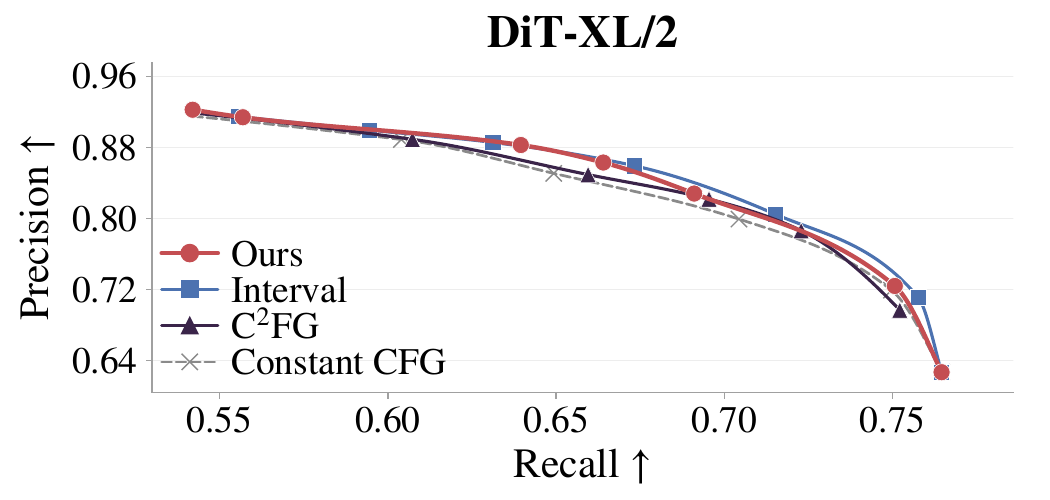}
        \caption{DiT-XL/2.}
        \label{fig:dit_pr_appendix}
    \end{subfigure}
    \hfill
    \begin{subfigure}[t]{0.48\textwidth}
        \centering
        \includegraphics[width=\linewidth]{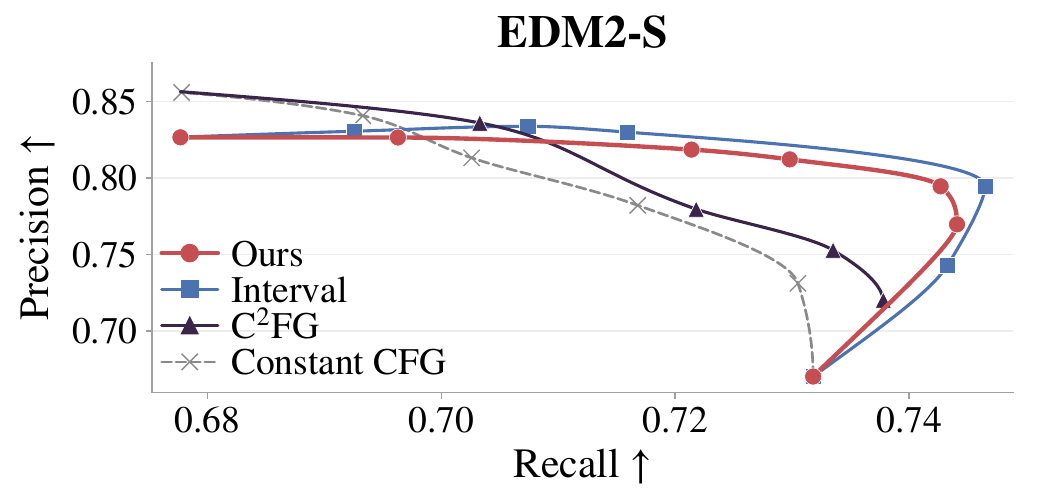}
        \caption{EDM2-S.}
        \label{fig:edms_pr_appendix}
    \end{subfigure}
    \caption{
    Precision--recall trade-offs on ImageNet-512.
    Each curve is obtained by varying the corresponding method-specific
    guidance parameter.
    }
    \label{fig:additional_pr_curves}
\end{figure*}

\paragraph{Objective trade-off.}
Figure~\ref{fig:internal_tradeoff_appendix} evaluates the learned schedules
using the consistency and KL quantities that define our trajectory
objective. The learned schedules improve the measured objective trade-off over
constant guidance, showing that the schedule optimization changes the criterion
it is designed to optimize. We use the CLIP--FID and classification accuracy--FID curves in the
main text as the primary evaluation because those metrics are not used during
schedule optimization.

\begin{figure}[t]
    \centering
    \includegraphics[width=0.68\linewidth]{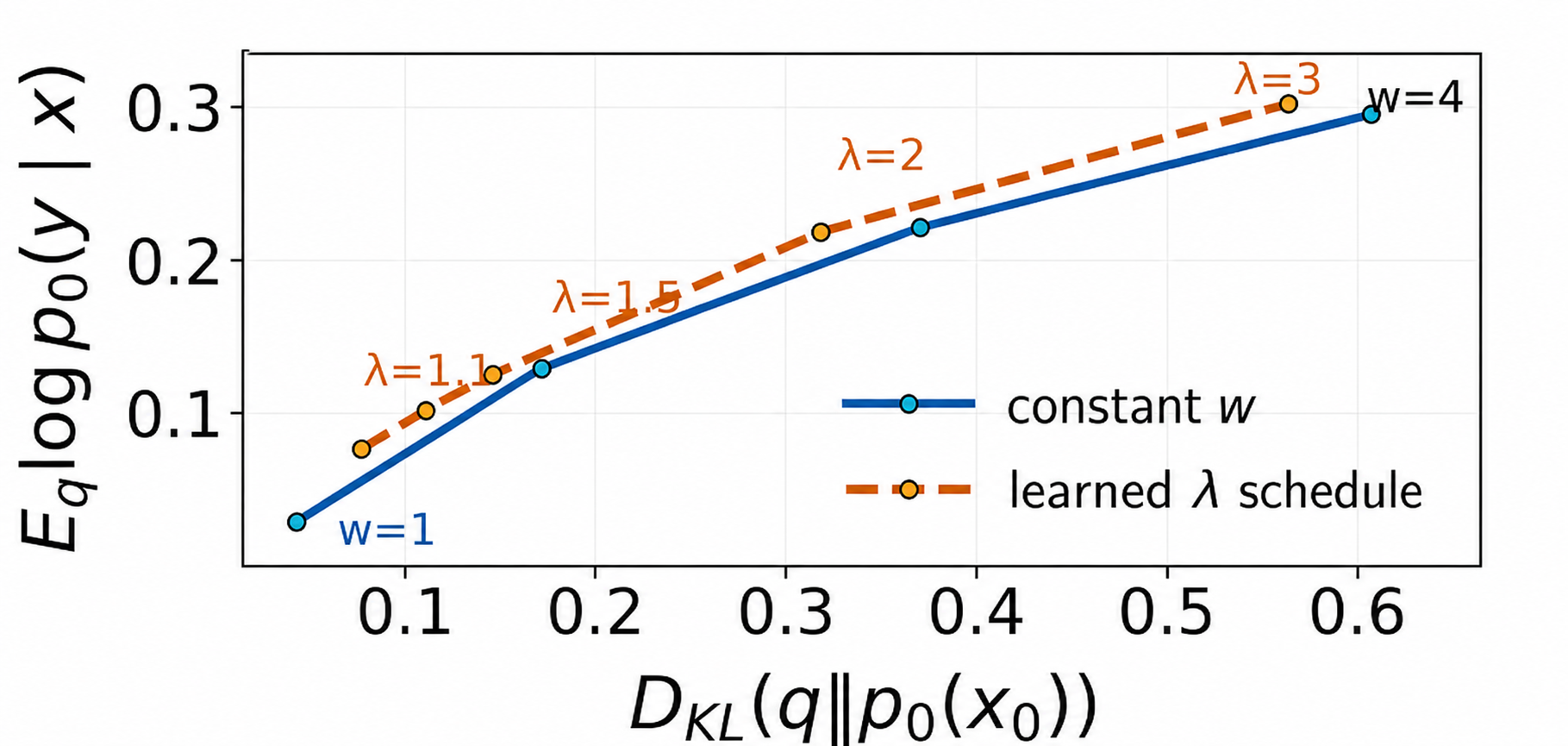}
    \caption{
    Objective trade-off between consistency and the KL term for learned schedules
    and constant guidance on the EDM experiment.
    }
    \label{fig:internal_tradeoff_appendix}
\end{figure}

\input{appendix_rebuttal_experiments}

\subsection{Additional Schedule Analysis}
\label{app:additional_schedule_analysis}

Figure~\ref{fig:additional_learned_schedules} complements
Figure~\ref{fig:learned_schedules} with the learned schedules for the other
two backbones.

\begin{figure*}[t]
    \centering
    \begin{subfigure}[t]{0.45\textwidth}
        \centering
        \includegraphics[width=\linewidth]{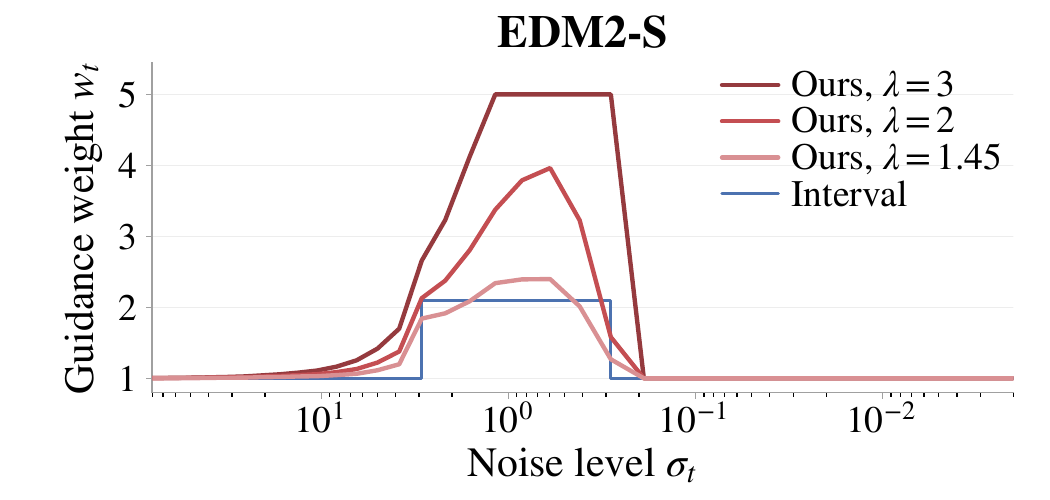}
        \caption{EDM2-S on ImageNet-512.}
        \label{fig:edm2s_learned_schedule_appendix}
    \end{subfigure}
    \hfill
    \begin{subfigure}[t]{0.45\textwidth}
        \centering
        \includegraphics[width=\linewidth]{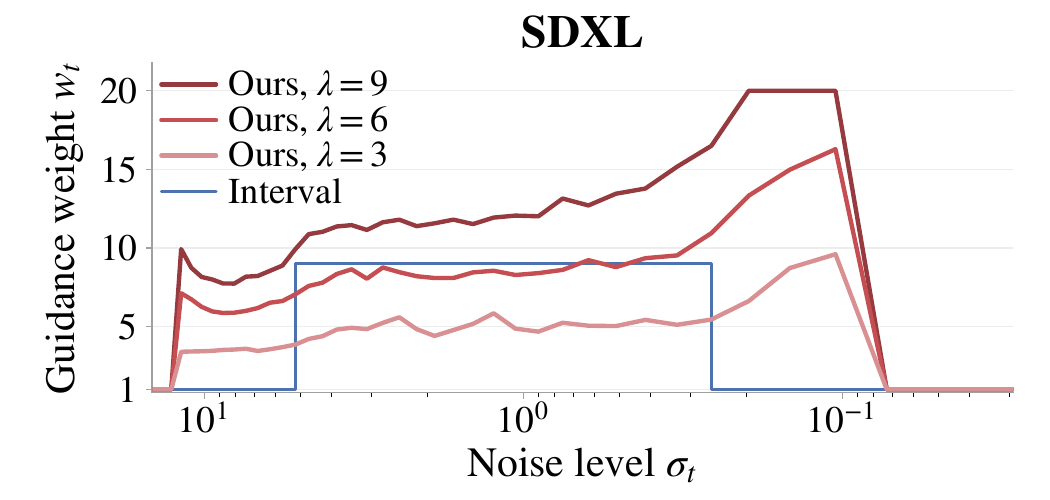}
        \caption{SDXL on COCO.}
        \label{fig:sdxl_learned_schedule_appendix}
    \end{subfigure}
    \caption{
    Additional learned guidance schedules across noise levels.
    Left: EDM2-S with $\lambda\in\{1.45,2,3\}$.
    Right: SDXL with $\lambda\in\{3,6,9\}$.
    Interval Guidance is shown for comparison.
    The horizontal axis follows the reverse trajectory from high to low noise.
    }
    \label{fig:additional_learned_schedules}
\end{figure*}

\subsection{Guidance Timing and Trajectory Analysis}
\label{subsec:trajectory_visualization}

\paragraph{Stage-wise guidance ablation.}
To isolate the effect of guidance location, we activate guidance only within
one noise range (high, middle, or low noise) and set \(w_t=1\) elsewhere.
Larger \(\sigma\)-indices correspond to higher noise levels. For EDM with
32 sampling steps, we use interval guidance \(w=3\) on the high-noise interval
\(\sigma_{32}\)--\(\sigma_{20}\), the middle-noise interval
\(\sigma_{16}\)--\(\sigma_{10}\), or the low-noise interval
\(\sigma_{9}\)--\(\sigma_{4}\). For SDXL with 40 sampling steps, we use
interval guidance \(w=9\) on the high-noise interval
\(\sigma_{40}\)--\(\sigma_{28}\), the middle-noise interval
\(\sigma_{27}\)--\(\sigma_{8}\), or the low-noise interval
\(\sigma_{7}\)--\(\sigma_{1}\). Figure~\ref{fig:interval_ablation} shows that
changing the location of guidance can alter both global structure and local
image details, providing additional qualitative evidence that guidance timing
matters.

\begin{figure*}[t]
    \centering
    \begin{subfigure}[t]{0.45\textwidth}
        \centering
        \includegraphics[width=\linewidth]{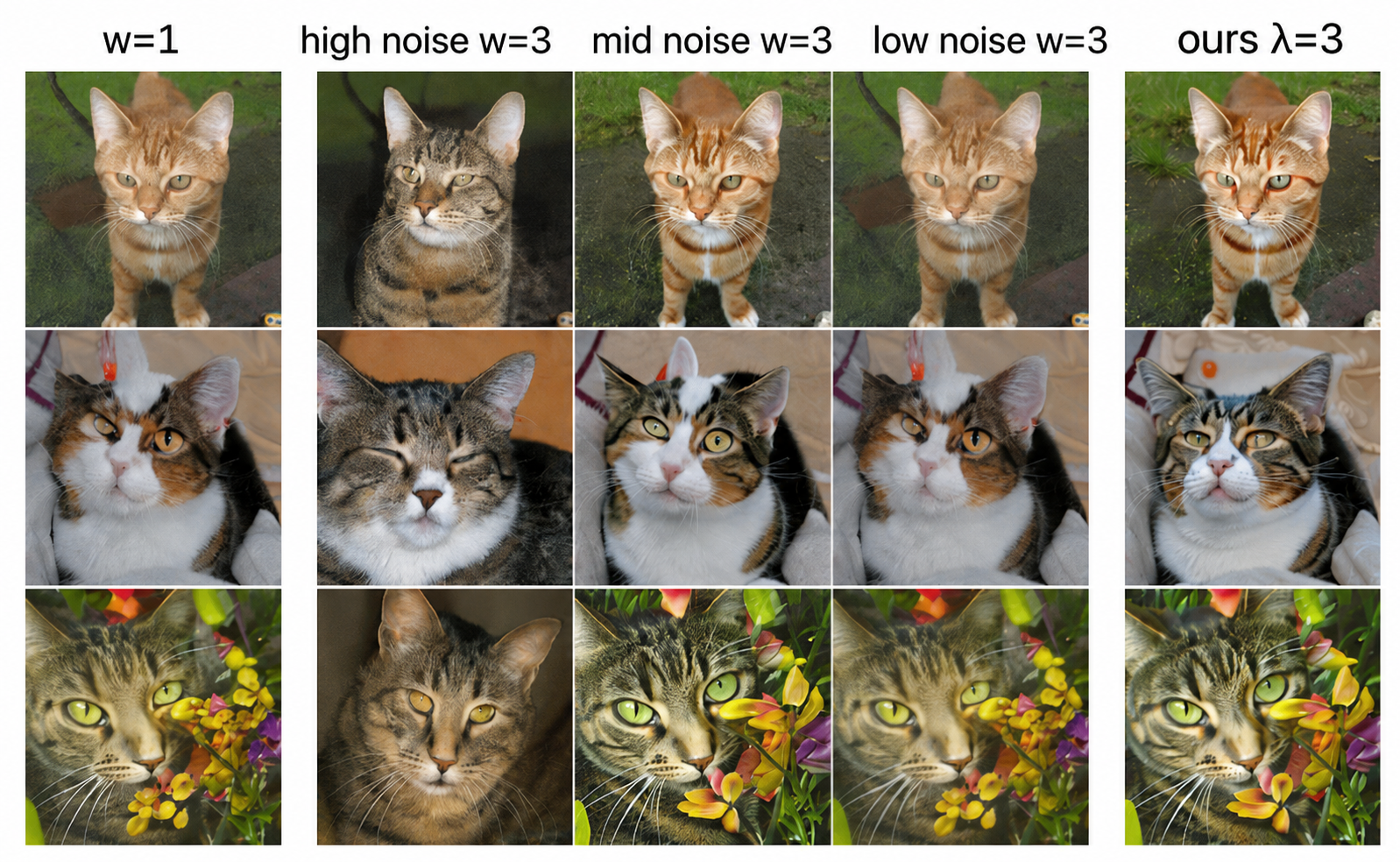}
        \caption{ImageNet class 281, tabby cat.}
        \label{fig:interval_ablation_imagenet}
    \end{subfigure}
    \hfill
    \begin{subfigure}[t]{0.45\textwidth}
        \centering
        \includegraphics[width=\linewidth]{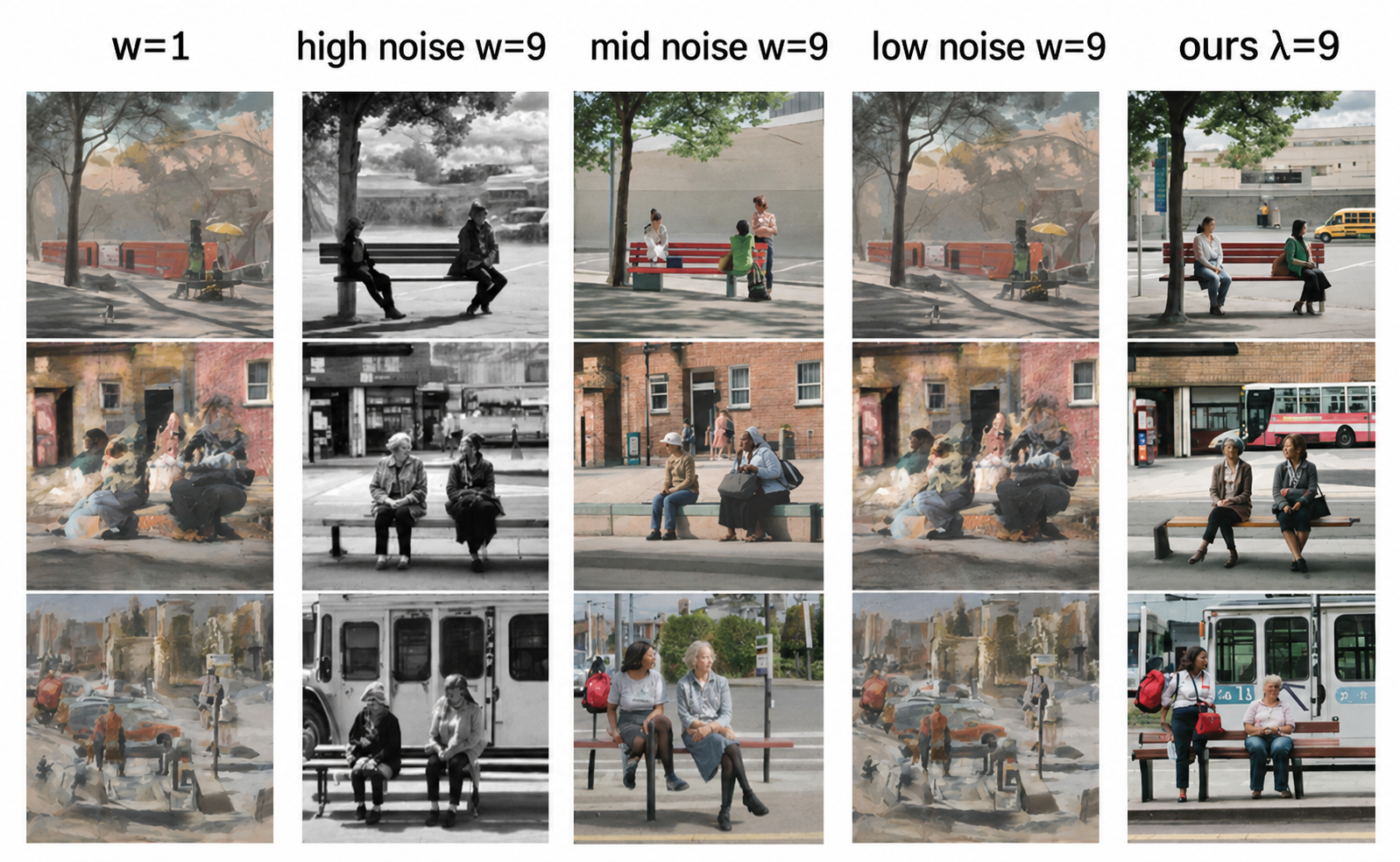}
        \caption{Prompt: A woman sitting on a bench and a woman standing waiting for the bus.}
        \label{fig:interval_ablation_coco}
    \end{subfigure}
    \caption{
    Stage-wise ablation of guidance location on ImageNet and COCO.
    Guidance is activated only in one noise range and set to \(w_t=1\)
    elsewhere. Columns compare \(w=1\), high-noise guidance, middle-noise
    guidance, low-noise guidance, and the learned adaptive schedule.
    }
    \label{fig:interval_ablation}
\end{figure*}

\paragraph{Denoising trajectory visualization.}
\begin{figure}[t]
    \centering
    \includegraphics[width=0.85\textwidth]{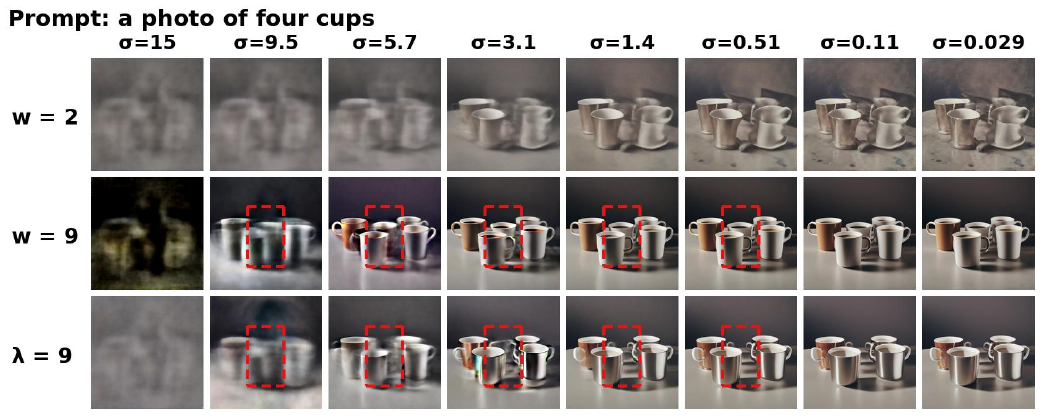}

    \caption{
    Denoising trajectory visualization for the prompt ``a photo of four cups.''
    At each noise level \(\sigma_t\), we visualize the clean prediction
    \(\hat x_0^{\mathbf w}(x_t,y)=x_t+\sigma_t^2s_t^{\mathbf w}(x_t,y)\),
    using the VE score-to-denoiser transformation. From left to
    right, columns show the clean predictions as the noise level decreases.
    }
    \label{fig:four_cups_denoising_trajectory}
    \vspace{-1.6em}
\end{figure}
To illustrate how guidance timing affects later denoising, we visualize the
clean prediction at several noise levels in
Figure~\ref{fig:four_cups_denoising_trajectory}. Under constant \(w=9\), an
ambiguous dark region is pushed toward a fifth cup at an early stage and
persists through subsequent denoising, yielding five cups in the final image.
Our learned schedule with \(\lambda=9\) applies weaker early guidance; the same
region remains less structured and is later suppressed, producing four cups.
This example illustrates how early guidance can affect later denoising,
motivating joint optimization of guidance across the trajectory.

\section{Additional Generation}
\label{app:additional_generation}
\begin{figure*}[t]
    \centering
    \begin{subfigure}[t]{0.48\textwidth}
        \centering
        \includegraphics[width=\linewidth]{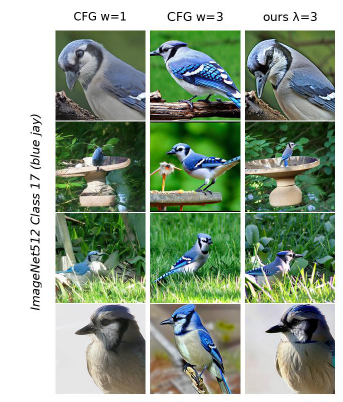}
        \label{fig:guidance_motivation_imagenet}
    \end{subfigure}
    \hfill
    \begin{subfigure}[t]{0.48\textwidth}
        \centering
        \includegraphics[width=\linewidth]{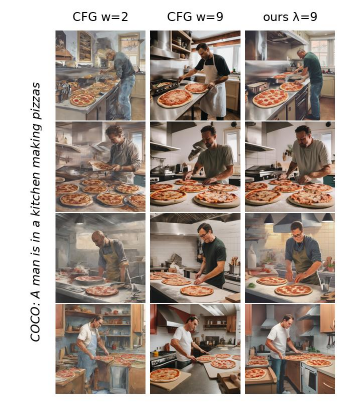}
        \label{fig:guidance_motivation_coco}
    \end{subfigure}
\vspace{-1.5em}
    \caption{
    Qualitative comparison of constant and learned guidance schedules.
    Left: ImageNet class-conditional samples under \(w=1\), constant guidance \(w=3\), and our learned schedule with \(\lambda=3\).
    Right: COCO text-to-image samples for the prompt ``A man is in a kitchen
    making pizzas,'' comparing constant guidance \(w=2\), constant
    guidance \(w=9\), and our learned schedule with \(\lambda=9\).
    }
    \label{fig:guidance_motivation_old}
    \vspace{-1.4em}
\end{figure*}
\begin{figure*}[t]
    \centering
    \includegraphics[width=\textwidth]{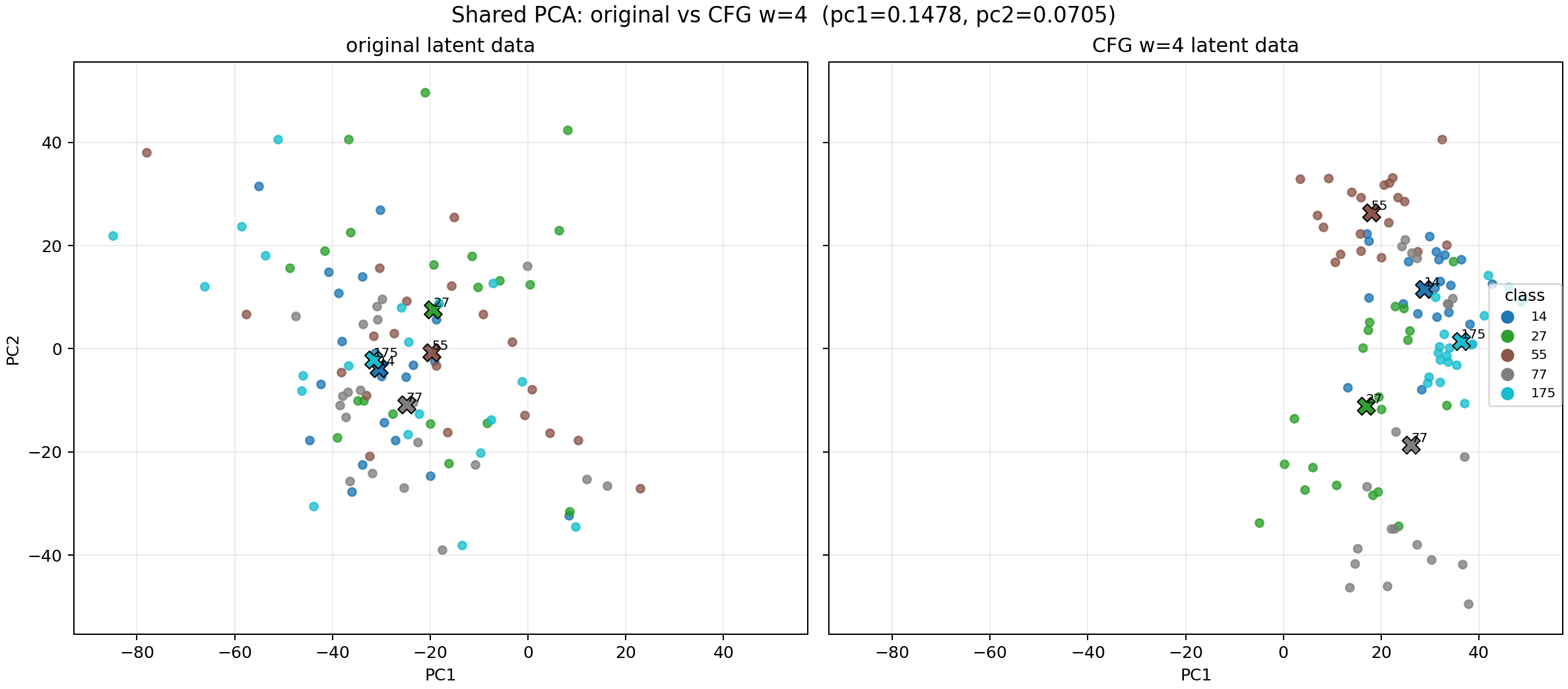}
    \caption{
    Shared PCA visualization of the original latent data and the latent data generated with CFG weight $w=4$. 
    Both panels are projected onto the same PCA basis. 
    The left panel shows the original latent data, while the right panel shows the CFG-generated latent data. 
    Colored markers indicate class membership, and large crossed markers indicate class means. 
    The CFG distribution exhibits stronger class separation and concentration in the shared PCA space.
    }
    \label{fig:shared_pca_cfgw4}
\end{figure*}

\begin{figure}[t]
    \centering
    \includegraphics[width=\textwidth]{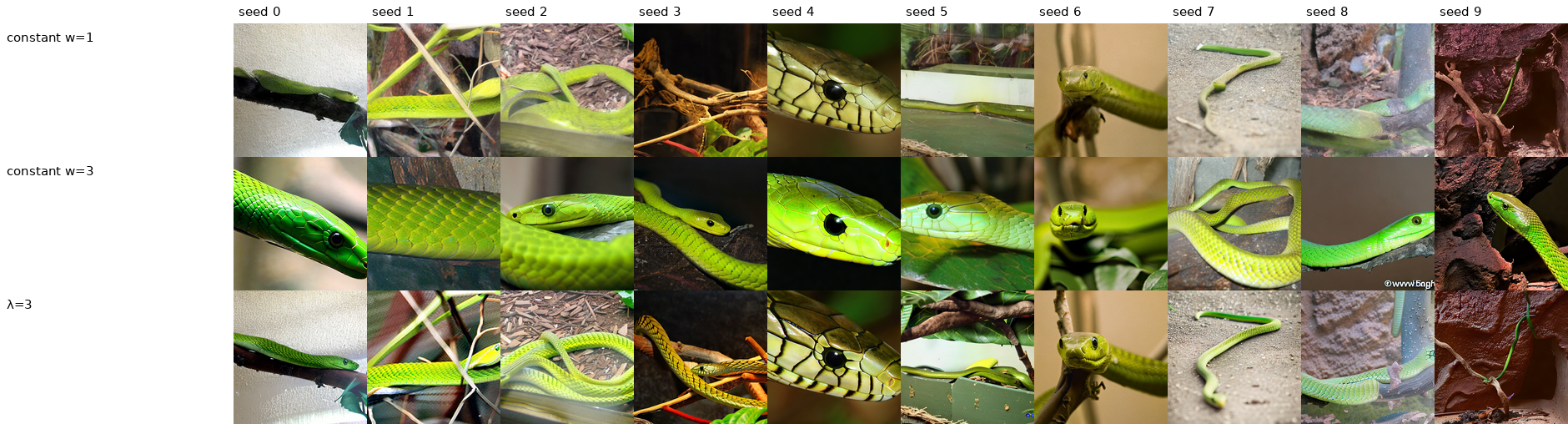}
    \caption{
    Qualitative comparison on ImageNet class 64, green mamba, using the same ten random seeds.
    The first row shows samples generated with constant guidance weight $w=1$.
    The second row shows samples generated with constant guidance weight $w=3$.
    The third row shows samples generated with the proposed step-wise optimized guidance schedule with tradeoff parameter $\lambda=3$.
    Compared with $w=1$, stronger guidance improves class consistency and visual sharpness.
    Compared with constant $w=3$, the optimized schedule preserves similar semantic consistency while maintaining more seed-level variation in pose, background, and composition.
    }
    \label{fig:green_mamba_qualitative}
\end{figure}

\begingroup
\color{black}
\section{Additional Class-Conditional $\beta$-CFG Comparison}
\label{app:rebuttal_beta_cfg_dit}
The original $\beta$-CFG paper does not report a validated
configuration for class-conditional ImageNet-512 generation.
However, C$^2$FG~\citep{gao2026c2fg} includes a class-conditional
$\beta$-CFG comparison on DiT-XL/2 at ImageNet-256, with a
different sampling-step setting. To address this additional
baseline, we evaluate $\beta$-CFG on DiT-XL/2 at ImageNet-512
using the same 100-step DDIM sampler and evaluation protocol
as our main DiT-XL/2 experiments. Following the reported beta
schedule shape, we use $a=b=2$ and test three guidance
strengths, $\omega\in\{1.0,1.5,1.7\}$.
Given the limited rebuttal period, this is a three-setting
comparison rather than an exhaustive hyperparameter sweep.
Figure~\ref{fig:app_beta_cfg_dit} shows that our method
achieves a more favorable accuracy--FID trade-off than the
tested class-conditional $\beta$-CFG settings.

\begin{figure}[t]
\centering
\captionsetup{labelfont={color=black},textfont={color=black}}
\IfFileExists{dit.png}{%
  \includegraphics[width=0.95\linewidth]{dit.png}%
}{\IfFileExists{figure/dit.png}{%
  \includegraphics[width=0.95\linewidth]{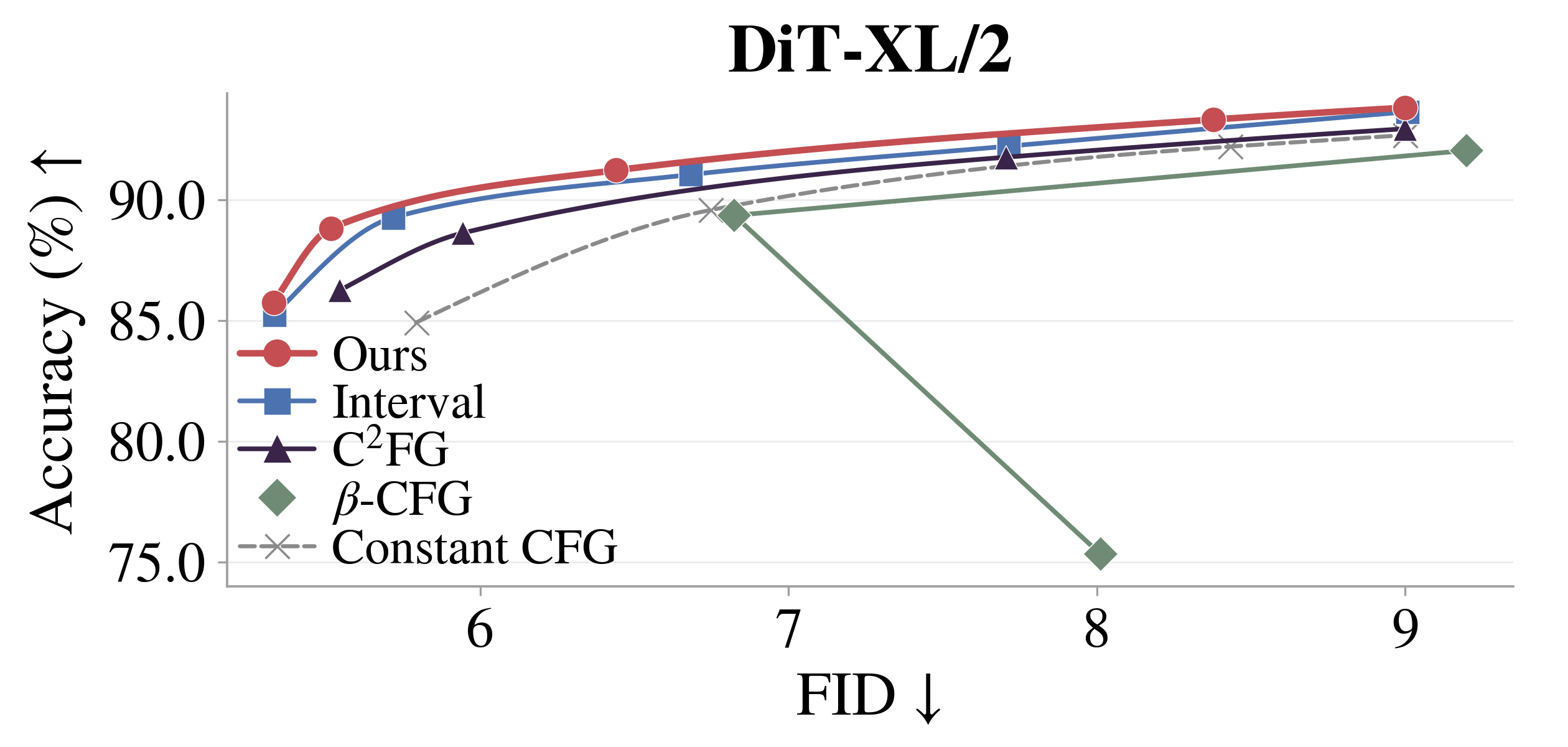}%
}{%
  \fbox{\parbox[c][3.2cm][c]{0.85\linewidth}{\centering
  DiT-XL/2 accuracy--FID comparison with $\beta$-CFG\\
  \texttt{dit.png}}}%
}}
\caption{Additional class-conditional $\beta$-CFG comparison
on DiT-XL/2 at ImageNet-512 with 100 DDIM sampling steps.
For $\beta$-CFG, $a=b=2$ and
$\omega\in\{1.0,1.5,1.7\}$.}
\label{fig:app_beta_cfg_dit}
\end{figure}
\endgroup

\clearpage
\newpage

%% file: appendix_rebuttal_gaussian.tex
\begingroup
\color{black}
\subsection{Gaussian Toy Experiment: Reproduction Settings}
\label{app:gaussian_reproduction}

For a fully specified reproduction of the qualitative illustration in
Figure~\ref{fig:matched_comparison}, we use the two-component Gaussian mixture
of \cref{sec:preliminaries} with equal class probabilities,
$Y\in\{-1,1\}$ and $\mu_{\pm}=\pm(0.85,0.55)$, and choose
\[
 \Sigma=\begin{pmatrix}0.16&0.14\\0.14&0.42\end{pmatrix}.
\]
We use the VE forward process $X_t=X_0+\sigma_t Z$ with terminal
noise $\sigma_T=5$ and the condition-independent initialization
$X_T\sim\mathcal N(0,25I_2)$.
The guided probability-flow ODE is integrated from $\sigma_T$ to
$\sigma_0=0$ with RK4 and 400 uniform noise steps.
For each $y$ and each guidance value $w\in\{1,2,5\}$, we generate
6,000 samples using random seed 2026, reusing the initial noise
across guidance settings. Conditional and unconditional scores
are computed analytically from the Gaussian mixture, and
$q_{0,\mathrm{clean}}^w(x\mid y)\propto p_0(x)p_0(y\mid x)^w$
is evaluated by numerical integration (Gaussian--Hermite quadrature
of order 48 per dimension for its means).
These settings reproduce the qualitative difference between the
guided sampler-induced distribution and the clean endpoint tilt.
\endgroup

%% file: appendix_fixed_trajectory_short.tex
\begingroup
\color{black}
\subsection{Descent and finite-trial analysis of the fixed-trajectory update}
\label{app:fixed-trajectory-analysis}

We analyze the fixed-trajectory derivative used in Algorithm~\ref{alg:stepwise_guidance_update} and derive its descent and finite-trial guarantees.
Throughout this subsection, $\mathbf w=(w_1,\ldots,w_K)$, $a_k$,
$A_k=A_{t_k}$, and $R_k=R_{t_k}$ retain their definitions from
\cref{sec:guidance_schedule_algorithm} and
\cref{app:loss-implementation}; in particular,
$R_k(x,y)=\|s_{t_k}^{\mathrm{diff}}(x,y)\|_2^2$ is \emph{not} a
trust-region radius. To distinguish the two, we write
$r_0:=R_0$, $r_{\mathrm{TR}}^{(m)}:=R_{\mathrm{TR}}^{(m)}$,
$r_{\min}:=R_{\min}$, and $r_{\max}:=R_{\max}$ for the radii
in Algorithm~\ref{alg:stepwise_guidance_update}.
The number of optimization trials is $M$ (whereas $T$ denotes the
terminal diffusion time).

\paragraph{Finite-grid objective and assumptions.}
For this analysis, consider the Euler discretization of the guided
probability-flow dynamics, indexed from high to low noise:
\begin{equation}
  X_{k-1}^{\mathbf w}=X_k^{\mathbf w}
  +a_k s_{t_k}^{w_k}(X_k^{\mathbf w},Y),\qquad
  s_{t_k}^{w_k}=s_{t_k}^{\mathrm{un}}+w_k s_{t_k}^{\mathrm{diff}},
  \quad k=K,\ldots,1.
  \label{eq:app-ft-euler}
\end{equation}
Here $X_K\sim p_T$ and $Y\sim p(y)$ are sampled independently.
Define the pointwise integrand and the \emph{full-rerun population finite-grid}
objective by
\begin{align}
 \phi_k(x,y;\omega)
   &:=(\lambda-\omega)A_k(x,y)+\omega(1-\lambda)R_k(x,y),\\
 F(\mathbf w):=\mathcal L_{\mathrm{traj}}(\mathbf w;\lambda)
   &:=\sum_{k=1}^K a_k\,\mathbb E\big[
       \phi_k(X_k^{\mathbf w},Y;w_k)\big].
 \label{eq:app-ft-F}
\end{align}
For fixed trajectories, the population counterpart of
\eqref{eq:step_update_direction} is
\begin{equation}
 G_j(\mathbf w):=a_j\,\mathbb E\big[
   -A_j(X_j^{\mathbf w},Y)+(1-\lambda)R_j(X_j^{\mathbf w},Y)
 \big],\qquad e(\mathbf w):=\nabla F(\mathbf w)-G(\mathbf w).
 \label{eq:app-ft-G}
\end{equation}
Assume differentiation can be interchanged with expectation, the
score fields and $\phi_k$ are continuously differentiable, and
$\|\nabla_xs_{t_k}^{w_k}(x,y)\|_{\mathrm{op}}\le\ell_k$ on the
admissible schedule box. We assume the second moments of the score
differences and $\nabla_x\phi_k$ below are finite. For the finite-trial
analysis, additionally assume $F$ is bounded below by $F_{\inf}$ and
has an $L$-Lipschitz gradient. We analyze the exact-evaluation version
of Algorithm~\ref{alg:stepwise_guidance_update} with untruncated TR
proposals; stochastic estimation and active box constraints are not
covered by this bound.

\paragraph{Lemma A (Exact omitted trajectory dependence).}
For $j=1,\ldots,K$, the full gradient decomposes exactly as
\begin{equation}
 \frac{\partial F}{\partial w_j}
 =G_j(\mathbf w)+e_j(\mathbf w),\qquad
 e_j(\mathbf w)=\sum_{k<j}a_k\,\mathbb E\!\left[
 \nabla_x\phi_k(X_k^{\mathbf w},Y;w_k)^\top
 \frac{\partial X_k^{\mathbf w}}{\partial w_j}\right].
 \label{eq:app-ft-exact}
\end{equation}
Writing $B_k:=I+a_k\nabla_xs_{t_k}^{w_k}(X_k^{\mathbf w},Y)$,
let $P_0:=0$ and
$P_k:=a_k\nabla_x\phi_k(X_k^{\mathbf w},Y;w_k)+B_k^\top P_{k-1}$
for $k=1,\ldots,K-1$. Then
\begin{equation}
 e_j(\mathbf w)=a_j\,\mathbb E\!\left[
   \left\langle P_{j-1},s_{t_j}^{\mathrm{diff}}(X_j^{\mathbf w},Y)
   \right\rangle\right],\qquad e_1(\mathbf w)=0.
 \label{eq:app-ft-adjoint}
\end{equation}
\begin{proof}
The state $X_k^{\mathbf w}$ depends only on
$(w_{k+1},\ldots,w_K)$. Hence
$\partial X_{j-1}^{\mathbf w}/\partial w_j
=a_js_{t_j}^{\mathrm{diff}}(X_j^{\mathbf w},Y)$ and, for $k<j-1$,
$\partial X_k^{\mathbf w}/\partial w_j
=B_{k+1}\cdots B_{j-1}
 a_js_{t_j}^{\mathrm{diff}}(X_j^{\mathbf w},Y)$.
Differentiating \eqref{eq:app-ft-F} gives
\eqref{eq:app-ft-exact}. The omitted term equals the inner product
with
\[
 a_j\sum_{k<j}a_k
 (B_{k+1}\cdots B_{j-1})^\top
 \nabla_x\phi_k(X_k^{\mathbf w},Y;w_k),
\]
where empty products are the identity. Induction on $j$ identifies
this sum with $a_j P_{j-1}$, proving \eqref{eq:app-ft-adjoint}.
For $j=1$ there are no downstream terms.
\end{proof}

\paragraph{Proposition B (Bound on the omitted term).}
Let
\begin{align*}
 M_k&:=\sup_{\mathbf w}\big(\mathbb E\|
 \nabla_x\phi_k(X_k^{\mathbf w},Y;w_k)\|_2^2\big)^{1/2},\\
 D_j&:=\sup_{\mathbf w}\mathbb E\|
 s_{t_j}^{\mathrm{diff}}(X_j^{\mathbf w},Y)\|_2^2,\\
 \Psi_j&:=\sum_{k<j}a_k M_k
  \exp\!\left(\sum_{i=k+1}^{j-1}a_i\ell_i\right).
\end{align*}
Under the above assumptions,
\begin{equation}
 |e_j(\mathbf w)|\le a_j\sqrt{D_j}\,\Psi_j,
 \qquad
 \bar e:=\sup_{\mathbf w}\|e(\mathbf w)\|_2
 \le\left(\sum_{j=2}^K a_j^2 D_j\Psi_j^2\right)^{1/2}<\infty.
 \label{eq:app-ft-ebound}
\end{equation}
\begin{proof}
Because $\|B_i\|_{\mathrm{op}}\le1+a_i\ell_i\le e^{a_i\ell_i}$,
expanding the adjoint recursion in \eqref{eq:app-ft-adjoint}
and applying Minkowski's inequality gives
$(\mathbb E\|P_{j-1}\|_2^2)^{1/2}\le\Psi_j$.
Cauchy--Schwarz then yields the componentwise bound;
squaring and summing gives \eqref{eq:app-ft-ebound}.
\end{proof}

\paragraph{Proposition H (Descent at the CFG initialization).}
This result concerns the \emph{ideal continuous-time} endpoint
objective $\mathcal L_0$ in \eqref{eq:terminal_cfg_objective}:
assume exact scores, exact probability-flow dynamics, the exact
conditional terminal initialization $X_T\sim p_T(\cdot\mid y)$,
valid integration by parts, and $\lambda>1$.
Write
$C(\mathbf w):=\mathbb E_{p(y)q_0^{\mathbf w}}
[\log p_0(y\mid X_0)]$.
At $\mathbf 1=(1,\ldots,1)$,
\begin{align}
 \nabla\mathcal L_0(\mathbf 1;\lambda)
   &=-(\lambda-1)\nabla C(\mathbf 1),\nonumber\\
 G_j^{\mathrm{cont}}(\mathbf 1)
   &=-(\lambda-1)\int_{t_{j-1}}^{t_j}
     \sigma_t\dot\sigma_t\,
     \mathbb E_{p(y)p_t(\cdot\mid y)}\|s_t^{\mathrm{diff}}\|_2^2\,dt\le0.
 \label{eq:app-ft-init}
\end{align}
Thus the proposal along $-G^{\mathrm{cont}}(\mathbf1)$ is a strict
descent direction for $\mathcal L_0$ whenever
$\sum_j|G_j^{\mathrm{cont}}(\mathbf1)|
\,\partial_{w_j}C(\mathbf1)>0$; in particular, it suffices that
$\partial_{w_j}C(\mathbf1)\ge0$ for all $j$, with a strict
inequality for some $j$ with $G_j^{\mathrm{cont}}(\mathbf1)\ne0$.
\begin{proof}
Bayes' rule gives, for any $\mathbf w$,
\[
 \mathcal L_0(\mathbf w;\lambda)
 =\mathbb E_{p(y)}\KL\!\left(q_0^{\mathbf w}(\cdot\mid y)
 \|p_0(\cdot\mid y)\right)
 -(\lambda-1)C(\mathbf w)+H(Y).
\]
With exact conditional initialization, the first KL term is zero
at $\mathbf w=\mathbf 1$ and has zero derivative there, giving
the first identity in \eqref{eq:app-ft-init}. At this schedule,
$q_t^{\mathbf1}=p_t(\cdot\mid y)$; integration by parts yields
$\mathbb E[A_t]=0$, so the continuous fixed-trajectory coefficient
is the nonpositive expression in \eqref{eq:app-ft-init}.
The stated sign condition makes
$\langle\nabla\mathcal L_0(\mathbf1;\lambda),
 -G^{\mathrm{cont}}(\mathbf1)\rangle<0$.
\end{proof}
The practical unconditional terminal initialization and finite-grid
sampler introduce approximation errors; Proposition H is not an
unconditional descent claim for those implementations.

\paragraph{Theorem C (Descent and finite-trial guarantee).}
Assume the regularity conditions above, exact population evaluations,
and untruncated TR proposals. For any $\mathbf w$ and update
$\boldsymbol\delta$,
\begin{equation}
 F(\mathbf w+\boldsymbol\delta)
 \le F(\mathbf w)+G(\mathbf w)^\top\boldsymbol\delta
 +\|e(\mathbf w)\|_2\|\boldsymbol\delta\|_2
 +\frac L2\|\boldsymbol\delta\|_2^2.
 \label{eq:app-ft-smooth}
\end{equation}
In particular, the untruncated proposal
$\boldsymbol\delta^{(m)}=-r_{\mathrm{TR}}^{(m)}
 G(\mathbf w^{(m)})/\|G(\mathbf w^{(m)})\|_2$
is guaranteed to decrease $F$ whenever
\begin{equation}
 \|G(\mathbf w^{(m)})\|_2
 >\|e(\mathbf w^{(m)})\|_2+\frac L2r_{\mathrm{TR}}^{(m)}.
 \label{eq:app-ft-descent}
\end{equation}
Moreover, let $r_{\min}=r_0/128$ and
$r_{\max}=8r_0$ as in Algorithm~\ref{alg:stepwise_guidance_update},
let $h$ bound its numerical prediction cutoff, and let $F_{\inf}$
be a lower bound of $F$. For every $M\ge11$ trials,
\emph{including rejected trials},
\begin{equation}
 \displaystyle
 \min_{0\le m<M}\|\nabla F(\mathbf w^{(m)})\|_2
 \le 2.2\bar e+0.6Lr_{\min}+\frac h{r_{\min}}
 +\frac{220\,[F(\mathbf w^{(0)})-F_{\inf}]}{r_{\min}M}.
 \label{eq:app-ft-finite}
\end{equation}
The last term is $O(1/M)$; the nonvanishing terms quantify the
limiting accuracy and do not imply unconditional stationary-point
convergence.
\begin{proof}
The descent lemma for the $L$-smooth function $F$, followed by
$\nabla F=G+e$ and Cauchy--Schwarz, gives
\eqref{eq:app-ft-smooth}. Substitution of
$\boldsymbol\delta^{(m)}$ proves \eqref{eq:app-ft-descent}.

For the finite-trial claim, let $\eta=0.1$ be the acceptance threshold,
$N:=1+\log_2(r_{\max}/r_{\min})=11$, and
\[
 \varepsilon_*:=\max\!\left\{
 \frac{\bar e+Lr_{\min}/2}{1-\eta},\ \frac h{r_{\min}},
 \frac{F(\mathbf w^{(0)})-F_{\inf}}
 {\eta r_{\min}\lfloor M/N\rfloor}\right\}.
\]
If $\|G(\mathbf w^{(m)})\|_2>\varepsilon_*$ at every trial,
then at $r_{\mathrm{TR}}^{(m)}=r_{\min}$ the predicted decrease
$r_{\min}\|G\|_2$ exceeds the numerical cutoff and, using
\eqref{eq:app-ft-smooth}, the acceptance ratio obeys
\[
 \frac{\Delta_{\mathrm{actual}}}{\Delta_{\mathrm{pred}}}
 \ge 1-\frac{\bar e+Lr_{\min}/2}{\|G\|_2}>\eta.
\]
Thus a trial at the minimum radius must be accepted; any rejection
leaves the schedule unchanged and halves the radius (down to
$r_{\min}$). There are at least $\lfloor M/N\rfloor$ accepted
trials among the first $M$ trials. Each accepted trial decreases
$F$ by at least $\eta r_{\min}\varepsilon_*$, so the total decrease
would exceed $F(\mathbf w^{(0)})-F_{\inf}$, a contradiction.
Hence some trial has $\|G\|_2\le\varepsilon_*$ and therefore
$\|\nabla F\|_2\le\varepsilon_*+\bar e$.
Using $\lfloor M/11\rfloor\ge M/22$ for $M\ge11$ and
$\max\{a,b,c\}\le a+b+c$ for nonnegative terms gives
\eqref{eq:app-ft-finite}.
\end{proof}
\endgroup

%% file: appendix_rebuttal_optimization.tex
\begingroup
\color{black}
\subsection{Optimization Validation on Gaussian and SD~1.5 Models}
\label{app:additional_optimization_validation}

\paragraph{Gaussian validation of update directions.}
We evaluate the alignment between the fixed-trajectory derivative
$G(\mathbf w)$ of \eqref{eq:app-ft-G} and the exact full-rerun
gradient $\nabla F(\mathbf w)$ on the Gaussian example of
\cref{app:gaussian-toy}. This validation uses a 400-step
\emph{Euler} discretization of the sampler, as assumed in the
analysis of \cref{app:fixed-trajectory-analysis}; it is distinct
from the RK4 illustration described in \cref{app:gaussian_reproduction}.
For each $\lambda\in\{2,3,5\}$ we evaluate 15 optimization trials.
\begin{table}[t]
\centering
\small
\setlength{\tabcolsep}{4.5pt}
\caption{Gaussian validation of fixed-trajectory update directions
and convergence behavior. Endpoint KL reduction is calculated
independently using generated densities.}
\label{tab:app_gaussian_direction_validation}
\resizebox{\linewidth}{!}{%
\begin{tabular}{c c c c c c}
\toprule
$\lambda$ &
Initial cosine $\uparrow$ &
Minimum cosine $\uparrow$ &
Accepted/trials &
$\|\nabla F\|_2$ (initial $\to$ final) $\downarrow$ &
Endpoint KL reduction $\downarrow$\\
\midrule
$2$ & $0.987$ & $0.894$ & $15/15$ & $0.045\to0.007$ & $62\%$\\
$3$ & $0.981$ & $0.915$ & $15/15$ & $0.081\to0.020$ & $69\%$\\
$5$ & $0.977$ & $0.899$ & $15/15$ & $0.154\to0.039$ & $75\%$\\
\bottomrule
\end{tabular}}
\end{table}
All 45 proposals are accepted, and the measured fixed-trajectory
directions align with the exact full-rerun gradients (minimum cosine
$0.894$). Over 15 trials, the gradient norms decrease and
the independently evaluated endpoint KL decreases by $62$--$75\%$.

\paragraph{SD~1.5 objective during optimization.}
At each trial, we rerun the full SD~1.5 sampler and evaluate
the trajectory objective on 64 samples, for $\lambda=9$.
The objective decreases during the first five trials and then
approximately plateaus, as illustrated in
Figure~\ref{fig:app_sd15_optimization}.
\begin{figure}[t]
\centering
\IfFileExists{figure/sd15_lambda9_tr_objective_iteration.png}{%
  \includegraphics[width=0.85\linewidth]{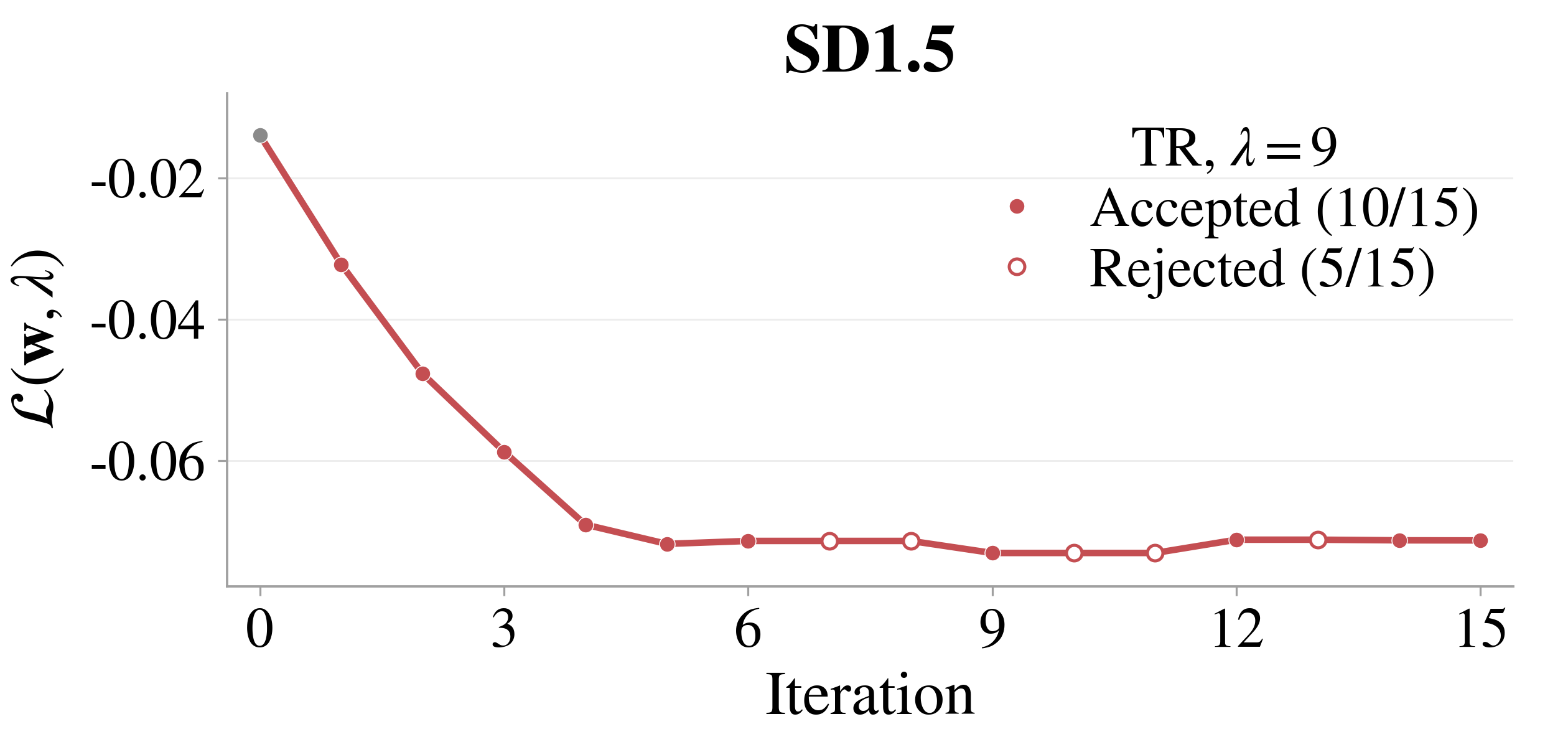}%
}{\IfFileExists{figure/sd15_lambda9_tr_objective_iteration.pdf}{%
  \includegraphics[width=0.85\linewidth]{figure/sd15_lambda9_tr_objective_iteration.pdf}%
}{%
  \fbox{\parbox[c][4.0cm][c]{0.78\linewidth}{\centering
  SD~1.5 optimization objective plot\\
  \texttt{sd15\_lambda9\_tr\_objective\_iteration.png}}}%
}}
\caption{SD~1.5: full-rerun objective evaluated on 64 samples
versus optimization iteration ($\lambda=9$).}
\label{fig:app_sd15_optimization}
\end{figure}
\endgroup

%% file: appendix_rebuttal_experiments.tex
\begingroup
\color{black}
\subsection{Additional Robustness and Baseline Comparisons}
\label{app:rebuttal_additional_experiments}

\subsubsection{Robustness across random seeds}
\label{app:rebuttal_seed_robustness}
To assess sensitivity to random seeds, we evaluate DiT-XL/2
on three disjoint groups of 10,000 generated images per method,
computing FID independently for each group using the official
ADM evaluation pipeline. Table~\ref{tab:app_seed_robustness}
reports the seed-group comparisons.
\begin{table}[t]
\centering
\caption{DiT-XL/2 FID over three disjoint 10K random-seed groups.}
\label{tab:app_seed_robustness}
\begin{tabular}{lcc}
\toprule
Random seeds & Ours FID $\downarrow$ & Interval FID $\downarrow$\\
\midrule
$0$--$9{,}999$ & $\mathbf{5.29}$ & $5.33$\\
$10{,}000$--$19{,}999$ & $\mathbf{5.24}$ & $5.29$\\
$20{,}000$--$29{,}999$ & $\mathbf{5.35}$ & $5.37$\\
\midrule
Reported mean (three groups) & $\mathbf{5.30}$ & $5.33$\\
\bottomrule
\end{tabular}
\end{table}
Our method obtains lower FID in all three groups,
providing evidence that the improvement is not limited to
one seed range.

\subsubsection{Different numbers of sampling steps}
\label{app:rebuttal_sampling_steps}
To assess performance with different numbers of sampling steps,
we evaluated EDM2-S with 8, 32, and 128 steps, generating
10K images per method at each step count.
Table~\ref{tab:app_sampling_steps} reports the results.
\begin{table}[t]
\centering
\small
\setlength{\tabcolsep}{5pt}
\caption{EDM2-S performance across sampling step counts
(10K generated images per setting).}
\label{tab:app_sampling_steps}
\begin{tabular}{rlccc}
\toprule
Steps & Method & FID $\downarrow$ & Accuracy (\%) $\uparrow$ & IS $\uparrow$\\
\midrule
8 & Constant CFG & 10.39 & 80.25 & 152.60\\
8 & Ours & \textbf{10.19} & \textbf{83.88} & \textbf{159.31}\\
32 & Constant CFG & 5.27 & 88.59 & 209.66\\
32 & Ours & \textbf{4.81} & \textbf{89.10} & \textbf{215.02}\\
128 & Constant CFG & 5.25 & 88.66 & 210.83\\
128 & Ours & \textbf{4.79} & \textbf{89.33} & \textbf{216.40}\\
\bottomrule
\end{tabular}
\end{table}
Our method outperforms Constant CFG at all three step counts,
achieving lower FID, higher classification accuracy, and higher IS.
The improvements persist when the number of sampling steps
is reduced to 8 or increased to 128.

\subsubsection{Comparison with stochastic optimal control}
\label{app:rebuttal_soc}
We compare with Adaptive Diffusion Guidance via Stochastic
Optimal Control (SOC)~\citep{azangulov2026adaptive} on EDM2-S.
Figure~\ref{fig:app_soc_comparison} shows the
classification accuracy--FID comparison.
\begin{figure}[t]
\centering
\IfFileExists{figure/edm_s512_soc_constant_ours_fid_accuracy.png}{%
  \includegraphics[width=0.82\linewidth]{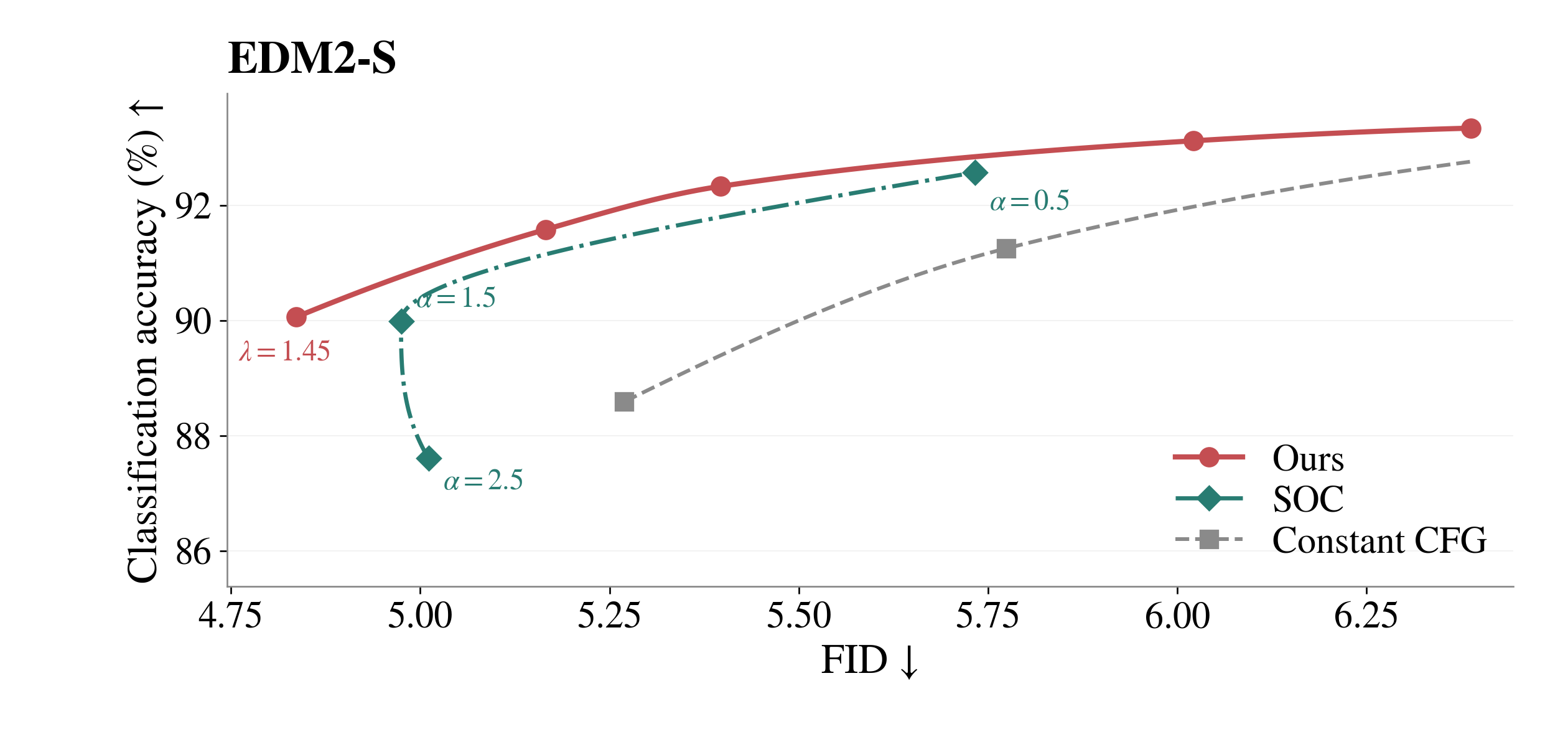}%
}{\IfFileExists{figure/edm_s512_soc_constant_ours_fid_accuracy.pdf}{%
  \includegraphics[width=0.82\linewidth]{figure/edm_s512_soc_constant_ours_fid_accuracy.pdf}%
}{%
  \fbox{\parbox[c][3.6cm][c]{0.76\linewidth}{\centering
  SOC--Constant CFG--Ours accuracy--FID figure\\
  \texttt{edm\_s512\_soc\_constant\_ours\_fid\_accuracy.png}}}%
}}
\caption{EDM2-S: classification accuracy--FID comparison
of the proposed method, stochastic optimal control (SOC),
and constant CFG.}
\label{fig:app_soc_comparison}
\end{figure}
Under the respective optimization settings, SOC uses 20,000 samples
(100 per iteration, 200 iterations), while our method draws
8 trajectories per iteration from a pool of 512 candidates
over 60 iterations. On a single NVIDIA RTX PRO 6000 Blackwell,
schedule optimization requires 17 min 14 s for our method
and approximately 12 h 59 min for SOC (about 45$\times$
difference in wall-clock time). This is a comparison of
\emph{total times under different optimization budgets},
not a same-budget per-iteration speedup.
The reported SOC timing corresponds to the $\alpha=1.5$ timing run.

\subsubsection{Consistency--recall trade-offs}
\label{app:rebuttal_accuracy_recall}
Our primary evaluation uses classification accuracy/CLIP as
a condition-consistency measure and FID as a standard measure
of overall distributional quality. FID reflects fidelity and
coverage but does not directly measure recall.
Moreover, precision measures sample fidelity, not condition
consistency, so precision--recall is not the consistency--coverage
comparison targeted by our objective. The original
Table~\ref{tab:imagenet_best_fid} reports the lowest-FID
operating point per method, not the highest-recall point.
We therefore additionally show the accuracy--recall curves
over the same operating points for DiT-XL/2 and EDM2-S.
\begin{figure*}[t]
\centering
\begin{subfigure}[t]{0.48\textwidth}
\centering
\IfFileExists{figure/dit512_recall_accuracy.png}{%
 \includegraphics[width=\linewidth]{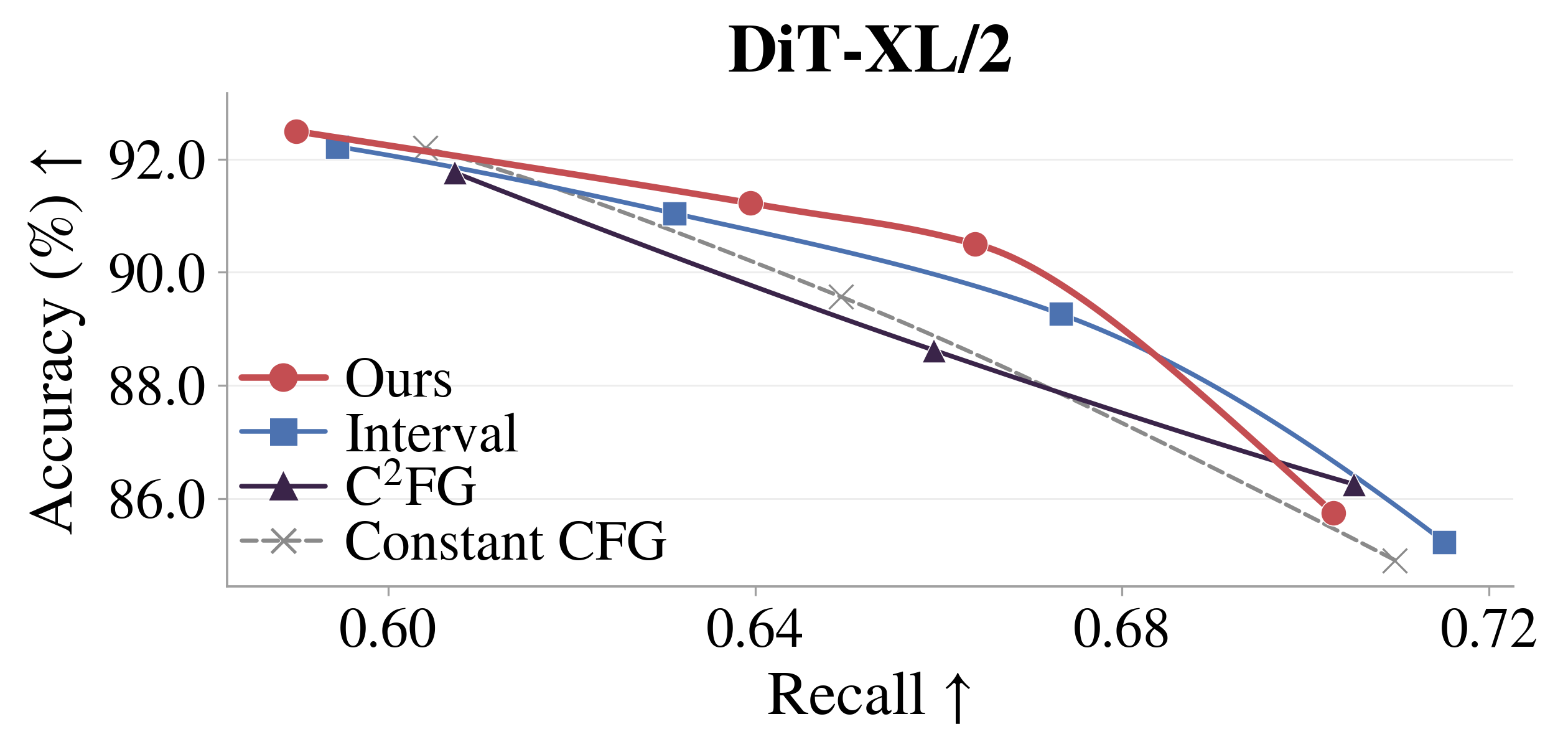}%
}{\IfFileExists{figure/dit512_recall_accuracy.pdf}{%
 \includegraphics[width=\linewidth]{figure/dit512_recall_accuracy.pdf}%
}{%
 \fbox{\parbox[c][3cm][c]{0.9\linewidth}{\centering
 DiT-XL/2 accuracy--recall figure}}%
}}
\caption{DiT-XL/2.}
\label{fig:app_dit_acc_recall}
\end{subfigure}\hfill
\begin{subfigure}[t]{0.48\textwidth}
\centering
\IfFileExists{figure/edm_s512_recall_accuracy.png}{%
 \includegraphics[width=\linewidth]{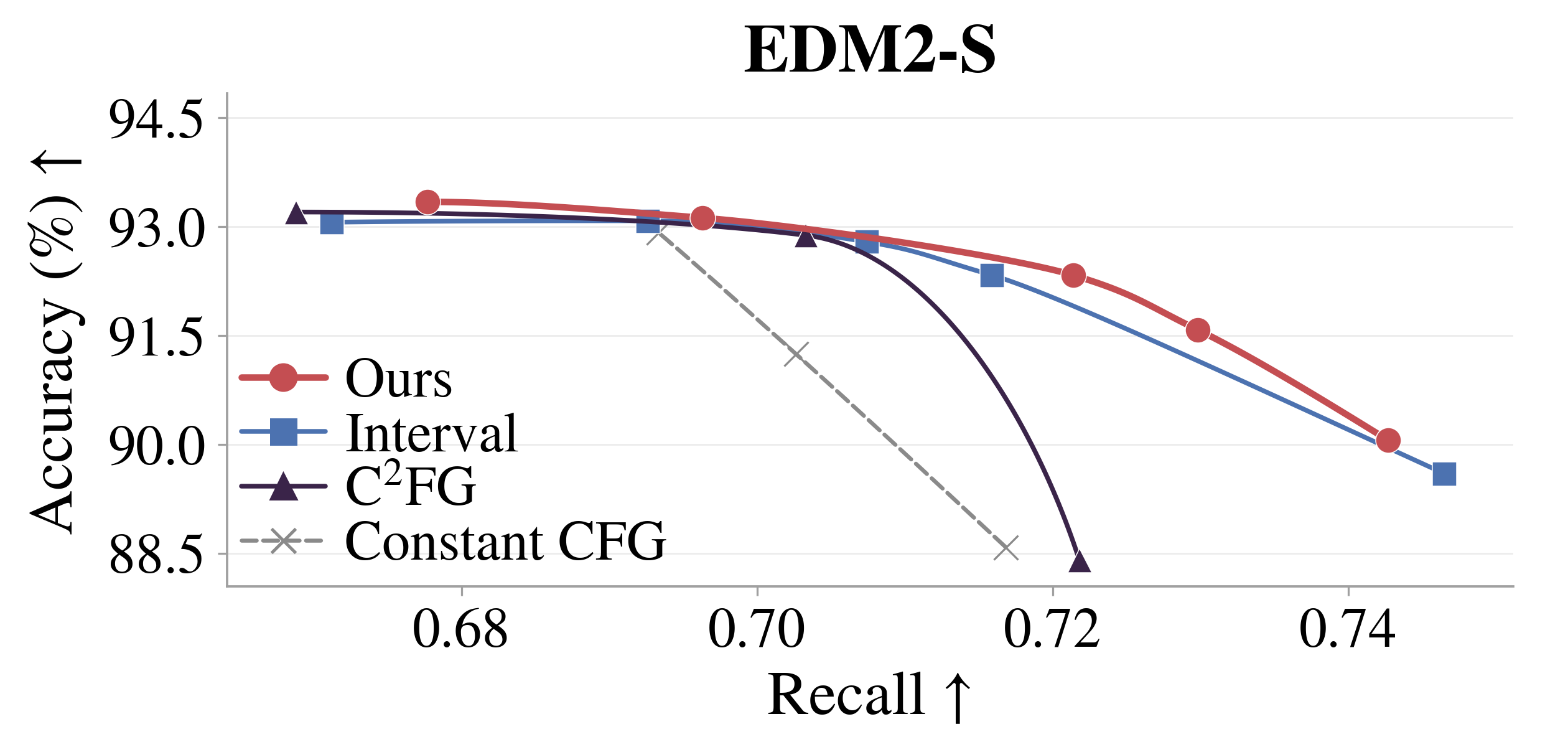}%
}{\IfFileExists{figure/edm_s512_recall_accuracy.pdf}{%
 \includegraphics[width=\linewidth]{figure/edm_s512_recall_accuracy.pdf}%
}{%
 \fbox{\parbox[c][3cm][c]{0.9\linewidth}{\centering
 EDM2-S accuracy--recall figure}}%
}}
\caption{EDM2-S.}
\label{fig:app_edm_acc_recall}
\end{subfigure}
\caption{Classification accuracy--recall trade-offs on the
ImageNet class-conditional backbones, using the same
operating points as the main comparisons.}
\label{fig:app_accuracy_recall}
\end{figure*}

\paragraph{Baseline applicability and optimization data.}
An additional class-conditional $\beta$-CFG comparison
on DiT-XL/2 at ImageNet-512 is provided in
Appendix~\ref{app:rebuttal_beta_cfg_dit},
Figure~\ref{fig:app_beta_cfg_dit}.
For SD~1.5 and SDXL, our schedule optimization uses a pool
of 64 prompt contexts, whereas evaluation uses 5,000
COCO image--caption pairs. EDM2-S uses a pool of 512
candidate trajectories for optimization, versus 10,000
generated images per reported operating point at evaluation.
FID, CLIP, and classifier accuracy are not included in the
schedule optimization objective.
\endgroup